\pdfoutput=1
\documentclass[letterpaper]{article}
\usepackage[totalwidth=480pt, totalheight=680pt]{geometry}

\usepackage{algorithm}
\usepackage{algorithmicx}
\usepackage[noend]{algpseudocode}
\usepackage{subcaption}
\usepackage{enumerate}

\usepackage{amsmath}
\usepackage{amssymb}
\usepackage{amsthm}
\usepackage{bbm}

\usepackage{pifont}
\usepackage{booktabs}
\usepackage{xcolor}
\definecolor{darkgreen}{rgb}{0,0.5,0}
\usepackage{hyperref}
\hypersetup{
    unicode=false,          
    colorlinks=true,        
    linkcolor=red,          
    citecolor=darkgreen,    
    filecolor=magenta,      
    urlcolor=blue           
}
\usepackage[capitalize, nameinlink]{cleveref}

\usepackage{thm-restate}
\theoremstyle{plain}
\newtheorem{theorem}{Theorem}

\newtheorem{lemma}[theorem]{Lemma}

\theoremstyle{definition}
\newtheorem{definition}[theorem]{Definition}
\newtheorem{problem}[theorem]{Problem}

\theoremstyle{remark}
\newtheorem{remark}[theorem]{Remark}

\usepackage{tikz}
\usetikzlibrary{calc, graphs, graphs.standard, shapes, arrows, arrows.meta, positioning, decorations.pathreplacing, decorations.markings, decorations.pathmorphing, fit, matrix, patterns, shapes.misc, tikzmark}

\newcommand{\cA}{\mathcal{A}}

\newcommand{\cD}{\mathcal{D}}
\newcommand{\cE}{\mathcal{E}}

\newcommand{\cI}{\mathcal{I}}

\newcommand{\cO}{\mathcal{O}}

\newcommand{\eps}{\varepsilon}
\newcommand{\E}{\mathbb{E}}
\newcommand{\N}{\mathbb{N}}

\newcommand{\skel}{\mathrm{skel}}
\newcommand{\Pa}{\texttt{Pa}}

\title{Causal Discovery under Off-Target Interventions}
\author{
Davin Choo\thanks{Equal contribution}\\
National University of Singapore\\
\texttt{davin@u.nus.edu}
\and
Kirankumar Shiragur\footnotemark[1]\\
Broad Institute of MIT and Harvard\\
\texttt{shiragur@stanford.edu}
\and
Caroline Uhler\\
Massachusetts Institute of Technology\\
\texttt{cuhler@mit.edu}
}
\date{}

\begin{document}

\maketitle

\begin{abstract}
Causal graph discovery is a significant problem with applications across various disciplines.
However, with observational data alone, the underlying causal graph can only be recovered up to its Markov equivalence class, and further assumptions or interventions are necessary to narrow down the true graph.
This work addresses the causal discovery problem under the setting of stochastic interventions with the natural goal of minimizing the number of interventions performed.
We propose the following stochastic intervention model which subsumes existing adaptive noiseless interventions in the literature while capturing scenarios such as fat-hand interventions and CRISPR gene knockouts: any intervention attempt results in an actual intervention on a random subset of vertices, drawn from a \emph{distribution dependent on attempted action}.
Under this model, we study the two fundamental problems in causal discovery of verification and search and provide approximation algorithms with polylogarithmic competitive ratios and provide some preliminary experimental results.
\end{abstract}

\section{Introduction}
\label{sec:intro}

Learning causal relationships is a fundamental task with applications in many fields, including epidemiology, public health, genomics, economics, and social sciences \cite{reichenbach1956direction,hoover1990logic,king2004functional,woodward2005making,rubin2006estimating,eberhardt2007interventions,cho2016reconstructing,tian2016bayesian,sverchkov2017review,rotmensch2017learning,pingault2018using,de2019combining}.
Since the development of Bayesian networks and structural equation modelling, using directed acyclic graphs (DAGs) has been a popular choice to represent causal relationships \cite{spirtes2000causation}.
It is well-known that one can recover underlying causal graphs only up to their Markov equivalence class using observational data \cite{verma1990, andersson1997characterization} and additional assumptions or interventions are necessary if one wishes to uncover the true underlying causal graph.
Here, we study causal discovery from interventions with the natural goal of minimizing the number of interventions performed.

Causal discovery using interventional data has been extensively studied with a rich literature of work developing adaptive \cite{shanmugam2015learning,greenewald2019sample,squires2020active,choo2022verification,choo2023subset,choo2023new} and non-adaptive \cite{eberhardt2005number,eberhardt2006n,eberhardt2010causal,hu2014randomized} strategies.
Traditionally, as in most prior works, an experimenter can select a vertex (or a subset of vertices) for intervention and cause the intended intervention \emph{with certainty}.
While simple and elegant, this fails to account for scenarios involving system stochasticity or noise.
For example, interventions commonly occur randomly on non-targeted genes in CRISPR gene knockouts \cite{fu2013high, wang2015unbiased, aryal2018crispr} and interventions in psychology \cite{eronen2020causal} are likely to affect variables beyond the target variable.
Such interventions are known as \emph{off-target} or \emph{fat-hand} interventions, and are prevalent in practical settings \cite{scheines2005similarity,eaton2007exact,eberhardt2007causation,duvenaud2010causal,glymour2017evaluating,eronen2020causal}.
In this work, we propose and study a stochastic interventional model that aims to model off-target interventions.

\subsection{Our off-target intervention model}

Suppose $G = (V,E)$ is a causal graph on $|V| = n$ vertices and we have $k \in \N$ possible interventional actions denoted by $A_1, \dots, A_k$.
When we perform action $A_i$, a subset $S \subseteq V$ is drawn, from the off-target distribution $\cD_i$ over the power set of vertices $2^V$, and intervened upon.
See \cref{sec:results} for more details and discussion.

Note that our interventional model\footnote{Our model is very flexible: $\cD_i$ may have support only on single vertices, or each vertex can be independently sampled into $S$, or vertex inclusion can be correlated, etc.} subsumes the traditional interventional setting when $A_i$s are on-target:
\[
\cD_i(S) =
\begin{cases}
	1 & \text{if $S = A_i$}\\
	0 & \text{otherwise}
\end{cases}
\quad
\text{, $\forall S \subseteq V$ and $\forall i \in [k]$}
\]

\textbf{Algorithmic assumptions.}\\

The complete problem is notoriously difficult and one cannot attain guarantees without making assumptions.
In this work, we make four key assumptions, each corresponding to a hard problem that is under active exploration.
As such, we view our work as initiating the study of a flexible off-target model and establishing the theoretical foundations for the problem of causal discovery under off-target inteventions.

\begin{enumerate}
    \item The distributions $\cD_1, \ldots, \cD_k$ are known\footnote{\label{footnote:weaker}Actually, we use something weaker; see \cref{sec:results}.}.

    Knowing the distributions $\cD_i$ is often a reasonable assumption in practice.
    For instance, in the case of CRISPR technology, the sequence of a target gene $i$ provides information about which other genes may also be affected.
    The possible off-target effects are generally well-understood, and one can also build a reasonable estimate of each distribution $\cD_i$ based on historical data.
    
    Moreover, when the distributions $\cD_i$ are unknown, it is straightforward to construct scenarios where any optimal algorithm must, at the very least, partially learn the distributions $\cD_i$ in order to achieve any non-trivial competitive guarantees; see \cref{sec:appendix-unknown-Di} for a discussion on the difficulties in designing algorithms with non-trivial theoretical guarantees when $\cD_i$'s are unknown.
    
    \item The actual intervened vertices $S$ are observed.
    
    While we acknowledge that this assumption may be violated in many real-world applications, there have been recent works which infer $S$ from the data \cite{kumar2021disentangling, UTIGSP}.
    For instance, \cite{UTIGSP} argues that, under further assumptions, the actual intervened targets can be recovered by checking which distributions changed after the intervention.
    We view our work as a preliminary step towards addressing the more general setting where the intervened vertices $S$ is not known to us.

    \item We are given access to the essential graph (or equivalently we know the Markov equivalence class) of the true causal graph.

    Under some standard causal assumptions, there are a plethora of algorithms that recover the essential graph from observational data (which is abundant in many applications), such as the PC \cite{spirtes2000causation}, FCI \cite{spirtes2000causation} and RFCI algorithms \cite{colombo2012learning}; see \cite{glymour2019review,vowels2022d} for a survey.
    
    \item We are able to determine orientations of edges that are incident to intervened vertices.

    While this is always possible using hard or ideal interventions, this assumption may still hold with weaker forms of interventions (soft, imperfect, shift, etc) under additional conditions.
\end{enumerate}

\subsection{Our contributions}
We study two fundamental problems in causal graph discovery: verification and search.
The former asks the question of checking whether a proposed DAG $G$ from the Markov equivalence class is the true underlying causal DAG $G^*$, which serves as a natural lower bound for the latter, which requires us to identify $G^*$ from the equivalence class.
Our contributions are as follows:
\begin{enumerate}
    \item We establish a two-way reduction between the off-target verification problem and the well-studied stochastic set covering problem.
    This equivalence allows us to leverage existing results and techniques in the literature to design our algorithms.
    \item We prove that no algorithm can achieve non-trivial competitive approximation guarantees against the off-target verification number, even when all actions have unit weight.
    This shows the difficulty of the off-target search problem and motivates the need for new benchmarks.
    \item Building on our negative result and a recent work \cite{choo2023new}, we propose algorithms that are competitive against a quantity that captures the performance of any algorithm against the worst-case causal graph within the same Markov equivalence class.
    Our algorithm runs in polynomial time and is guaranteed to use at most a polylogarithmic number of expected interventions more than the worst-case optimal solution.
\end{enumerate}
\vspace{-3pt}

One can convert expectation results to high probability ones by paying an extra $\cO(\log n)$ factor via standard applications of Markov and Chernoff bounds\footnote{By Markov inequality, each event succeeds with constant probability. Then, Chernoff bounds ensures that at least one out of $\cO(\log n)$ independent runs succeeds w.h.p.}.
 
\textbf{Outline of paper.}
We begin with preliminary notions and related work in \cref{sec:prelims}, then state our main results in \cref{sec:results}, with details of our verification and search results given in \cref{sec:verification} and \cref{sec:search} respectively.
\cref{sec:experiments} shows some experimental results and \cref{sec:conclusion} concludes with discussion and future work.
For convenience, we will restate our main theorem statements before proving them.
Some proofs are deferred to the appendix.

\section{Preliminaries and related work}
\label{sec:prelims}

For any set $A$, we use $|A|$ to denote its size and $2^A$ to denote its powerset.
Let $G = (V,E)$ be a graph on $n$ vertices/nodes.
We use $u \sim v$ to denote adjacency, and $u \to v$ or $u \gets v$ to specify directions.
For any $v \in V$ in a directed graph, we use $N(v)$ and $\Pa(v)$ to denote its neighbors and parents respectively.
The \emph{skeleton} $\skel(G)$ refers to the underlying graph where all edges are made undirected.
A \emph{v-structure}\footnote{Also known as an unshielded collider.} in $G$ refers to three vertices $u,v,w \in V$ such that $u \to v \gets w$ and $u \not\sim w$.
For any subset $S \subseteq V$, $G[S]$ denotes the node-induced subgraph of $G$ and $u \in S$ is a \emph{source vertex} if $u$ has no incoming arcs from any other edges from $S$.

A directed acyclic graph (DAG) is a fully directed graph without directed cycles.
We can associate a (not necessarily unique) \emph{valid permutation} $\sigma : V \to [n]$ to any (partially directed) DAG such that oriented arcs $u \to v$ satisfy $\sigma(u) < \sigma(v)$.
For any DAG $G$, we denote its Markov equivalence class (MEC) by $[G]$ and essential graph by $\cE(G)$.
DAGs in the same MEC $[G]$ have the same skeleton and essential graph $\cE(G)$ is a partially directed graph such that an arc $u \to v$ is directed if $u \to v$ in \emph{every} DAG in MEC $[G]$, and an edge $u \sim v$ is undirected if there exists two DAGs $G_1, G_2 \in [G]$ such that $u \to v$ in $G_1$ and $v \to u$ in $G_2$.
It is known that two graphs are Markov equivalent if and only if they have the same skeleton and v-structures \cite{verma1990,andersson1997characterization}.
A DAG is called a \emph{moral DAG} if it has no v-structures, in which case its essential graph is just its skeleton.
An edge $u \to v$ is a \emph{covered edge} \cite{chickering2013transformational} if $\Pa(u) = \Pa(v) \setminus \{u\}$.
We use $C(G)$ to denote the set of covered edges of DAG $G$.

An \emph{ideal intervention} $S \subseteq V$ is an experiment where all variables $v \in S$ are forcefully set to some value, independent of the underlying causal structure. 
The effect of interventions is formally captured by Pearl's do-calculus \cite{pearl2009causality}.
Graphically, intervening on $S$ induces a mutilated interventional graph $G_{\bar{S}}$ where all incoming arcs to vertices $v \in S$ are removed \cite{eberhardt2005number}.
It is known that intervening on $S$ allows us to infer the edge orientation of any edge cut by $S$ and $V \setminus S$ \cite{eberhardt2007causation,hyttinen2013experiment,hu2014randomized,shanmugam2015learning,kocaoglu2017cost}.
An \emph{intervention set} is a set $\cI \subseteq 2^V$ of interventions where each intervention corresponds to a subset of variables and an $\cI$-essential graph $\cE_{\cI}(G)$ of $G$ is the essential graph representing the Markov equivalence class of graphs whose interventional graphs for each intervention is Markov equivalent to $G_{\bar{S}}$ for any intervention $S \in \cI$.
There are several known properties about $\cI$-essential graphs \cite{hauser2012characterization,hauser2014two,choo2023subset}.
For instance, every $\cI$-essential graph ($\cI$ could be $\emptyset$) is a chain graph with chordal chain components\footnote{A partially directed graph is a \emph{chain graph} if it does \emph{not} contain any partially directed cycles where all directed arcs point in the same direction along the cycle. A chordal graph is a graph where every cycle of length at least 4 has an edge that is not part of the cycle but connects two vertices of the cycle; see \cite{blair1993introduction} for more.} and orientations in one chain component do not affect orientations in other components.
So, to fully orient any essential graph $\cE(G^*)$, it is necessary and sufficient to orient every chain component in $\cE(G^*)$.
We use $CC(\cE(G))$ to denote the set of chain components obtained by ignoring the oriented edges in $\cE(G)$, where each $H \in CC(\cE(G))$ is a connected undirected subgraph of $G$ and vertices across $H$'s partition $V(G)$.
See \cref{fig:toy-example} for an example.

A \emph{verifying set} $\cI$ for a DAG $G \in [G^*]$ is an intervention set that fully orients $G$ from $\cE(G^*)$, possibly with repeated applications of Meek rules (see \cref{sec:appendix-meek-rules}), i.e.\ $\cE_{\cI}(G^*) = G^*$.
Furthermore, if $\cI$ is a verifying set for $G^*$, then so is $\cI \cup S$ for any additional intervention $S \subseteq V$.
While there may be multiple verifying sets in general, we are often interested in finding one with a minimum size/cost.
We say that an intervention $S \subseteq V$ \emph{cuts} an edge $u \sim v \in E$ if $|S \cap \{u, v\}| = 1$.

\begin{definition}[Minimum size/cost verifying set and verification number/cost]
\label{def:min-verifying-set}
Let $w$ be a weight function on intervention sets.
An intervention set $\cI$ is called a verifying set for a DAG $G^*$ if $\cE_{\cI}(G^*) = G^*$.
$\cI$ is a \emph{minimum size (resp.\ cost) verifying set} if $\cE_{\cI'}(G^*) \neq G^*$ for any $|\cI'| < |\cI|$ (resp.\ for any $w(\cI') < w(\cI)$).
The \emph{minimum verification number} $\nu(G)$ and the \emph{minimum verification cost} $\overline{\nu}(G)$ denote the size/cost of the minimum size/cost verifying set respectively.
\end{definition}

There is a rich literature in recovering causal graphs via interventions under various settings such as bounded size interventions, interventions with varying vertex costs, allowing for randomization, modelling as Bayesian approaches, incorporating domain knowledge as constraints, etc. \cite{heckerman1995bayesian,heckerman1995learning,cooper1999causal,friedman2000being,tong2001active,murphy2001active,he2008active,masegosa2013interactive,cho2016reconstructing,heckerman2006bayesian,eberhardt2006n,eberhardt2010causal,eberhardt2012number,hauser2014two,hyttinen2013experiment,hu2014randomized,shanmugam2015learning,kocaoglu2017cost,ghassami2018budgeted,lindgren2018experimental,agrawal2019abcd,greenewald2019sample,katz2019size,squires2020active,choo2022verification,tigas2022interventions,pmlr-v202-tigas23a,choo2023subset,choo2023new}.
In this work, we are interested in studying \emph{off-target interventions} where attempting to intervene on vertex $v \in V$ (or a subset of vertices $S \subseteq V$) may result in intervening on other vertices, and possibly not even $v \in V$ itself \cite{scheines2005similarity,eaton2007exact,eberhardt2007causation,duvenaud2010causal,glymour2017evaluating,eronen2020causal}.

In the context of causal graph discovery via ideal interventions, \cite{choo2023subset} tells us it suffices to study the verification and search problems on moral DAGs as any oriented arcs in the observational graph can be removed \emph{before performing any interventions} as the optimality of the solution is unaffected.
Below, we review other known results that we later use.
For instance, \cref{thm:can-freely-orient-unoriented-consistently} implies that any topological ordering of the vertices $\sigma$ consistent with the given set of arcs yields a DAG from the same Markov equivalence class.

\begin{theorem}[Theorem 9 of \cite{choo2022verification}]
\label{thm:verifying-must-cut-covered-edges}
Any verifying set of a DAG $G$ must cut all the covered edges.
\end{theorem}

\begin{theorem}[Proposition 16 of \cite{hauser2012characterization}, Theorem 7 of \cite{choo2023subset}]
\label{thm:can-freely-orient-unoriented-consistently}
Given any (interventional) essential graph, any acyclic orientation of the unoriented edges that does not form new v-structures induces a DAG within the same Markov equivalence class.
\end{theorem}

As mentioned in the introduction, we will relate the problem of verification in our model to the problem of stochastic set cover.
Introduced by \cite{goemans2006stochastic}, the stochastic set cover is a subproblem in the wider problem domain known as stochastic optimization whereby one wishes to optimize a certain objective under uncertainty (e.g.\ see \cite{golovin2011adaptive} and references therein).

\begin{restatable}[Stochastic set cover, with multiplicity]{problem}{stochasticsetcoverwithmultiplicity}
\label{prob:stochastic-set-cover-with-multiplicity}
Consider a set $X$ of $d$ elements and $k$ stochastic sets $S_1, \ldots, S_k$.
Each $S_i$ is associated with a weight $w_i$ and a set-specific distribution $\cD_i$ where $\cD_i$ and $\cD_j$ are independent for $i \neq j$.
When $S_i$ is picked, a random subset of $X$ is drawn according to $\cD_i$ and the elements in the subset are said to be covered.
The goal is to minimize the weighted set selection cost (sets may be picked multiple times) until all elements in $X$ are covered.
\end{restatable}

Denoting $\mu_j(i)$ as the probability of set $S_i$ covering element $j \in X$, \cite{goemans2006stochastic} showed that \emph{any} policy $\pi$ that succeeds in covering $X$ satisfies the inequality $\sum_{i=1}^k \mu_j(i) \cdot x_i^\pi \geq 1$, where $x^{\pi}_i$ is the expected number of times $\pi$ picks $S_i$.
They further showed that minimum expected cost incurred by any adaptive policy $\pi$ for solving \cref{prob:stochastic-set-cover-with-multiplicity} is at least the optimal value of \eqref{eq:LP}:
\begin{equation}
\label{eq:LP}
\tag{LP}
\begin{array}{rrlll}
\text{minimize} & \sum_{i=1}^{k} w_i \cdot x_i\\
\text{such that}
& \sum_{i=1}^{k} \mu_j(i) \cdot x_i & \geq & 1 & , \forall j \in X\\
& x_i & \geq & 0 & , \forall i \in [k]
\end{array}
\end{equation}

\section{Results}
\label{sec:results}

As discussed in the introduction, one of our primary contributions is the definition of a new noisy off-target intervention model.
Under this model, we study two fundamental questions: verification and search.

\textbf{Off-target verification.}
Under our intervention model, the verification number is defined as follows:

\begin{problem}[Off-target verification]
\label{prob:off-target-verification}
We are given a graph $G = (V,E)$ and $k$ actions $A_1, \ldots, A_k$.
For $i \in [k]$, each set $A_i$ is associated with a distribution $\cD_i$, where $\cD_i$ and $\cD_j$ are independent for $i \neq j$.
When action $A_i$ is picked, a random subset of $V$ is drawn according to $\cD_i$ and the vertices in the subset are intervened upon.
The goal is to select as few actions as possible (actions may be picked multiple times) until the interventional essential graph is fully oriented.
\end{problem}

From \cref{thm:verifying-must-cut-covered-edges}, any verifying set of a DAG $G$ has to cut all covered edges $C(G)$ of $G$.
Given actions $A_1, \ldots, A_k$, let $\Pi$ be the space of (possibly randomized) policies that repeatedly pick actions until all covered edges are cut and $x^{\pi}_i$ be a random variable that counts the number of times action $A_i$ was chosen by policy $\pi \in \Pi$.
Then, the off-target verification number and weighted off-target verfication number are given by $\nu(G) = \min_{\pi \in \Pi} \E \left[ \sum_{i=1}^k x^{\pi}_i \right]$ and $\overline{\nu}(G) = \min_{\pi \in \Pi} \E \left[ \sum_{i=1}^k w_i \cdot x^{\pi}_i \right]$ respectively, where $w_i$ is the cost of choosing action $A_i$.
When all interventions are on-target, the terms $\nu(G)$ and $\overline{\nu}(G)$ recover existing definitions in the literature.

Our first main technical result is a lower bound on $\overline{\nu}(G^*)$ and an off-target verification algorithm with a logarithmic competitive ratio.
This is made possible by a reduction between the stochastic set cover problem and the off-target verification problem, and then applying known results of \cite{goemans2006stochastic}.

\begin{restatable}{theorem}{verificationviareduction}
\label{thm:verification-via-reduction}
Stochastic set cover and off-target verification are equivalent.
There is a polynomial time adaptive policy which verifies with a cost of $\cO(\overline{\nu}(G^*) \cdot \log n)$ in expectation while obtaining an approximation ratio within $(1 - \eps) \cdot \ln n$ is NP-hard for every $\eps > 0$.
\end{restatable}

\cref{thm:verification-via-reduction} essentially tells us that our verification results attain the optimal asymptotic approximation ratio achievable in polynomial time, unless P = NP.

\textbf{Off-target search.}
For off-target search, we begin with a rather negative result that no algorithm can provide non-trivial competitive bounds against $\overline{\nu}(G^*)$, even when all actions have unit cost.

\begin{restatable}{theorem}{hardness}
\label{thm:hardness}
For any $n \in \N$, there exists a DAG $G^*$ on $n$ nodes and unit-weight actions $A_1, \ldots, A_n$ such that any algorithm pays $\Omega(n \cdot \nu(G^*))$ to recover $G^*$.
\end{restatable}

A similar inapproximability result was known in the on-target intervention literature for weighted causal graph discovery \cite{choo2023new}, in which they proved that no algorithm (even with infinite computational power) can achieve an asymptotically better approximation than $\cO(n)$ with respect to the verification cost $\overline{\nu}(G^*)$ for all ground truth causal graphs on $n$ nodes.
The authors of \cite{choo2023new} then defined a newer and more nuanced benchmark $\overline{\nu}^{\max}(G^*) = \max_{G \in [G^*]} \overline{\nu}(G)$ which shifts the comparison away from an oracle that knows $G^*$ to the best algorithm which knows $[G^*]$.
Motivated by \cite{choo2023new} and \cref{thm:hardness}, we also compare against $\overline{\nu}^{\max}(G^*)$ instead of $\overline{\nu}(G^*)$.
To this end, we give an efficient search algorithm that achieves polylogarithmic approximation guarantees against $\overline{\nu}^{\max}(G^*)$.
Our algorithm is based on 1/2-clique graph separators on chordal graphs and heavily relies on the fact that we only need to compete against \emph{some} causal DAG in the equivalence class.
Crucially, $\overline{\nu}^{\max}(G^*)$ provides us the freedom to force certain unoriented edges to become covered edges judiciously (by competing against some $G' \in [G^*]$ via \cref{thm:can-freely-orient-unoriented-consistently}) in our search algorithm.

\begin{restatable}{theorem}{searchnumaxupperbound}
\label{thm:search-nu-max-upper-bound}
Given an essential graph $\cE(G^*)$, there is an algorithm (\cref{alg:off-target-search}) which runs in polynomial time and recovers $G^*$ while incurring a cost of $\cO(\overline{\nu}^{\max}(G^*) \cdot \log^4 n)$ in expectation.
\end{restatable}

The high-level intuition for the four log factors are as follows (see \cref{sec:search} for details):
(1) action stochasticity;
(2) guessing identities of covered edges;
(3) repeated applications of graph separators;
(4) additional work to ensure recursion onto smaller subgraphs.

\textbf{Remark about action distributions and cutting probabilities.}

In our problem, we are interested in cutting edges (in particular covered edges; see \cref{thm:verifying-must-cut-covered-edges}).
For both our verification and search algorithms, all we need are \emph{edge cut probabilities} $c_i(e) \in [0,1]$ of how likely an action $A_i$ will cut an edge $e \in E$.
Given $\cD_i$s, one can compute these edge cut probabilities (with time complexities depending on the description of $\cD_i$s).
For example, in the special case where $\cD_i$s are product distributions over the vertices, i.e.\ probability of a vertex $v$ being intervened upon is $p_v$ and the $p_v$s are independent, we have
$
c_i(e = \{u, v\})
= \Pr[\text{$A_i$ cuts $e$}]
= 1 - p_u \cdot p_v + (1 - p_u) \cdot (1 - p_v)
$.
Note that the distributions $\cD_i$s are independent with respect to the actions, but not necessarily edges: for action indices $i \neq j$ and edges $e \neq e'$, $\Pr[A_i$ cuts $e]$ is independent of $\Pr[A_j$ cuts $e]$ but not necessarily independent of $\Pr[A_i$ cuts $e']$.

In the rest of this work, we assume that cutting probabilities, a weaker requirement than $\cD_i$s, are known.
The assumption is weaker because one can derive cutting probabilities from off-target distributions but not vice-versa.
For example, consider the case of a single edge $e = \{v_1,v_2\}$ and two possible distinct distributions $\cD_1$ and $\cD_2$, where $\cD_i$ always intervenes on $\{v_i\}$: both distributions induce the same cut probability $c(e)$ but we cannot identify $\cD_i$ simply from $c(e)$.
Our algorithms only rely on cut probabilities and does not need the distributions inducing them.

\section{Verification}
\label{sec:verification}

Here we prove our results for the off-target verification problem.
In particular, we prove \cref{thm:verification-via-reduction} which shows the equivalence between our off-target verification problem (\cref{prob:off-target-verification}) and the stochastic set cover problem (\cref{prob:stochastic-set-cover-with-multiplicity}).
The main technical idea behind the reduction follows from \cref{thm:verifying-must-cut-covered-edges}, which states that any verifying set must cut all the covered edges of the DAG.
The reduction treats covered edges $C(G)$ as the set of elements to cover and an action ``covers'' a cover edge if the action cuts the covered edge.

\begin{restatable}{lemma}{verificationtostochastic}
\label{lem:verification-to-stochastic}
Every off-target verification instance on causal DAG $G$ with actions $A_1, \ldots, A_k$ and covered edges $C(G)$ corresponds to a stochastic set cover instance on $|C(G)|$ items and stochastic sets $S_1, \ldots, S_k$.
\end{restatable}
\begin{proof}
To establish a one-to-one correspondence between the two problems, let $X$ be the covered edges $C(G)$ and let each stochastic set $S_i$ have exactly the distribution of $A_i$.
Then, we can solve off-target verification by invoking any algorithm $\cA$ for the stochastic set cover problem by choosing action $A_i$ whenever $S_i$ is being chosen, and informing $\cA$ that an element is covered when the corresponding covered edge is cut.
\end{proof}

Reduction in the other direction follows by designing a graph where the elements correspond to disjoint covered edges such that sets can be mapped to actions.

\begin{restatable}{lemma}{stochastictoverification}
\label{lem:stochastic-to-verification}
Every stochastic cover instance with $d$ elements and $k$ stochastic sets $S_1, \ldots, S_k$ corresponds to an off-target verification instance on a causal DAG $G$ with $2d$ vertices, $d$ edges, and $k$ actions $A_1, \ldots, A_k$.
\end{restatable}
\begin{proof}
For each element $j \in [d]$, we create two vertices $v_{j,1}$ and $v_{j,2}$ and an edge $e_j = v_{j,1} \to v_{j,2}$.
By construction, $G$ is a collection of $d$ disjoint edges and the set of covered edges is $C(G) = E$.
For $i \in [k]$, let us define action $A_i$ to be exactly $S_i$ on $\{v_{j,1}\}_{j \in [d]}$, i.e.\ $A_i$ assigns the same probability mass as $S_i$ to any subset of $\{v_{j,1}\}_{j \in [d]}$ as whatever $S_i$ places on element $j$.
Thus, the probability of action $A_i$ cutting covered edge $e_j$ is exactly the probability that $S_i$ covers element $j$.
Therefore, we can solve the stochastic set cover problem by invoking any algorithm $\cA$ for off-target verification problem by choosing set $S_i$ whenever $A_i$ is being intervened upon, and performing an intervention on $G$ according to the random realization of $S_i$.
\end{proof}

\cref{thm:verification-via-reduction} follows from our above reductions and by applying Theorems 1 and 2 of \cite{goemans2006stochastic}.
To be precise, we invoke \textsc{CutViaLP} (\cref{alg:cut-via-LP}) with covered edges as the input to provide an efficient approximation algorithm for the off-target verification problem.

\verificationviareduction*
\begin{proof}
The equivalence is via our two-way reductions: \cref{lem:verification-to-stochastic} and \cref{lem:stochastic-to-verification}.

When $\cD_i$s deterministically map to a single fixed subset of elements, stochastic set cover recovers set cover, which is NP-hard \cite{karp1972reducibility}.
Thus, stochastic set cover is also NP-hard and so it is NP-hard to exactly solve the off-target verification problem to obtain $\overline{\nu}(G^*)$.
Furthermore, approximating set cover to within a factor of $(1-\eps) \cdot \ln n$ is also NP-hard for any $\eps > 0$ \cite{dinur2014analytical}.

To obtain a policy which incurs a cost of at most $\cO(\overline{\nu}(G^*) \cdot \log n)$ in expectation, we can apply Theorems 1 and 2 of \cite{goemans2006stochastic}.
To be precise, given a set of $k$ actions $A_1, \ldots, A_k$ and a subset of target edges $T \subseteq E$, consider the following LP adapted from \eqref{eq:LP}:
\begin{equation}
\label{eq:VLP}
\tag{VLP}
\begin{array}{rrlll}
\text{min} & \sum_{i=1}^{k} w_i \cdot x_i\\
\text{s.t.}
& \sum_{i=1}^{k} c_i(e) \cdot x_i & \geq & 1 & , \forall e \in T\\
& x_i & \geq & 0 & , \forall i \in [k]
\end{array}
\end{equation}

Theorem 1 of \cite{goemans2006stochastic} tells us that $\overline{\nu}(G^*)$ is at least the optimal value of \eqref{eq:VLP}.
Meanwhile, Theorem~2 of \cite{goemans2006stochastic} describes a policy which incurs a cost of $\cO(\overline{\nu}(G^*) \cdot \log |T|) \subseteq \cO(\overline{\nu}(G^*) \cdot \log n)$: solve \eqref{eq:VLP} with optimal values $x^*_1, \ldots, x^*_k$, pick $\cO(x^*_i \cdot \log |T|)$ copies of $S_i$ in expectation, and repeat this process a constant number of times in expectation to cover all elements (see \cref{alg:cut-via-LP}).
For us, the set $T$ will be instantiated with the set of covered edges $C(G^*)$ of $G^*$.
\end{proof}

\begin{algorithm}[htb]
\begin{algorithmic}[1]
\caption{\textsc{CutViaLP}.}
\label{alg:cut-via-LP}
    \Statex \textbf{Input}: $k$ actions $A_1, \ldots, A_k$, action weights $w_1, \ldots, w_k$, $d$ edges to cut $T \subseteq E$, cutting probabilities $\{ c_i(e) \}_{i \in [k], e \in T}$.
    \Statex \textbf{Output}: A sequence of attempted interventions such that all edges are cut.
    \State Solve \eqref{eq:VLP} and let $x^*_1, \ldots, x^*_k$ be the optimal value and define $y_i = 9 \cdot x_i \cdot \ln d$ for all $i \in [k]$.\footnotemark
    \While{some edge is still not cut}
        \For{$i \in [k]$}
            \State Do action $A_i$ $\lfloor y_i \rfloor$ times deterministically.
            \State With probability $y_i - \lfloor y_i \rfloor$, do action $A_i$.
        \EndFor
    \EndWhile
\end{algorithmic}
\end{algorithm}
\footnotetext{The constant 9 is from \cite{goemans2006stochastic}; any appropriate constant works just fine.}

We will later use \cref{alg:cut-via-LP} as a subroutine in our off-target search algorithm.

\section{Search}
\label{sec:search}

We first present a negative result (\cref{thm:hardness}) which states that one cannot hope to obtain an approximation ratio better than $\Omega(n)$ with respect to the off-target verification number $\overline{\nu}(G^*)$.
Motivated by \cref{thm:hardness}, we consider the benchmark $\overline{\nu}^{\max}(G^*) = \max_{G \in [G^*]} \overline{\nu}(G)$ defined by \cite{choo2023new}, against which the authors provided competitive algorithms for causal graph discovery under weighted on-target interventions.
In \cref{thm:search-nu-max-upper-bound}, we provide an efficient approximation search algorithm (\cref{alg:off-target-search}) with polylogarithmic approximation guarantees against $\overline{\nu}^{\max}(G^*)$.

\subsection{Why compare against \texorpdfstring{$\overline{\nu}^{\max}(G^*)$}{the max}?}

We begin with the negative result that one cannot hope to attain non-trivial approximation to $\overline{\nu}(G^*)$ even when all actions have unit weights.

\hardness*

The proof sketch for \cref{thm:hardness} is as follows.
Consider the star graph on $n$ vertices $v_1, v_2, \ldots, v_n$ with $v_n$ as the center and $v_1, \ldots, v_{n-1}$ as leaves.
Such an essential graph correspond to $n$ possible DAGs, with each vertex as a possible ``hidden root''.
Suppose there are $n$ unit-weight actions $A_1, \ldots, A_n$ where $\cD_i$ each deterministically picks the leaf $v_i$ (for $1 \leq i \leq n-1$) and $\cD_n$ picks a random leaf uniformly at random.
That is, no action will ever intervene on the center vertex $v_n$.
When any leaf $v_i$ is the root node, $\nu(G^*) = 1$ as choosing action $A_i$ suffices.
However, without knowing the identity of $v_i$, one would incur a cost of $\Omega(n) \subseteq \Omega(n \cdot \nu(G^*))$.

Meanwhile, observe that $\nu^{\max}(G^*) = n$ when the center vertex $v_n$ is the root node: all edges will be covered edges and we need to cut all of them by intervening on all leaves.
Thus, instead of competing against $\nu(G^*)$, competing against $\nu^{\max}(G^*)$ would allow for one to design algorithms with more meaningful theoretical guarantees.
That is, instead of comparing against an oracle that knows $G^*$, we should compare against any algorithm which only knows $\cE(G^*)$.
Similar sentiments were also highlighted in recent works \cite{choo2023subset,choo2023new}.

\paragraph{Remark (What if a covered edge is never cut?)}
From \cite{choo2022verification}, we know that the true causal graph $G^*$ will be fully oriented \emph{if and only if} every covered edge of $G^*$ is cut.
Therefore, if a cut probability of a certain covered edge is 0, then \emph{no} algorithm can successfully orient the graph, i.e.\ $\nu(G^*) = \infty$.
Note that this is \emph{not} the case for the star graph described above since we are able to cut every edge by intervening on the leaves (via multiple off-target intervention attempts).

\subsection{A search algorithm with polylogarithmic approximation to \texorpdfstring{$\overline{\nu}^{\max}(G^*)$}{the max}}

On a high level, \cref{alg:off-target-search} repeatedly finds 1/2-clique graph separators (see \cref{thm:chordal-separator} below) to break up the chain components so that we can recurse on smaller sized chain components, \`{a} la \cite{choo2022verification,choo2023subset,choo2023new}.
See \cref{fig:toy-example} for an illustration of graphical concepts.
However, due to action stochasticities, we encounter new challenges in designing and analyzing a search algorithm for off-target interventions.

\begin{definition}[$\alpha$-separator and $\alpha$-clique separator, Definition 19 of \cite{choo2022verification}]
Let $A,B,C$ be a partition of the vertices $V$ of a graph $G = (V,E)$.
We say that $C$ is an \emph{$\alpha$-separator} if no edge joins a vertex in $A$ with a vertex in $B$ and $|A|, |B| \leq \alpha \cdot |V|$. We call $C$ is an \emph{$\alpha$-clique separator} if it is an \emph{$\alpha$-separator} and a clique.
\end{definition}

\begin{theorem}[\cite{gilbert1984separatorchordal}, instantiated for unweighted graphs]
\label{thm:chordal-separator}
Let $G = (V,E)$ be a chordal graph with $|V| \geq 2$ and $p$ vertices in its largest clique.
There exists a $1/2$-clique-separator $C$ involving at most $p-1$ vertices.
The clique $C$ can be computed in $\cO(|E|)$ time.
\end{theorem}

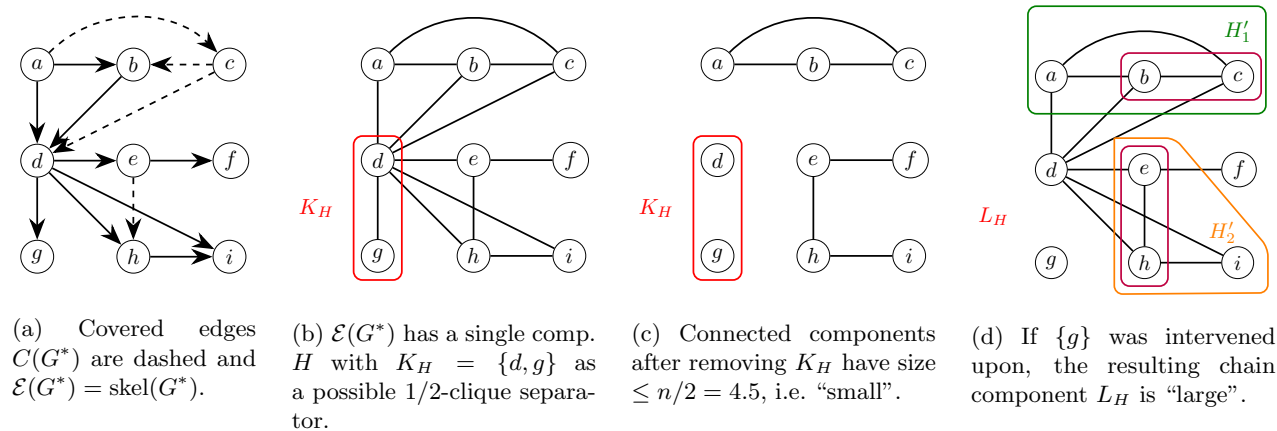
\begin{figure*}[htb]
\centering
\begin{subfigure}[t]{0.1925\linewidth}
    \resizebox{\linewidth}{!}{%
\begin{tikzpicture}
\node[draw, circle, minimum size=15pt, inner sep=2pt] at (0,0) (d) {$d$};
\node[draw, circle, minimum size=15pt, inner sep=2pt, right=of d] (e) {$e$};
\node[draw, circle, minimum size=15pt, inner sep=2pt, right=of e] (f) {$f$};
\node[draw, circle, minimum size=15pt, inner sep=2pt, above=of d] (a) {$a$};
\node[draw, circle, minimum size=15pt, inner sep=2pt, right=of a] (b) {$b$};
\node[draw, circle, minimum size=15pt, inner sep=2pt, right=of b] (c) {$c$};
\node[draw, circle, minimum size=15pt, inner sep=2pt, below=of d] (g) {$g$};
\node[draw, circle, minimum size=15pt, inner sep=2pt, right=of g] (h) {$h$};
\node[draw, circle, minimum size=15pt, inner sep=2pt, right=of h] (i) {$i$};

\node[draw, thick, rounded corners, fit=(d)(g), white] (KH) 
{};

\draw[thick, -{Stealth[scale=1.5]}] (a) -- (b);
\draw[thick, -{Stealth[scale=1.5]}, dashed] (a) to[in=135,out=45] (c);
\draw[thick, -{Stealth[scale=1.5]}] (a) -- (d);
\draw[thick, -{Stealth[scale=1.5]}] (b) -- (d);
\draw[thick, -{Stealth[scale=1.5]}, dashed] (c) -- (b);
\draw[thick, -{Stealth[scale=1.5]}, dashed] (c) -- (d);
\draw[thick, -{Stealth[scale=1.5]}] (d) -- (e);
\draw[thick, -{Stealth[scale=1.5]}] (d) -- (g);
\draw[thick, -{Stealth[scale=1.5]}] (d) -- (h);
\draw[thick, -{Stealth[scale=1.5]}] (d) -- (i);
\draw[thick, -{Stealth[scale=1.5]}] (e) -- (f);
\draw[thick, -{Stealth[scale=1.5]}, dashed] (e) -- (h);
\draw[thick, -{Stealth[scale=1.5]}] (h) -- (i);
\end{tikzpicture}
}
    \vspace{-17pt}
    \caption{Covered edges $C(G^*)$ are dashed and $\mathcal{E}(G^*) = \textrm{skel}(G^*)$.}
    \label{fig:G-star}
\end{subfigure}
\quad
\begin{subfigure}[t]{0.24\linewidth}
    \resizebox{\linewidth}{!}{%
\begin{tikzpicture}
\node[draw, circle, minimum size=15pt, inner sep=2pt] at (0,0) (d) {$d$};
\node[draw, circle, minimum size=15pt, inner sep=2pt, right=of d] (e) {$e$};
\node[draw, circle, minimum size=15pt, inner sep=2pt, right=of e] (f) {$f$};
\node[draw, circle, minimum size=15pt, inner sep=2pt, above=of d] (a) {$a$};
\node[draw, circle, minimum size=15pt, inner sep=2pt, right=of a] (b) {$b$};
\node[draw, circle, minimum size=15pt, inner sep=2pt, right=of b] (c) {$c$};
\node[draw, circle, minimum size=15pt, inner sep=2pt, below=of d] (g) {$g$};
\node[draw, circle, minimum size=15pt, inner sep=2pt, right=of g] (h) {$h$};
\node[draw, circle, minimum size=15pt, inner sep=2pt, right=of h] (i) {$i$};

\draw[thick] (a) -- (b);
\draw[thick] (a) to[in=135,out=45] (c);
\draw[thick] (a) -- (d);
\draw[thick] (b) -- (d);
\draw[thick] (c) -- (b);
\draw[thick] (c) -- (d);
\draw[thick] (d) -- (e);
\draw[thick] (d) -- (g);
\draw[thick] (d) -- (h);
\draw[thick] (d) -- (i);
\draw[thick] (e) -- (f);
\draw[thick] (e) -- (h);
\draw[thick] (h) -- (i);

\node[draw, thick, rounded corners, fit=(d)(g), red] (KH) {};
\node[red, left=5pt of KH] {$K_H$};
\end{tikzpicture}
}
    \vspace{-15pt}
    \caption{$\mathcal{E}(G^*)$ has a single comp.\ $H$ with $K_H = \{d,g\}$ as a possible 1/2-clique separator.}
    \label{fig:EG}
\end{subfigure}
\quad
\begin{subfigure}[t]{0.24\linewidth}
    \resizebox{\linewidth}{!}{%
\begin{tikzpicture}
\node[draw, circle, minimum size=15pt, inner sep=2pt] at (0,0) (d) {$d$};
\node[draw, circle, minimum size=15pt, inner sep=2pt, right=of d] (e) {$e$};
\node[draw, circle, minimum size=15pt, inner sep=2pt, right=of e] (f) {$f$};
\node[draw, circle, minimum size=15pt, inner sep=2pt, above=of d] (a) {$a$};
\node[draw, circle, minimum size=15pt, inner sep=2pt, right=of a] (b) {$b$};
\node[draw, circle, minimum size=15pt, inner sep=2pt, right=of b] (c) {$c$};
\node[draw, circle, minimum size=15pt, inner sep=2pt, below=of d] (g) {$g$};
\node[draw, circle, minimum size=15pt, inner sep=2pt, right=of g] (h) {$h$};
\node[draw, circle, minimum size=15pt, inner sep=2pt, right=of h] (i) {$i$};

\draw[thick] (a) -- (b);
\draw[thick] (a) to[in=135,out=45] (c);
\draw[thick] (c) -- (b);
\draw[thick] (e) -- (f);
\draw[thick] (e) -- (h);
\draw[thick] (h) -- (i);

\node[draw, thick, rounded corners, fit=(d)(g), red] (KH) {};
\node[red, left=5pt of KH] {$K_H$};
\end{tikzpicture}
}
    \vspace{-15pt}
    \caption{Connected components after removing $K_H$ have size $\leq n/2 = 4.5$, i.e.\ ``small''.}
    \label{fig:remove-KH}
\end{subfigure}
\quad
\begin{subfigure}[t]{0.24\linewidth}
    \resizebox{\linewidth}{!}{%
\begin{tikzpicture}
\node[draw, circle, minimum size=15pt, inner sep=2pt] at (0,0) (d) {$d$};
\node[draw, circle, minimum size=15pt, inner sep=2pt, right=of d] (e) {$e$};
\node[draw, circle, minimum size=15pt, inner sep=2pt, right=of e] (f) {$f$};
\node[draw, circle, minimum size=15pt, inner sep=2pt, above=of d] (a) {$a$};
\node[draw, circle, minimum size=15pt, inner sep=2pt, right=of a] (b) {$b$};
\node[draw, circle, minimum size=15pt, inner sep=2pt, right=of b] (c) {$c$};
\node[draw, circle, minimum size=15pt, inner sep=2pt, below=of d] (g) {$g$};
\node[draw, circle, minimum size=15pt, inner sep=2pt, right=of g] (h) {$h$};
\node[draw, circle, minimum size=15pt, inner sep=2pt, right=of h] (i) {$i$};

\node[draw, thick, rounded corners, fit=(d)(g), white] (KH) {};
\node[red, left=5pt of KH] {$L_H$};

\draw[thick] (a) -- (b);
\draw[thick] (a) to[in=135,out=45] (c);
\draw[thick] (a) -- (d);
\draw[thick] (b) -- (d);
\draw[thick] (c) -- (b);
\draw[thick] (c) -- (d);
\draw[thick] (d) -- (e);
\draw[thick] (d) -- (h);
\draw[thick] (d) -- (i);
\draw[thick] (e) -- (f);
\draw[thick] (e) -- (h);
\draw[thick] (h) -- (i);

\node[green!50!black, above=5pt of c] (H1-prime) {$H'_1$};
\node[draw, thick, rounded corners, fit={($(a.west) + (0,-0.5)$)(b)(c)(H1-prime)}, green!50!black] {};
\node[orange] at ($(i) + (-0.25,0.5)$) (H2-prime) {$H'_2$};
\draw[thick, rounded corners, orange] ($(e.north) + (-0.5,0.25)$) -- ($(h.south) + (-0.5,-0.25)$) -- ($(i) + (0.5,-0.5)$) -- ($(i) + (0.5,0.25)$) -- ($(e.north) + (0.5,0.25)$) -- ($(e.north) + (-0.5,0.25)$);

\node[draw, thick, rounded corners, fit=(b)(c), purple] {};
\node[draw, thick, rounded corners, fit=(e)(h), purple] {};
\end{tikzpicture}
}
    \vspace{-19pt}
    \caption{If $\{g\}$ was intervened upon, the resulting chain component $L_H$ is ``large''.}
    \label{fig:off-target-intervention}
\end{subfigure}
\caption{
A moral DAG $G^*$ on $n=9$ nodes illustrating graphical concepts such as $\cE(G^*)$, $\textrm{skel}(G^*)$, $C(G^*)$, 1/2-clique separators, and large chain components.
The essential graph $\mathcal{E}(G^*)$ has a single chain component $H$ with {\color{red}$K_H$} as one possible 1/2-clique separator of $H$.
Suppose we oriented $d \sim g$ through an intervention on $\{g\}$ due to action stochasticity in \cref{fig:off-target-intervention}.
The resulting chain component {\color{red}$L_H$} after intervening on $\{g\}$ is \emph{large} since $|V(L_H)| = |\{a,b,c,d,e,f,h,i\}| = 8 > n/2$.
\textsc{PerformPartitioning} (\cref{alg:perform-partitioning}) breaks {\color{red}$L_H$} up by trying to intervene within $L_H[V(L_H) \cap N({\color{red}u_H})]$, where $d \equiv {\color{red}u_H}$ is the unique vertex from {\color{red}$K_H$} within $V({\color{red}K_H}) \cap V({\color{red}L_H})$.
There are two chain components {\color{green!50!black}$H'_1$} and {\color{orange}$H'_2$} in the induced subgraph $L_H[V(L_H) \cap N({\color{red}u_H})]$.
Since $f$ is \emph{not} a neighbor of $d$, the vertex $f$ is not part of {\color{orange}$H'_2$}.
If we pick 1/2-clique separators {\color{purple}$Z'_1 = \{b,c\}$} and {\color{purple}$Z'_2 = \{e,h\}$} for {\color{green!50!black}$H'_1$} and {\color{orange}$H'_2$}, then $c \equiv {\color{purple}z'_1}$ and $e \equiv {\color{purple}z'_2}$ are the sources of {\color{purple}$Z'_1$} and {\color{purple}$Z'_2$} respectively.}
\label{fig:toy-example}
\end{figure*}

\begin{algorithm}[htb]
\begin{algorithmic}[1]
\caption{\textsc{OffTargetSearch}.}
\label{alg:off-target-search}
    \Statex \textbf{Input}: Essential graph $\cE(G^*)$ of a moral DAG $G^*$, $k$ actions $A_1, \ldots, A_k$, action weights $w_1, \ldots, w_k$, cutting probabilities $\{ c_i(e) \}_{i \in [k], e \in E}$.
    \Statex \textbf{Output}: A complete orientation of $G^*$.
    \While{there are still unoriented edges}
        \State $K \gets \emptyset$ \Comment{Collection of 1/2-clique separators}.
        \For{each chain comp.\ $H$ of size $|V(H)| \geq 2$}
            \State Find 1/2-clique separator $K_H$ of $H$.
            \State Add $K_H$ to $K$ and store the size of $|V(H)|$.
        \EndFor
        \State Run \textsc{PerformPartitioning} on $K$.
    \EndWhile
\end{algorithmic}
\end{algorithm}

To appreciate of some of the new challenges, consider the simple case where the input essential graph $\cE(G^*)$ is a clique on $n$ vertices.
If $G^*$ orients $v_1 \to \ldots \to v_n$, then the covered edges are $C(G^*) = \{v_1 \to v_2, \ldots, v_{n-1} \to v_n\}$.
By \cref{thm:verifying-must-cut-covered-edges}, interventions need to cut all edges in $C(G^*)$ to fully orient $\cE(G^*)$.
However, we do \emph{not} actually know the true edge orientations of $G^*$ (and hence $C(G^*)$).
To orient $\cE(G^*)$, \cite{choo2022verification,choo2023subset} simply intervene on all vertices while \cite{choo2023new} intervenes on all vertices except the costliest vertex.
In either case, these prior works are guaranteed to fully orient $\cE(G^*)$ regardless of the underlying orientation of $G^*$ while paying $\cO(\overline{\nu}^{\max}(G^*))$.
Unfortunately, with off-target interventions, such strategies do not work as it may be too costly to attempt to intervene on each vertex one at at time.

While an 1/2-clique separator approach typically ensures that search completes in $\cO(\log n)$ iterations if we can repeatedly identify separators and recurse on smaller subgraphs, action stochasticities may prevent us from orienting all incident edges to the separators.
We overcome this via the subroutine \textsc{PerformPartitioning} (\cref{alg:perform-partitioning}) described later.

\begin{restatable}{lemma}{offtargetsearchiterations}
\label{lem:off-target-search-iterations}
If resulting components in $H$ always have size at most $|V(H)|/2$ after invoking \textsc{PerformPartitioning}, then \textsc{OffTargetSearch} terminates in $\cO(\log n)$ iterations and outputs $G^*$.
\end{restatable}

To understand \textsc{PerformPartitioning}, we first need to describe the subroutine \textsc{OrientInternalCliqueEdges} (\cref{alg:orient-internal-clique-edges}) used to orient internal edges within any collection of disjoint cliques.

\begin{algorithm}[htb]
\begin{algorithmic}[1]
\caption{\textsc{OrientInternalCliqueEdges}.}
\label{alg:orient-internal-clique-edges}
    \Statex \textbf{Input}: Target edges $T \subseteq E$ of disjoint cliques.
    \Statex \textbf{Output}: A complete orientation of all edges in $T$.
    \While{there are still unoriented edges in $T$}
        \State Pick an arbitrary topological ordering $\sigma$ that is consistent with all revealed arc orientations.
        \State Run \textsc{CutViaLP} on the covered edges w.r.t.\ $\sigma$.
    \EndWhile
\end{algorithmic}
\end{algorithm}

\begin{restatable}{lemma}{orientinternalcliqueedgesproperties}
\label{lem:orient-internal-clique-edges-properties}
\textsc{OrientInternalCliqueEdges} incurs a cost of $\cO(\overline{\nu}^{\max}(G^*) \cdot \log^2 n)$ in expectation.
\end{restatable}

We can obtain competitive ratios against $\overline{\nu}^{\max}(G^*)$ as we can freely orient unoriented edges in an acyclic fashion to obtain a DAG within the equivalence class (\cref{thm:can-freely-orient-unoriented-consistently}).
One logarithmic factor is due to stochasticity while invoking \textsc{CutViaLP} (\cref{alg:cut-via-LP}); in a similar spirit as the coupon collector problem.
Meanwhile, the other logarithmic factor is because we do not know the identities of the covered edges within the cliques.
In more detail: after a call to \textsc{CutViaLP} on an arbitrary ordering $\sigma$, we are guaranteed that no two adjacent vertices (with respect to $\sigma$) will be in the same chain component (\cref{thm:oriented-endpoints-in-different-chain-components}), so the size of each clique is at least halved and logarithmic rounds suffice.

\begin{algorithm}[htb]
\begin{algorithmic}[1]
\caption{\textsc{PerformPartitioning}.}
\label{alg:perform-partitioning}
    \Statex \textbf{Input}: A collection of 1/2-clique separators $K = \{ K_H \}_{H \in CC(\cE_{\cI}(G^*))}$, for some $\cE_{\cI}(G^*)$.
    \Statex \textbf{Output}: An interventional essential graph where sizes of chain components in $H$ are $\leq |H|/2$.
    \State Run \textsc{OrientInternalCliqueEdges} on $E(K)$ to get interventional essential graph $\cE_{\cI'}(G^*)$.
    \State Let $L = \{ L_H \in H : |V(L_H)| > |V(H)| / 2\}$ be the large chain components with unique vertex $u_H \in V(K_H) \cap V(L_H)$ for each $L_H \in L$. \Comment{\cref{lem:exactly-one-from-KH}}
    \While{$\exists L_H \in L$ such that $|V(L_H)| \geq 2$}
        \State Find 1/2-clique separator $Z_{H'}$ of $H'$ for each $L_H \in H$ and $H' \in CC(L_H[V(L_H) \cap N(u_H)])$.
        \State Define a consistent ordering $\sigma$ s.t.\ for all $L_H \in L$ and $H' \in CC(L_H[V(L_H) \cap N(u_H)])$, we have
        \Statex\hspace{\algorithmicindent}$\sigma(u_H) < \sigma(z) < \sigma(y)$ for all $z \in V(Z_{H'})$, and for all $y \in V(H') \setminus V(Z_{H'})$.
        \State Run \textsc{OrientInternalCliqueEdges} on the covered edges w.r.t.\ $\sigma$. For each $H'$, let $z_{H'}$ be the source
        \Statex\hspace{\algorithmicindent}vertex of $Z_{H'}$ which we can identify after orienting internal edges.
        \State Define a consistent ordering $\sigma'$ s.t.\ for all $L_H \in L$ and $H' \in CC(L_H[V(L_H) \cap N(u_H)])$,
        we have
        \Statex\hspace{\algorithmicindent}$\sigma'(u_H) < \sigma'(z_{H'}) < \sigma'(y)$ for all $y \in V(L_H) \setminus \{u_H, z_{H'}\}$.
        \State Run \textsc{CutViaLP} on the covered edges w.r.t.\ $\sigma'$.
        \Comment{$u_H \sim z_{H'}$ is a covered edge}
        \State Restrict $L_H$s to the component with $u_H$s.
    \EndWhile
\end{algorithmic}
\end{algorithm}

Given 1/2-clique separators for chain components of an essential graph $\cE_{\cI}(G^*)$, \textsc{PerformPartitioning} first invokes \textsc{OrientInternalCliqueEdges} to orient the edges internal to the separators, yielding a new interventional essential graph $\cE_{\cI'}(G^*)$.
For an arbitrary chain component $H$, there may still be unoriented edges incident to separator vertices as some vertices may not be intervened by \textsc{OrientInternalCliqueEdges}.
A resulting chain component $L_H \in CC(\cE_{\cI'}(G^*))[H]$ is considered ``large'' if its size is strictly larger than $|V(H)|/2$; if it exists, there can be at most one of them for each $H$.

\begin{restatable}{lemma}{exactlyonefromKH}
\label{lem:exactly-one-from-KH}
Consider the interventional essential graph $\cE_{\cI'}(G^*)$ at Line 1 of \cref{alg:perform-partitioning} and an arbitrary chain component $H \in CC(\cE_{\cI}(G^*))$.
If $\cE_{\cI'}(G^*)[H]$ has a large chain component $L_H$ of size $|V(L_H)| > |V(H)|/2$, then $|V(L_H) \cap V(K_H)| = 1$ for any 1/2-clique separator $K_H$ of $H$.
\end{restatable}

Fix a large component $L_H$.
By \cref{lem:exactly-one-from-KH}, $L_H$ contains exactly one vertex from $K_H$, say $u_H$.
By the property of 1/2-clique separators, $L_H[V(L_H) \setminus \{u_H\}]$ consists of components of size at most $|V(H)|/2$.
That is, orienting all edges incident to $u_H$ suffices to break up $L_H$ into small components.
However, this may be costly.\footnote{\cref{lem:perform-partitioning-iterations} does not guarantee that all incident edges of $u_H$ are oriented but it suffices for breaking up $L_H$.}
Instead, we use separators while invoking \textsc{OrientInternalCliqueEdges} and \textsc{CutViaLP}.

\begin{restatable}{lemma}{performpartitioningiterations}
\label{lem:perform-partitioning-iterations}
Fix any chain component $H$ and consider any chain component $H' \in CC(L_H[V(L_H) \cap N(u_H)])$.
\textsc{PerformPartitioning} cuts $u_H \sim z_{H'}$ either in Line 6 or 8.
If $u_H$ is still connected to some chain component $C \subseteq H'$ after Line 8, then $|V(C)| \leq |V(H')|/2$.
\end{restatable}

Note that, within the while-loop, it may be the case that $H'$ is a singleton $\{z_{H'}\}$ (e.g.\ when $L_H$ is a large star with $u_H$ at the center).
In this case, the edge $u_H \sim z_{H'}$ will be cut on Line 8.

Throughout, to compete against $\overline{\nu}^{\max}(G^*)$, we only need to compete against \emph{some} causal DAG in the equivalence class and apply \cref{thm:can-freely-orient-unoriented-consistently} in our analysis suitably.
As \textsc{PerformPartitioning} uses $\cO(\log n)$ iterations to remove large chain components, and each iteration of the while loop invokes \textsc{OrientInternalCliqueEdges} and \textsc{CutViaLP} once, we have:

\begin{restatable}{lemma}{performpartitioningguarantees}
\label{lem:perform-partitioning-guarantees}
\textsc{PerformPartitioning} incurs a cost of at most $\cO(\overline{\nu}^{\max}(G^*) \cdot \log^3 n)$ per invocation, and chain components in $H$ always have size at most $|V(H)|/2$ after invoking \textsc{PerformPartitioning}.
\end{restatable}

Our search result (\cref{thm:search-nu-max-upper-bound}) immediately follows by combining \cref{lem:off-target-search-iterations} and \cref{lem:perform-partitioning-guarantees}.

\section{Experiments}
\label{sec:experiments}

While our main contributions are theoretical, we also implemented our algorithms and performed some experiments; see \cref{sec:appendix-experiments} for more details and plots.

Each experimental instance is defined on some \emph{hidden} ground truth DAG $G^*$ along with the off-target distributions $\mathcal{D}_i$s.
A search algorithm aims to recover $G^*$ with as few interventions as possible while only given the partially oriented essential graph of $G^*$ and the cutting probabilities derived from $\mathcal{D}_i$s as input.

\textbf{Algorithms.}
As our off-target intervention setting has not been studied before from an algorithmic perspective, there is no suitable prior work to compare against.
We adapted existing state-of-the-art adaptive \emph{on-target} intervention algorithms in a generic way: we solve \eqref{eq:VLP}, interpret the optimal vector $X$ as a probability distribution $p$ over the actions, then repeatedly sample from $p$ until the desired on-target intervention is completed before picking the next one.
We also show the optimal value of \eqref{eq:VLP}, along with performance of the naive \texttt{Random} and a \texttt{One-shot}\footnote{\texttt{One-shot} aims to simulate non-adaptive algorithms in the context of off-target interventions; see \cref{sec:appendix-experiments}.} baselines.

\textbf{Graph instances.}
We tested on synthetic \texttt{GNP\_TREE} graphs \cite{choo2023subset} of various sizes, and on some real-world graphs from \texttt{bnlearn} \cite{scutari2010learning}.
We associate a unit-cost action $A_v$ to each vertex $v \in V$ of the input graph.

\textbf{Interventional distributions.}
We designed 3 types of off-target interventions when taking action $A_v$.\\
(1) $r$-hop: Sample a uniform random vertex from $r$-hop neighborhood of $v$, including $v$ itself.\\
(2) Decaying: Sample a random vertex from $V$, with probability decreasing as we move from $v$.\\
(3) Fat finger: Intervene on $v$, and also possibly on some neighbors of $v$ at the same time.

\textbf{Qualitative conclusions.}
There is no prior baseline as we are the \emph{first} to study off-target interventions from an algorithmic standpoint.
Although empirical gains are not significant over some algorithms we adapted from other settings, we have theoretical guarantees for off-target interventions while they don't.
\texttt{Random} and \texttt{One-shot} fare poorly while our method is visibly better or at least competitive with the adapted on-target methods.
Regardless, note that our algorithm has provable guarantees even for non-uniform action costs and it is designed to handle worst-case off-target instances.

\section{Conclusion and discussion}
\label{sec:conclusion}

We studied causal graph discovery under off-target interventions.
Under our model, verification is equivalent to the well-studied problem of stochastic set cover and so we inherit existing approximation results from that literature.
For search, we argued that no algorithm can provide meaningful approximations to $\overline{\nu}(G^*)$, and we provided an algorithm with polyalgorithmic approximation guarantees against $\overline{\nu}^{\max}(G^*)$.

\textbf{Open directions.}
We assumed known cutting probabilities $c_i(e)$ and intervened sets $S$.
More generally, our theoretical guarantees relied on some standard causal inference assumptions in the literature\footnote{e.g.\ causal sufficiency, faithfulness, and correctly inferring conditional independences from data.} and we view our work as laying the theoretical foundations for studying off-target interventions.
For a wider applicability, it is of great interest to validate/weaken/remove these assumptions.
In particular, extending our results into the finite sample regime would be very exciting.

\subsubsection*{Acknowledgements}
This research/project is supported by the National Research Foundation, Singapore under its AI Singapore Programme (AISG Award No: AISG-PhD/2021-08-013).
KS and CU were partially supported by the Eric and Wendy Schmidt Center at the Broad Institute, NCCIH/NIH (1DP2AT012345), ONR (N00014-22-1-2116), the MIT-IBM Watson AI Lab, and a Simons Investigator Award (to CU). We would like to thank the reviewers for valuable feedback and discussions.

\bibliography{refs}
\bibliographystyle{alpha}

\newpage
\appendix

\section{Augmenting the preliminaries}

``Cutting an edge'' behaves differently from ``intervening on one of the endpoints of that edge''.
Consider a triangle $u \sim v \sim w \sim u$.
The non-atomic intervention $\{v,w\}$ cuts the edge $u \sim v$ but not the edge $v \sim w$ because $|\{u,v\} \cap \{v,w\}| = 1$ while $|\{v,w\} \cap \{v,w\}| = 2 \neq 1$.
So, one would expect the edge $\{v,w\}$ to remain unoriented unless Meek rules apply.
For example, see \cref{fig:cut-example}:
\begin{description}
    \item[Example 1] If the ground truth DAG was $u \to v \gets w \gets u$, then intervening on $\{v,w\}$ will only orient $u \to v$ and $u \to w$ while $w \to v$ remains unoriented.
    \item[Example 2] If the ground truth DAG was $u \gets v \to w \gets u$, then intervening on $\{v,w\}$ orients $u \gets v$ and $w \gets u$, and then Meek rule R2 orients $v \to w$ via $v \to u \to w \sim v$.
\end{description}

\begin{figure}[htb]
\centering
\resizebox{\linewidth}{!}{%
\begin{tikzpicture}
\node[draw, circle, minimum size=15pt, inner sep=2pt] at (0,0) (v) {$v$};
\node[draw, circle, minimum size=15pt, inner sep=2pt, right=of v] (w) {$w$};
\node[draw, circle, minimum size=15pt, inner sep=2pt] at ($(v)!0.5!(w) + (0,1)$) (u) {$u$};
\draw[thick] (u) -- (v);
\draw[thick] (u) -- (w);
\draw[thick] (v) -- (w);

\node[draw, circle, minimum size=15pt, inner sep=2pt] at (4,0) (v1) {$v$};
\node[draw, circle, minimum size=15pt, inner sep=2pt, right=of v1] (w1) {$w$};
\node[draw, circle, minimum size=15pt, inner sep=2pt] at ($(v1)!0.5!(w1) + (0,1)$) (u1) {$u$};
\draw[thick, -stealth] (u1) -- (v1);
\draw[thick, -stealth] (u1) -- (w1);
\draw[thick, stealth-] (v1) -- (w1);

\node[draw, circle, minimum size=15pt, inner sep=2pt] at (8,0) (v1i) {$v$};
\node[draw, circle, minimum size=15pt, inner sep=2pt, right=of v1i] (w1i) {$w$};
\node[draw, circle, minimum size=15pt, inner sep=2pt] at ($(v1i)!0.5!(w1i) + (0,1)$) (u1i) {$u$};
\draw[thick, -stealth] (u1i) -- (v1i);
\draw[thick, -stealth] (u1i) -- (w1i);
\draw[thick] (v1i) -- (w1i);

\node[draw, red, thick, rounded corners, fit=(v1)(w1)] {};
\node[draw, thick, rounded corners, inner sep=10pt, fit=(u1)(v1)(w1)(u1i)(v1i)(w1i)] (eg1) {};
\node[below=5pt of eg1] {Example 1: $u \to v \gets w \gets u$};

\node[draw, circle, minimum size=15pt, inner sep=2pt] at (12,0) (v2) {$v$};
\node[draw, circle, minimum size=15pt, inner sep=2pt, right=of v2] (w2) {$w$};
\node[draw, circle, minimum size=15pt, inner sep=2pt] at ($(v2)!0.5!(w2) + (0,1)$) (u2) {$u$};
\draw[thick, stealth-, red] (u2) -- (v2);
\draw[thick, -stealth, red] (u2) -- (w2);
\draw[thick, -stealth] (v2) -- (w2);

\node[draw, circle, minimum size=15pt, inner sep=2pt] at (16,0) (v2i) {$v$};
\node[draw, circle, minimum size=15pt, inner sep=2pt, right=of v2i] (w2i) {$w$};
\node[draw, circle, minimum size=15pt, inner sep=2pt] at ($(v2i)!0.5!(w2i) + (0,1)$) (u2i) {$u$};
\draw[thick, stealth-, red] (u2i) -- (v2i);
\draw[thick, -stealth, red] (u2i) -- (w2i);
\draw[thick, -stealth, blue] (v2i) -- node[below] {\tiny Meek R2} (w2i);

\node[draw, red, thick, rounded corners, fit=(v2)(w2)] {};
\node[draw, thick, rounded corners, inner sep=10pt, fit=(u2)(v2)(w2)(u2i)(v2i)(w2i)] (eg2) {};
\node[below=5pt of eg2] {Example 2: $u \gets v \to w \gets u$};
\end{tikzpicture}
}
\caption{Illustrating the difference between cutting an edge and intervening on one of the endpoints of that edge on the moral graph $u \sim v \sim w \sim u$. In both examples, the non-atomic intervention $\{v,w\}$ cuts the edges $u \sim v$ and $u \sim w$ but not $v \sim w$. In the first case, $v \sim w$ remains unoriented. In the second case, $v \sim w$ is oriented due to Meek rules.}
\label{fig:cut-example}
\end{figure}
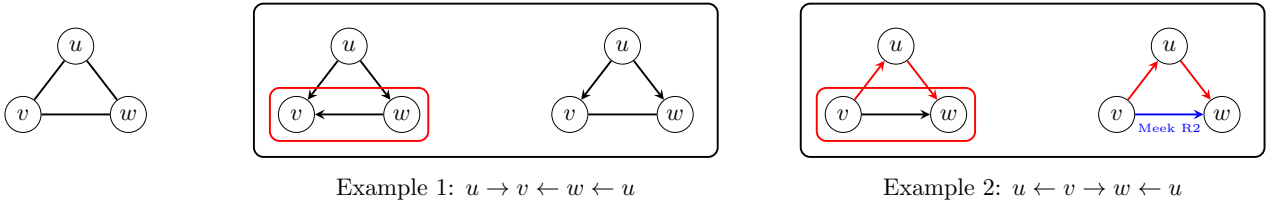

The following known results will be useful for our proofs later.

\begin{theorem}[Theorem 7 of \cite{choo2023subset}]
\label{thm:oriented-endpoints-in-different-chain-components}
If $G$ is a moral DAG and an arc $u \to v$ is oriented in an (interventional) essential graph of $G$, then $u$ and $v$ belong to different chain components.
\end{theorem}

Recall the problem of stochastic set cover with multiplicity, defined in \cref{sec:prelims}:
\stochasticsetcoverwithmultiplicity*

\cref{thm:theorems-of-GV06} states the known results of \cite{goemans2006stochastic} pertaining to \cref{prob:stochastic-set-cover-with-multiplicity}.

\begin{theorem}[Theorems 1 and 2 of \cite{goemans2006stochastic}]
\label{thm:theorems-of-GV06}
Any adaptive policy $\pi$ for solving \cref{prob:stochastic-set-cover-with-multiplicity} is at least the optimal value of \eqref{eq:LP}.
Meanwhile, there is a polynomial time non-adaptive policy for solving \cref{prob:stochastic-set-cover-with-multiplicity} such that it incurs a cost of at most $12 \ln d$ times the cost incurred by an optimal adaptive policy: solve \eqref{eq:VLP} with optimal values $x^*_1, \ldots, x^*_k$, pick $\cO(x^*_i \cdot \log |T|)$ copies of $S_i$ in expectation, and repeat this process a constant number of times in expectation to cover all elements (see \cref{alg:cut-via-LP}).
\end{theorem}

Note that obtaining an approximation ratio within $(1 - \eps) \cdot \ln n$ is NP-hard\footnote{See the proof of \cref{thm:verification-via-reduction}.} for any $\eps > 0$, so multiplicative overhead of $12 \ln d$ is asymptotically optimal.

\section{Unknown off-target distributions \texorpdfstring{$\cD_i$s}{}}
\label{sec:appendix-unknown-Di}

Here, we discuss difficulties in designing algorithms with non-trivial theoretical guarantees when $\cD_i$s are unknown.
These difficulties persist even if we know the covered edges.

Consider the following toy example where the essential graph is a tree.
That is, the underlying causal DAG is some rooted tree with an unknown root node $v^*$ and the covered edges are the edges incident to $v^*$.
So, to fully orient the graph, we need to cut all the edges incident to $v^*$.
Regardless of the number of actions $k$, there could be only one action $A_{i^*}$ with corresponding distribution $\cD_{i^*}$ with non-zero probability of cutting the covered edges; all other $A_i$s (for $i \not = i^*$) will never cut any covered edge.
Clearly, the optimal algorithm should only repeatedly take action $A_{i^*}$ until all the covered edges are cut.
Meanwhile, if we do not know the $\cD_i$s, then there is no hope for designing algorithms with non-trivial theoretical guarantees, even against $\overline{\nu}^{\max}(G^*)$:
an algorithm that knows the $\cD_i$s only picks a single action $A_{i^*}$ while any other algorithm would need to try all actions.

Note that in our lower bound result and construction (\cref{thm:hardness}), the $\cD_i$s are known.
Yet, despite this, no algorithm can have a competitive ratio better than $\Omega(n)$ against $\overline{\nu}(G^*)$; but competitiveness against $\overline{\nu}^{\max}(G^*)$ is possible.

\section{Meek rules}
\label{sec:appendix-meek-rules}

Meek rules are a set of 4 edge orientation rules that are sound and complete with respect to any given set of arcs that has a consistent DAG extension \cite{meek1995}.\footnote{This section of well-known facts is adapted from the appendices of \cite{choo2022verification,choo2023subset}.}
Given any edge orientation information, one can always repeatedly apply Meek rules till a fixed point to maximize the number of oriented arcs.

\begin{definition}[Consistent extension]
A set of arcs is said to have a \emph{consistent DAG extension} $\pi$ for a graph $G$ if there exists a permutation on the vertices such that (i) every edge $\{u,v\}$ in $G$ is oriented $u \to v$ whenever $\pi(u) < \pi(v)$, (ii) there is no directed cycle, (iii) all the given arcs are present.
\end{definition}

\begin{definition}[The four Meek rules \cite{meek1995}, see \cref{fig:meek-rules} for an illustration]
\hspace{0pt}
\begin{description}
    \item [R1] Edge $\{a,b\} \in E \setminus A$ is oriented as $a \to b$ if $\exists$ $c \in V$ such that $c \to a$ and $c \not\sim b$.
    \item [R2] Edge $\{a,b\} \in E \setminus A$ is oriented as $a \to b$ if $\exists$ $c \in V$ such that $a \to c \to b$.
    \item [R3] Edge $\{a,b\} \in E \setminus A$ is oriented as $a \to b$ if $\exists$ $c,d \in V$ such that $d \sim a \sim c$, $d \to b \gets c$, and $c \not\sim d$.
    \item [R4] Edge $\{a,b\} \in E \setminus A$ is oriented as $a \to b$ if $\exists$ $c,d \in V$ such that $d \sim a \sim c$, $d \to c \to b$, and $b \not\sim d$.
\end{description}
\end{definition}

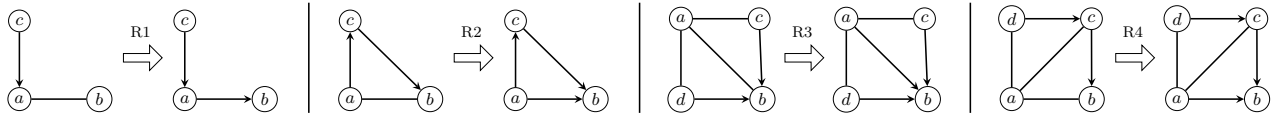
\begin{figure}[htbp]
\centering
\resizebox{\linewidth}{!}{%
\begin{tikzpicture}
%
%
\node[draw, circle, inner sep=2pt] at (0,0) (R1a-before) {\small $a$};
\node[draw, circle, inner sep=2pt, right=of R1a-before] (R1b-before) {\small $b$};
\node[draw, circle, inner sep=2pt, above=of R1a-before](R1c-before) {\small $c$};
\draw[thick, -stealth] (R1c-before) -- (R1a-before);
\draw[thick] (R1a-before) -- (R1b-before);

\node[draw, circle, inner sep=2pt] at (3,0) (R1a-after) {\small $a$};
\node[draw, circle, inner sep=2pt, right=of R1a-after] (R1b-after) {\small $b$};
\node[draw, circle, inner sep=2pt, above=of R1a-after](R1c-after) {\small $c$};
\draw[thick, -stealth] (R1c-after) -- (R1a-after);
\draw[thick, -stealth] (R1a-after) -- (R1b-after);

\node[single arrow, draw, minimum height=2em, single arrow head extend=1ex, inner sep=2pt] at (2.2,0.75) (R1arrow) {};
\node[above=5pt of R1arrow] {\footnotesize R1};

%
%
\node[draw, circle, inner sep=2pt] at (6,0) (R2a-before) {\small $a$};
\node[draw, circle, inner sep=2pt, right=of R2a-before] (R2b-before) {\small $b$};
\node[draw, circle, inner sep=2pt, above=of R2a-before](R2c-before) {\small $c$};
\draw[thick, -stealth] (R2a-before) -- (R2c-before);
\draw[thick, -stealth] (R2c-before) -- (R2b-before);
\draw[thick] (R2a-before) -- (R2b-before);

\node[draw, circle, inner sep=2pt] at (9,0) (R2a-after) {\small $a$};
\node[draw, circle, inner sep=2pt, right=of R2a-after] (R2b-after) {\small $b$};
\node[draw, circle, inner sep=2pt, above=of R2a-after](R2c-after) {\small $c$};
\draw[thick, -stealth] (R2a-after) -- (R2c-after);
\draw[thick, -stealth] (R2c-after) -- (R2b-after);
\draw[thick, -stealth] (R2a-after) -- (R2b-after);

\node[single arrow, draw, minimum height=2em, single arrow head extend=1ex, inner sep=2pt] at (8.2,0.75) (R2arrow) {};
\node[above=5pt of R2arrow] {\footnotesize R2};

%
%
\node[draw, circle, inner sep=2pt] at (12,0) (R3d-before) {\small $d$};
\node[draw, circle, inner sep=2pt, above=of R3d-before](R3a-before) {\small $a$};
\node[draw, circle, inner sep=2pt, right=of R3a-before] (R3c-before) {\small $c$};
\node[draw, circle, inner sep=2pt, right=of R3d-before](R3b-before) {\small $b$};
\draw[thick, -stealth] (R3c-before) -- (R3b-before);
\draw[thick, -stealth] (R3d-before) -- (R3b-before);
\draw[thick] (R3c-before) -- (R3a-before) -- (R3d-before);
\draw[thick] (R3a-before) -- (R3b-before);

\node[draw, circle, inner sep=2pt] at (15,0) (R3d-after) {\small $d$};
\node[draw, circle, inner sep=2pt, above=of R3d-after](R3a-after) {\small $a$};
\node[draw, circle, inner sep=2pt, right=of R3a-after] (R3c-after) {\small $c$};
\node[draw, circle, inner sep=2pt, right=of R3d-after](R3b-after) {\small $b$};
\draw[thick, -stealth] (R3c-after) -- (R3b-after);
\draw[thick, -stealth] (R3d-after) -- (R3b-after);
\draw[thick] (R3c-after) -- (R3a-after) -- (R3d-after);
\draw[thick, -stealth] (R3a-after) -- (R3b-after);

\node[single arrow, draw, minimum height=2em, single arrow head extend=1ex, inner sep=2pt] at (14.2,0.75) (R3arrow) {};
\node[above=5pt of R3arrow] {\footnotesize R3};

%
%
\node[draw, circle, inner sep=2pt] at (18,0) (R4a-before) {\small $a$};
\node[draw, circle, inner sep=2pt, above=of R4a-before](R4d-before) {\small $d$};
\node[draw, circle, inner sep=2pt, right=of R4d-before] (R4c-before) {\small $c$};
\node[draw, circle, inner sep=2pt, right=of R4a-before](R4b-before) {\small $b$};
\draw[thick, -stealth] (R4d-before) -- (R4c-before);
\draw[thick, -stealth] (R4c-before) -- (R4b-before);
\draw[thick] (R4d-before) -- (R4a-before) -- (R4c-before);
\draw[thick] (R4a-before) -- (R4b-before);

\node[draw, circle, inner sep=2pt] at (21,0) (R4a-after) {\small $a$};
\node[draw, circle, inner sep=2pt, above=of R4a-after](R4d-after) {\small $d$};
\node[draw, circle, inner sep=2pt, right=of R4d-after] (R4c-after) {\small $c$};
\node[draw, circle, inner sep=2pt, right=of R4a-after](R4b-after) {\small $b$};
\draw[thick, -stealth] (R4d-after) -- (R4c-after);
\draw[thick, -stealth] (R4c-after) -- (R4b-after);
\draw[thick] (R4d-after) -- (R4a-after) -- (R4c-after);
\draw[thick, -stealth] (R4a-after) -- (R4b-after);

\node[single arrow, draw, minimum height=2em, single arrow head extend=1ex, inner sep=2pt] at (20.2,0.75) (R4arrow) {};
\node[above=5pt of R4arrow] {\footnotesize R4};

\draw[thick] (5.25,1.75) -- (5.25,-0.25);
\draw[thick] (11.25,1.75) -- (11.25,-0.25);
\draw[thick] (17.25,1.75) -- (17.25,-0.25);
\end{tikzpicture}
}
\caption{An illustration of the four Meek rules}
\label{fig:meek-rules}
\end{figure}

There exists an algorithm (Algorithm 2 of \cite{wienobst2021extendability}) that runs in $\cO(d \cdot |E|)$ time and computes the closure under Meek rules, where $d$ is the degeneracy of the graph skeleton\footnote{A $d$-degenerate graph is an undirected graph in which every subgraph has a vertex of degree at most $d$. Note that the degeneracy of a graph is typically smaller than the maximum degree of the graph.}.

\section{Deferred proofs}
\label{sec:appendix-proofs}

It would be helpful to keep the definition of
\[
\overline{\nu}^{\max}(G^*) = \max_{G \in [G^*]} \overline{\nu}(G)
\]
and \cref{thm:verification-via-reduction} (which we restate below for convenience) in mind for our proofs below.

\verificationviareduction*

The following lemma bounds the cost incurred by \textsc{CutViaLP} to $\overline{\nu}^{\max}(G^*)$ using \cref{thm:theorems-of-GV06} when we invoke \textsc{CutViaLP} on a subset of covered edges of \emph{some} DAG $G' \in [G^*]$.

\begin{lemma}
\label{lem:cut-via-lp-cost}
If we invoke \textsc{CutViaLP} on a subset of covered edges $T$ of \emph{some} DAG $G' \in [G^*]$, then \textsc{CutViaLP} incurs a cost of $\cO(\overline{\nu}(G') \cdot \log |T|) \subseteq \cO(\overline{\nu}^{\max}(G^*) \cdot \log n)$ in expectation.
\end{lemma}
\begin{proof}
As $T$ is a subset of covered edges $C(G')$ of \emph{some} DAG $G' \in [G^*]$, \cref{thm:theorems-of-GV06} tells us that \textsc{CutViaLP} incurs a cost of $\cO(\overline{\nu}(G') \cdot \log |T|)$ in expectation.
The claim follows as $|T| \leq n$ and $\overline{\nu}^{\max}(G^*) = \max_{G \in [G^*]} \overline{\nu}(G)$.
\end{proof}

\hardness*
\begin{proof}
Consider the star graph on $n$ vertices $v_1, v_2, \ldots, v_n$ with $v_n$ as the center and $v_1, \ldots, v_{n-1}$ as leaves (\cref{fig:hardness}).
Such an essential graph is a tree and corresponds to $n$ possible DAGs, with one of the vertices as a ``hidden root'' node.
Suppose there are $n$ unit-weight actions $A_1, \ldots, A_n$ where $\cD_i$ each deterministically picks the leaf $v_i$ (for $1 \leq i \leq n-1$) and $\cD_n$ picks a random leaf uniformly at random.
That is, no action will ever intervene on the center vertex $v_n$.

\begin{figure}[htb]
\centering
\begin{tikzpicture}
    \node[draw, circle, minimum size=15pt, inner sep=2pt] at (0,0) (v0) {$v_n$};
    \node[draw, circle, minimum size=15pt, inner sep=2pt, left=of v0] (v1) {};
    \node[draw, circle, minimum size=15pt, inner sep=2pt, above left=of v0] (v2) {};
    \node[draw, circle, minimum size=15pt, inner sep=2pt, above=of v0] (v3) {};
    \node[draw, circle, minimum size=15pt, inner sep=2pt, above right=of v0] (v4) {};
    \node[draw, circle, minimum size=15pt, inner sep=2pt, right=of v0] (v5) {};
    \node[draw, circle, minimum size=15pt, inner sep=2pt, below right=of v0] (v6) {};
    \node[draw, circle, minimum size=15pt, inner sep=2pt, below=of v0] (v7) {};
    \node[draw, circle, minimum size=15pt, inner sep=2pt, below left=of v0] (v8) {};

    \draw[thick] (v0) -- (v1);
    \draw[thick] (v0) -- (v2);
    \draw[thick] (v0) -- (v3);
    \draw[thick] (v0) -- (v4);
    \draw[thick] (v0) -- (v5);
    \draw[thick] (v0) -- (v6);
    \draw[thick] (v0) -- (v7);
    \draw[thick] (v0) -- (v8);
\end{tikzpicture}
\caption{A star graph with $v_n$ as the center.}
\label{fig:hardness}
\end{figure}
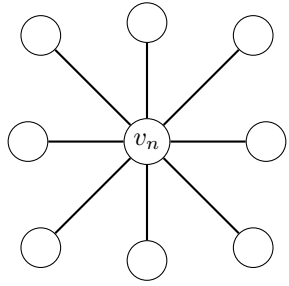

Orienting $G^*$ from the essential graph is exactly finding the hidden root leaf node, which corresponds to the problem of searching for a specific number in an unsorted array with $n-1$ numbers.
It is well-known that \emph{any} algorithm (even randomized ones) incurs $\Omega(n)$ array probes for the problem of searching in an unsorted array (e.g.\ see Theorem 15.1 of \cite{shaffer2013data}).
Since our setting restricts the set of actions (pick each index deterministically or choosing a random index uniformly at random) on the same problem, it also has a lower bound of $\Omega(n)$.
Note that the lower bound of $\Omega(n)$ holds both in the worst-case and in expectation.

Meanwhile, recall that the actions $A_i$ are deterministically picks the leaf $v_i$, for $1 \leq i \leq n-1$.
So, $\nu(G^*) = 1$ by taking action $A_{i^*}$ to intervene on $v_{i^*}$, where $i^* \in [n-1]$ is the index of the hidden root leaf node.

Therefore, any algorithm pays $\Omega(n \cdot \nu(G^*))$ to fully recover $G^*$ (both in the worst case and in expectation).
\end{proof}

\offtargetsearchiterations*
\begin{proof}
When a chain component has size 1, it means that incident edges to that singleton vertex are oriented.
Since the size of the chain components decreases by a factor of two in each phase, $\cO(\log n)$ iterations suffices.
\end{proof}

\orientinternalcliqueedgesproperties*
\begin{proof}
We first argue that \textsc{OrientInternalCliqueEdges} runs for $\cO(\log |T|) \subseteq \cO(\log n)$ while-loop iterations then argue that it incurs a cost of $\cO(\overline{\nu}^{\max}(G^*) \cdot \log n)$ in expectation per iteration.

\paragraph{Upper bounding the number of iterations}

Consider an arbitrary while-loop iteration.
Suppose the largest connected component (induced by edges of $T$) is of size $k$, where $1 \leq k \leq \sqrt{|T|}$.
When $k = 1$, all internal clique edges are oriented.
Suppose the vertices in this component are $v_1, \ldots, v_k$.

Fix an arbitrary ordering $\sigma$.
Without loss of generality, by relabelling, suppose that $\sigma(v_1) < \ldots < \sigma(v_k)$.
Under this ordering, we see that $v_1 \sim v_2, \ldots, v_{k-1} \sim v_k$ are covered edges which will be cut when we invoke \textsc{CutViaLP}.
By \cref{thm:oriented-endpoints-in-different-chain-components}, we are guaranteed that no two adjacent vertices (with respect to $\sigma$) will be in the same chain components after invoking \textsc{CutViaLP}.
To be precise, for $1 \leq i \leq k-1$, there exists an intervention where exactly one of $v_i$ and $v_{i+1}$ will be intervened upon, so the edge $v_i \sim v_{i+1}$ becomes oriented, thus $v_i$ and $v_{i+1}$ be in different chain components.

Note that any chain component of size strictly larger than $\lceil k/2 \rceil + 1$ within this size $k$ chain component includes two adjacent vertices (with respect to $\sigma$ index ordering).
This \emph{cannot} happen by the above argument, so we can conclude that the size of the largest connected component drops by a constant factor in each while-loop iteration, and thus \textsc{OrientInternalCliqueEdges} runs for $\cO(\log |T|)$ while-loop iterations.

Crucially, the resulting chain components are again a collection of disjoint cliques, so the above argument can be repeated recursively.
It is again a collection of disjoint cliques because edges exist between any pair of vertices (since $T$ was originally a set of edges that induce a collection of disjoint cliques) and oriented edges cannot have endpoints in the same chain component due to \cref{thm:oriented-endpoints-in-different-chain-components}.

\paragraph{Upper bounding the cost per iteration}

Fix an arbitrary while-loop iteration.
By \cref{thm:can-freely-orient-unoriented-consistently}, the set of covered edges $T'$ induced by the arbitrarily chosen $\sigma$ is a subset of covered edges of \emph{some} $G' \in [G^*]$.
By \cref{lem:cut-via-lp-cost}, \textsc{CutViaLP} incurs a cost of $\cO(\overline{\nu}^{\max}(G^*) \cdot \log n)$ in expectation.
\end{proof}

\exactlyonefromKH*
\begin{proof}

We will first argue that $L_H$ contains at least one vertex from $K_H$ and then argue that $L_H$ contains at most one vertex from $K_H$.

\paragraph{Contains at least one}

Since $K_H$ was a 1/2-clique separator of $H$, if $L_H$ did not contain any vertices from $K_H$, then we must have $|L_H| \leq |V(H)|/2$.
So, $L_H$ includes at least one vertex from $K_H$.

\paragraph{Contains at most one}

Suppose, for a contradiction, that $L_H$ includes more than one vertex from the clique $K_H$, say $u$ and $v$.
After invoking \textsc{OrientInternalCliqueEdges}, the edge $u \sim v$ should be oriented.
So, according \cref{thm:oriented-endpoints-in-different-chain-components}, $u$ and $v$ should not be in the same chain component.
This is a contradiction to the assumption that the chain component $L_H$ includes both $u$ and $v$.
\end{proof}

\performpartitioningiterations*
\begin{proof}
We first argue that \textsc{PerformPartitioning} will cut $u_H \sim z_{H'}$ either in Line 6 or 8.
Suppose Line 6 did not cut $u_H \sim z_{H'}$.
Then, under $\sigma'$, the edge $u_H \sim z_{H'}$ will be an unoriented covered edge and will be cut in Line 8.

Now suppose $u_H$ is still connected to some chain component $C \subseteq H'$ after Line 8.
Upon cutting $u_H \sim z_{H'}$, the edge becomes oriented and so the vertices $u_H$ and $z_{H'}$ will belong in different chain components thereafter (\cref{thm:oriented-endpoints-in-different-chain-components}), so $z_{H'} \not\in V(C)$.
After Line 6, \textsc{OrientInternalCliqueEdges} ensures that internal edges of $Z_{H'}$ are oriented.
Now, since $V(C) \setminus \{u_H\} \subseteq V(H')$ and $Z_{H'}$ was a 1/2-clique separator for $H'$, we see that $|V(C)| \leq |V(H')|/2$.
\end{proof}

Note that the \emph{number} of chain components \emph{within} $L_H$ may increase during the while-loop but we will only concern ourselves with the chain component that is still connected to $u_H$.
In other words, $L_H$ may break into multiple chain components, but at most one will contain $u_H$.
See \cref{fig:toy-example2} for an example illustration.

\begin{figure}[htb]
\centering
\begin{subfigure}[t]{0.1925\linewidth}
    \resizebox{\linewidth}{!}{%
\begin{tikzpicture}
\node[draw, circle, minimum size=15pt, inner sep=2pt] at (0,0) (d) {$d$};
\node[draw, circle, minimum size=15pt, inner sep=2pt, right=of d] (e) {$e$};
\node[draw, circle, minimum size=15pt, inner sep=2pt, right=of e] (f) {$f$};
\node[draw, circle, minimum size=15pt, inner sep=2pt, above=of d] (a) {$a$};
\node[draw, circle, minimum size=15pt, inner sep=2pt, right=of a] (b) {$b$};
\node[draw, circle, minimum size=15pt, inner sep=2pt, right=of b] (c) {$c$};
\node[draw, circle, minimum size=15pt, inner sep=2pt, below=of d] (g) {$g$};
\node[draw, circle, minimum size=15pt, inner sep=2pt, right=of g] (h) {$h$};
\node[draw, circle, minimum size=15pt, inner sep=2pt, right=of h] (i) {$i$};

\node[draw, thick, rounded corners, fit=(d)(g), white] (KH) 
{};

\draw[thick, -{Stealth[scale=1.5]}] (a) -- (b);
\draw[thick, -{Stealth[scale=1.5]}, dashed] (a) to[in=135,out=45] (c);
\draw[thick, -{Stealth[scale=1.5]}] (a) -- (d);
\draw[thick, -{Stealth[scale=1.5]}] (b) -- (d);
\draw[thick, -{Stealth[scale=1.5]}, dashed] (c) -- (b);
\draw[thick, -{Stealth[scale=1.5]}, dashed] (c) -- (d);
\draw[thick, -{Stealth[scale=1.5]}] (d) -- (e);
\draw[thick, -{Stealth[scale=1.5]}] (d) -- (g);
\draw[thick, -{Stealth[scale=1.5]}] (d) -- (h);
\draw[thick, -{Stealth[scale=1.5]}] (d) -- (i);
\draw[thick, -{Stealth[scale=1.5]}] (e) -- (f);
\draw[thick, -{Stealth[scale=1.5]}, dashed] (e) -- (h);
\draw[thick, -{Stealth[scale=1.5]}] (h) -- (i);
\end{tikzpicture}
}
    \vspace{-17pt}
    \caption{Covered edges $C(G^*)$ are dashed and $\mathcal{E}(G^*) = \textrm{skel}(G^*)$.}
    \label{fig:appendix-G-star}
\end{subfigure}
\quad
\begin{subfigure}[t]{0.24\linewidth}
    \resizebox{\linewidth}{!}{%
\begin{tikzpicture}
\node[draw, circle, minimum size=15pt, inner sep=2pt] at (0,0) (d) {$d$};
\node[draw, circle, minimum size=15pt, inner sep=2pt, right=of d] (e) {$e$};
\node[draw, circle, minimum size=15pt, inner sep=2pt, right=of e] (f) {$f$};
\node[draw, circle, minimum size=15pt, inner sep=2pt, above=of d] (a) {$a$};
\node[draw, circle, minimum size=15pt, inner sep=2pt, right=of a] (b) {$b$};
\node[draw, circle, minimum size=15pt, inner sep=2pt, right=of b] (c) {$c$};
\node[draw, circle, minimum size=15pt, inner sep=2pt, below=of d] (g) {$g$};
\node[draw, circle, minimum size=15pt, inner sep=2pt, right=of g] (h) {$h$};
\node[draw, circle, minimum size=15pt, inner sep=2pt, right=of h] (i) {$i$};

\node[draw, thick, rounded corners, fit=(d)(g), white] (KH) {};
\node[red, left=5pt of KH] {$L_H$};

\draw[thick] (a) -- (b);
\draw[thick] (a) to[in=135,out=45] (c);
\draw[thick] (a) -- (d);
\draw[thick] (b) -- (d);
\draw[thick] (c) -- (b);
\draw[thick] (c) -- (d);
\draw[thick] (d) -- (e);
\draw[thick] (d) -- (h);
\draw[thick] (d) -- (i);
\draw[thick] (e) -- (f);
\draw[thick] (e) -- (h);
\draw[thick] (h) -- (i);

\node[green!50!black, above=5pt of c] (H1-prime) {$H'_1$};
\node[draw, thick, rounded corners, fit={($(a.west) + (0,-0.5)$)(b)(c)(H1-prime)}, green!50!black] {};
\node[orange] at ($(i) + (-0.25,0.5)$) (H2-prime) {$H'_2$};
\draw[thick, rounded corners, orange] ($(e.north) + (-0.5,0.25)$) -- ($(h.south) + (-0.5,-0.25)$) -- ($(i) + (0.5,-0.5)$) -- ($(i) + (0.5,0.25)$) -- ($(e.north) + (0.5,0.25)$) -- ($(e.north) + (-0.5,0.25)$);

\node[draw, thick, rounded corners, fit=(b)(c), purple] {};
\node[draw, thick, rounded corners, fit=(e)(h), purple] {};
\end{tikzpicture}
}
    \vspace{-19pt}
    \caption{If $\{g\}$ was intervened upon, the resulting chain component $L_H$ is ``large''.}
    \label{fig:appendix-off-target-intervention}
\end{subfigure}
\quad
\begin{subfigure}[t]{0.24\linewidth}
    \resizebox{\linewidth}{!}{%
\begin{tikzpicture}
\node[draw, circle, minimum size=15pt, inner sep=2pt] at (0,0) (d) {$d$};
\node[draw, circle, minimum size=15pt, inner sep=2pt, right=of d] (e) {$e$};
\node[draw, circle, minimum size=15pt, inner sep=2pt, right=of e] (f) {$f$};
\node[draw, circle, minimum size=15pt, inner sep=2pt, above=of d] (a) {$a$};
\node[draw, circle, minimum size=15pt, inner sep=2pt, right=of a] (b) {$b$};
\node[draw, circle, minimum size=15pt, inner sep=2pt, right=of b] (c) {$c$};
\node[draw, circle, minimum size=15pt, inner sep=2pt, below=of d] (g) {$g$};
\node[draw, circle, minimum size=15pt, inner sep=2pt, right=of g] (h) {$h$};
\node[draw, circle, minimum size=15pt, inner sep=2pt, right=of h] (i) {$i$};

\node[draw, thick, rounded corners, fit=(d)(g), white] (KH) {};
\node[white, left=5pt of KH] {$L_H$};

\draw[thick, -{Stealth[scale=1.5]}, red] (a) -- (b);
\draw[thick] (a) to[in=135,out=45] (c);
\draw[thick, -{Stealth[scale=1.5]}, red] (a) -- (d);
\draw[thick] (b) -- (d);
\draw[thick, -{Stealth[scale=1.5]}, red] (c) -- (b);
\draw[thick, -{Stealth[scale=1.5]}, red] (c) -- (d);
\draw[thick, -{Stealth[scale=1.5]}, blue] (d) -- (e);
\draw[thick, -{Stealth[scale=1.5]}, blue] (d) -- (h);
\draw[thick, -{Stealth[scale=1.5]}, blue] (d) -- (i);
\draw[thick, -{Stealth[scale=1.5]}, blue] (e) -- (f);
\draw[thick] (e) -- (h);
\draw[thick] (h) -- (i);

\node[green!50!black, above=5pt of c] (H1-prime) {$H'_1$};
\node[draw, thick, rounded corners, fit={($(a.west) + (0,-0.5)$)(b)(c)(H1-prime)}, green!50!black] {};

\node[draw, thick, rounded corners, fit=(a), red] {};
\node[draw, thick, rounded corners, fit=(c), red] {};
\end{tikzpicture}
}
    \vspace{-19pt}
    \caption{Intervening on $\{a,c\}$ orients the red edges, then Meek rules orient the blue edges.}
    \label{fig:appendix-off-target-intervention2}
\end{subfigure}
\quad
\begin{subfigure}[t]{0.24\linewidth}
    \resizebox{\linewidth}{!}{%
\begin{tikzpicture}
\node[draw, circle, minimum size=15pt, inner sep=2pt] at (0,0) (d) {$d$};
\node[draw, circle, minimum size=15pt, inner sep=2pt, right=of d] (e) {$e$};
\node[draw, circle, minimum size=15pt, inner sep=2pt, right=of e] (f) {$f$};
\node[draw, circle, minimum size=15pt, inner sep=2pt, above=of d] (a) {$a$};
\node[draw, circle, minimum size=15pt, inner sep=2pt, right=of a] (b) {$b$};
\node[draw, circle, minimum size=15pt, inner sep=2pt, right=of b] (c) {$c$};
\node[draw, circle, minimum size=15pt, inner sep=2pt, below=of d] (g) {$g$};
\node[draw, circle, minimum size=15pt, inner sep=2pt, right=of g] (h) {$h$};
\node[draw, circle, minimum size=15pt, inner sep=2pt, right=of h] (i) {$i$};

\node[draw, thick, rounded corners, fit=(d)(g), white] (KH) {};
\node[white, left=5pt of KH] {$L_H$};

\draw[thick] (a) to[in=135,out=45] (c);
\draw[thick] (b) -- (d);
\draw[thick] (e) -- (h);
\draw[thick] (h) -- (i);

\node[green!50!black, above=5pt of c] (H1-prime) {$H'_1$};
\node[draw, thick, rounded corners, fit={($(a.west) + (0,-0.5)$)(b)(c)(H1-prime)}, green!50!black] {};

\end{tikzpicture}
}
    \vspace{-19pt}
    \caption{The resulting chain component after removing oriented edges.}
    \label{fig:appendix-off-target-intervention3}
\end{subfigure}
\caption{
Recall the moral DAG $G^*$ on $n=9$ nodes given in \cref{fig:toy-example}.
Suppose we oriented $d \sim g$ through an intervention on $\{g\}$ due to action stochasticity in \cref{fig:appendix-off-target-intervention}.
The resulting chain component {\color{red}$L_H$} after intervening on $\{g\}$ is \emph{large} since $|V(L_H)| = |\{a,b,c,d,e,f,h,i\}| = 8 > n/2$.
\textsc{PerformPartitioning} (\cref{alg:perform-partitioning}) breaks {\color{red}$L_H$} up by trying to intervene within $L_H[V(L_H) \cap N({\color{red}u_H})]$, where $d \equiv {\color{red}u_H}$ is the unique vertex from {\color{red}$K_H$} within $V({\color{red}K_H}) \cap V({\color{red}L_H})$.
There are two chain components {\color{green!50!black}$H'_1$} and {\color{orange}$H'_2$} in the induced subgraph $L_H[V(L_H) \cap N({\color{red}u_H})]$.
\textbf{Let us focus our remaining discussion on $H'_1$.}
If we pick 1/2-clique separator {\color{purple}$Z'_1 = \{b,c\}$} for {\color{green!50!black}$H'_1$}, then $c \equiv {\color{purple}z'_1}$ is the source of {\color{purple}$Z'_1$}.
Now, suppose in \cref{fig:appendix-off-target-intervention2} while trying to orient $b \sim c$, we intervened on $\{a,c\}$.
Then, red edges $\{a \to b, a \to d, c \to b, c \to d\}$ will be cut and oriented which then triggers Meek R2 to orient the blue edges $\{d \to e, d \to h, d \to i, e \to f\}$.
\cref{fig:appendix-off-target-intervention3} illustrates the resulting chain component without the newly oriented edges.
Observe that $u_H \not\sim z_{H'_1} \equiv d \not\sim c$ as expected while $u_H$ is still connected to a chain compoment $C = \{b\} \subseteq V(H'_1)$.
As proven in \cref{lem:perform-partitioning-iterations}, we have $|V(C)| = |\{b\}| = 1 \leq 1.5 = |V(H')|/2$.
Although $H'_1$ now has two components $\{b\}$ and $\{a,c\}$, Line 9 of \textsc{PerformPartitioning} will restrict $H'_1$ to just $\{b\}$ going forward.
}
\label{fig:toy-example2}
\end{figure}
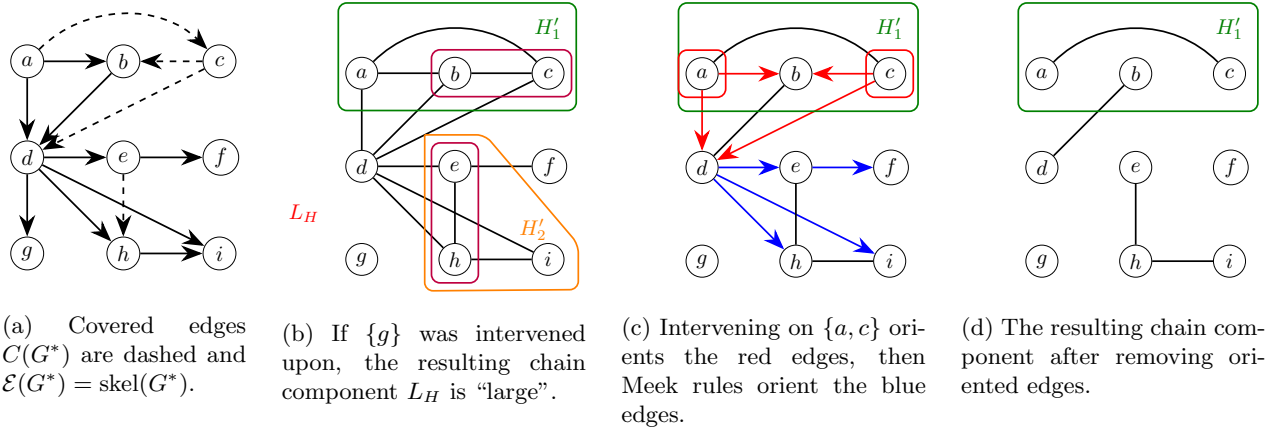

\performpartitioningguarantees*
\begin{proof}
We first argue that size indeed drops by a factor of two before upper bounding cost incurred.

\textbf{Correctness}
Since $K_H$ is a 1/2-clique separator for $H$, the resulting chain components have size at most $|V(H)|/2$ if we manage to orient all edges incident to $K_H$.

At the end of \textsc{OrientInternalCliqueEdges}, all edges \emph{within} the $K_H$'s will be oriented and there is at most one ``large'' chain component $L_H$ of size strictly larger than $|V(H)|/2$ containing a vertex $u_H \in V(K_H)$; see \cref{lem:exactly-one-from-KH}.

Consider any chain component $H' \in CC(L_H[V(L_H) \cap N(u_H)])$.
\begin{itemize}
    \item If $u_H \in V(L_H) \cap V(K_H)$ is no longer connected to any chain component of $H'$ after Line 8, then we know that $|V(H')|$ is now ``small''.
    \item Meanwhile, if $u_H$ was still connected to some chain component $C \subseteq H'$ after Line 8, then \cref{lem:perform-partitioning-iterations} tells us that $|V(C)| \leq |V(H')|/2$ and Line 9 restricts $H'$ to $C$.
    So, after $\cO(\log n)$ while-loop iterations, $|V(H')| \leq 1$.
    \item In the case that $H'$ is a singleton $\{z_{H'}\}$, e.g.\ when $L_H$ is a large star with $u_H$ at the center, the edge $u_H \sim z_{H'}$ will be cut on Line 8.
\end{itemize}

\textbf{Upper bounding the cost incurred}

In each while-loop iteration, we invoke \textsc{OrientInternalCliqueEdges} once and \textsc{CutViaLP} once.
By \cref{lem:orient-internal-clique-edges-properties}, \textsc{OrientInternalCliqueEdges} incurs a cost of $\cO(\overline{\nu}^{\max}(G^*) \cdot \log^2 n)$.
By \cref{lem:cut-via-lp-cost}, \textsc{CutViaLP} incurs a cost of $\cO(\overline{\nu}^{\max}(G^*) \cdot \log n)$ in expectation since we always invoke \textsc{CutViaLP} with a subset of covered edges of some $G' \in [G^*]$.
\end{proof}

\searchnumaxupperbound*
\begin{proof}
Combine \cref{lem:off-target-search-iterations} and \cref{lem:perform-partitioning-guarantees}.
\end{proof}

\section{Experiments}
\label{sec:appendix-experiments}

While our main contributions are theoretical, we also implemented our algorithms and performed some experiments.
All experiments were run on a laptop with Apple M1 Pro chip and 16GB of memory.
Our source code and experimental scripts are available at \url{https://github.com/cxjdavin/causal-discovery-under-off-target-interventions}.

\subsection{Executive summary}

An instance is defined by an underlying ground truth DAG $G^*$ and $k$ actions $A_1, \ldots, A_k$ with corresponding interventional distributions $\cD_1, \ldots, \cD_k$.
We tested on both synthetic and real-world graphs and 3 different classes of interventional distributions; see \cref{sec:experiment-graphs} and \cref{sec:experiment-interventional-distributions} for details.

We compared against 4 baselines: \texttt{Random}, \texttt{One-shot}, \texttt{Coloring}, \texttt{Separator}; see \cref{sec:experiment-algorithms} for details.
\texttt{One-shot} tries to emulate non-adaptive interventions while the last two are state-of-the-art on-target search algorithms adapted to the off-target setting.
As \texttt{Coloring} and \texttt{Separator} were designed specifically for unweighted settings, we test using uniform cost actions despite our off-target search algorithm being able to work with non-uniform action costs.
We also plotted the optimal value of \eqref{eq:VLP} for comparison.

Qualitatively, \texttt{Random} and \texttt{One-shot} perform visibly worse than the others.
While the adapted on-target algorithms may empirically outperform \texttt{Off-Target} sometimes, we remark that our algorithm has provable guarantees even for non-uniform action costs and it is designed to handle worst-case off-target instances.
Since we do not expect real-world causal graphs to be adversarial, it is unsurprising to see that our algorithm performs similarly to \texttt{Coloring} and \texttt{Separator}.

\begin{remark}
To properly evaluate adaptive algorithms, one would need data corresponding to all the interventions that these algorithms intend to perform.
Therefore, in addition to observational data, any experimental dataset to evaluate these algorithms should contain interventional data for all possible interventions.
Unfortunately, such real world datasets do not currently exist and thus the state-of-the-art adaptive search algorithms still use synthetic experiments to evaluate their performances.
To slightly mitigate a possible concern of synthetic graphs, we use real-world DAGs from \texttt{bnlearn} \cite{scutari2010learning} as our ground truth DAGs $G^*$s.
\end{remark}

\subsection{Graph instances}
\label{sec:experiment-graphs}

We tested on synthetic \texttt{GNP\_TREE} graphs \cite{choo2023subset} of various sizes, and on some real-world graphs from \texttt{bnlearn} \cite{scutari2010learning}.
We associate a unit-cost action $A_v$ to each vertex $v \in V$ of the input graph.

\subsubsection{Synthetic graphs}

For given $n$ and $p$ parameters, the moral \texttt{GNP\_TREE} graphs are generated in the following way\footnote{Description from Appendix F.1.1 of \cite{choo2023subset}.}:
\begin{itemize}
    \item Generate a random Erdos-Renyi graph $G(n,p)$.
    \item Generate a random tree on $n$ nodes.
    \item Combine their edgesets and orient the edges in an acyclic fashion: orient $u \to v$ whenever vertex $u$ has a smaller vertex numbering than $v$.
    \item Add arcs to remove v-structures: for every v-structure $u \to v \gets w$ in the graph, we add the arc $u \to w$ whenever vertex $u$ has a smaller vertex numbering from $w$.
\end{itemize}
We generated \texttt{GNP\_TREE} graphs with $n \in \{10, 20, 30, 40, 50\}$ and $p = 0.1$.
For each $(n,p)$ setting, we generated 10 such graphs.

\subsubsection{Real-world graphs}

The \texttt{bnlearn} \cite{scutari2010learning} graphs are available at \url{https://www.bnlearn.com/bnrepository/}.
In particular, we used the graphical structure of the \texttt{Discrete Bayesian Networks} for all sizes: ``Small Networks ($< 20$ nodes)'', ``Medium Networks ($20-50$ nodes)'', ``Large Networks ($50-100$ nodes)'', and ``Very Large Networks ($100-1000$ nodes)'', and ``Massive Networks ($>1000$ nodes)''.
Some graphs such as ``pigs'', ``cancer'', ``survey'', ``earthquake'', and ``mildew'' already have fully oriented essential graphs and are thus excluded from the plots as they do not require any interventions.

\subsection{Interventional distributions}
\label{sec:experiment-interventional-distributions}

In our experiments, we associated each vertex $v$ with unit cost and an action $A_v$ with four different possible types of interventional distributions (see below).
The first two are atomic in nature (all actions return a single intervened vertex) while the third is slightly more complicated interventional distribution where multiple vertices may be intervened upon.
Atomic interventional distributions enables a simple way to compute the probability that edge $\{u,v\}$ is cut by action $A_i$: it is simply $p^i_u + p^i_v$, where $p^i_v$ is the probability that $v$ is intervened upon when we perform action $A_i$.

The 3 classes of off-target interventions we explored are as follows:
\begin{description}
    \item[$r$-hop]
        When taking action $A_v$, $D_v$ samples a uniform random vertex from the closed $r$-hop neighborhood of $v$, including $v$.
    \item[Decaying with parameter $\alpha$]
        When taking action $A_v$, $D_v$ samples a random vertex from a weighted probability distribution $w$ obtained by normalizing the following weight vector: assign weight $\alpha^r$ for all vertices exactly $r$-hops from $v$, where $v$ itself has weight 1.
        So, vertices closer to $v$ have higher chance of being intervened upon when we attempt to intervene on $v$.
    \item[Fat hand with parameter $p$]
        When taking action $A_v$, $D_v$ will always intervene on $v$, but will additionally intervene on $v$'s neighbors, each with independent probability $p$.
        Note that the probability of cutting an edge $\{u,v\}$ now is no longer a simple sum of two independent probabilities, but it is still relatively easy to compute in closed-form.
\end{description}

In our experiments, we tested the following 6 settings:
\begin{enumerate}
    \item $r$-hop with $r = 1$
    \item $r$-hop with $r = 2$
    \item Decaying with $\alpha = 0.5$
    \item Decaying with $\alpha = 0.9$
    \item Fat hand with $p = 0.5$
    \item Fat hand with $p = 0.9$
\end{enumerate}

\subsection{Algorithms}
\label{sec:experiment-algorithms}

Since our off-target intervention setting has not been studied before from an algorithmic perspective, there is no suitable prior algorithms to compare against.
As such, we propose the following baselines:
\begin{description}
    \item[Random]
        Repeatedly sample actions uniformly at random until the entire DAG is oriented.
        This is a natural naive baseline to compare against.
    \item[One-shot]
        Solve our linear program \eqref{eq:VLP} in the paper on all unoriented edges.
        Intepret the optimal vector $X$ of VLP as a probability distribution $p$ over the actions and sample actions according to $p$ until all the unoriented edges are oriented.
        One-shot aims to simulate non-adaptive algorithms in the context of off-target interventions: while it can optimally solve \eqref{eq:VLP} (c.f.\ compute graph separating system), One-shot cannot update its knowledge based on arc orientations that are subsequently revealed.
    \item[Coloring \textrm{and} Separator]
        Two state-of-the-art adaptive on-target intervention algorithms in the literature: Separator \cite{choo2022verification} and Coloring \cite{shanmugam2015learning}.
        As these algorithms are not designed for off-target intervention, we need to orient all the edges incident to $v$ to simulate an on-target intervention at $v$.
        To do so, we run \eqref{eq:VLP} on the unoriented edges incident to $v$ and interpret the optimal vector $X$ of \eqref{eq:VLP} as a probability distribution $p$ over the actions, then sample actions according to $p$ until all the unoriented edges incident to $v$ are oriented.
        Note that this modification provides a generic way to convert any usual intervention algorithm to the off-target setting.
\end{description}

\subsection{Experimental plots}

For each combination of graph instance and interventional distribution, we ran $10$ times and plotted the average with standard deviation error bars.
This is because there is inherent randomness involved when we attempt to perform an intervention.
For synthetic graphs, we also aggregated the performance over all graphs with the same number of nodes $n$ in hopes of elucidating trends with respect to the size of the graph.
As the naive baseline Random incurs significantly more cost than the others, we also plotted all experiments without it.

\newpage
\subsubsection{Plots (without ``random'')}

\begin{figure}[H]
\centering
\begin{subfigure}[t]{0.45\linewidth}
\includegraphics[width=\linewidth]{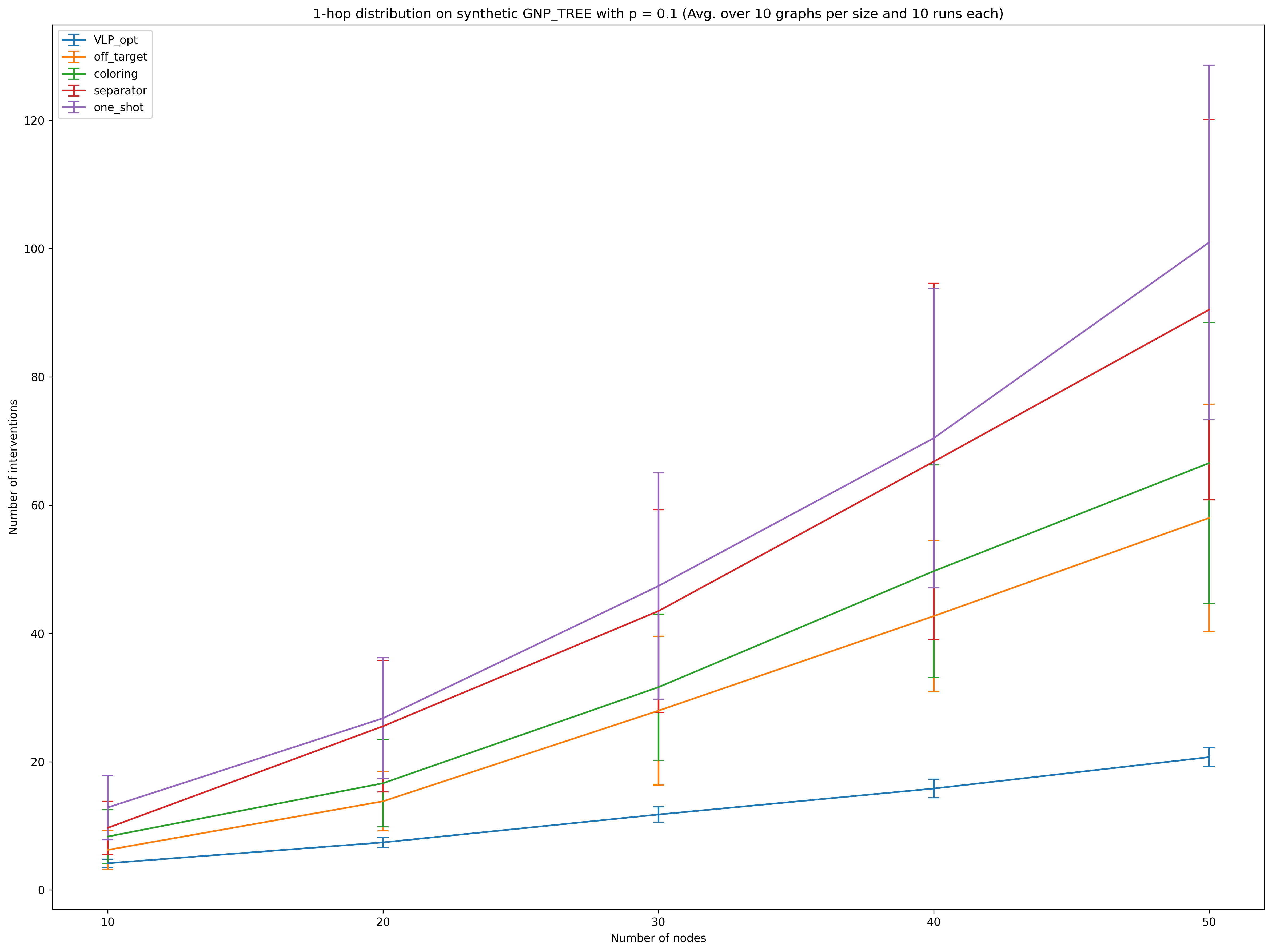}
\caption{1-hop}
\label{fig:gnp-tree-no-random-1-hop}
\end{subfigure}
\quad
\begin{subfigure}[t]{0.45\linewidth}
\includegraphics[width=\linewidth]{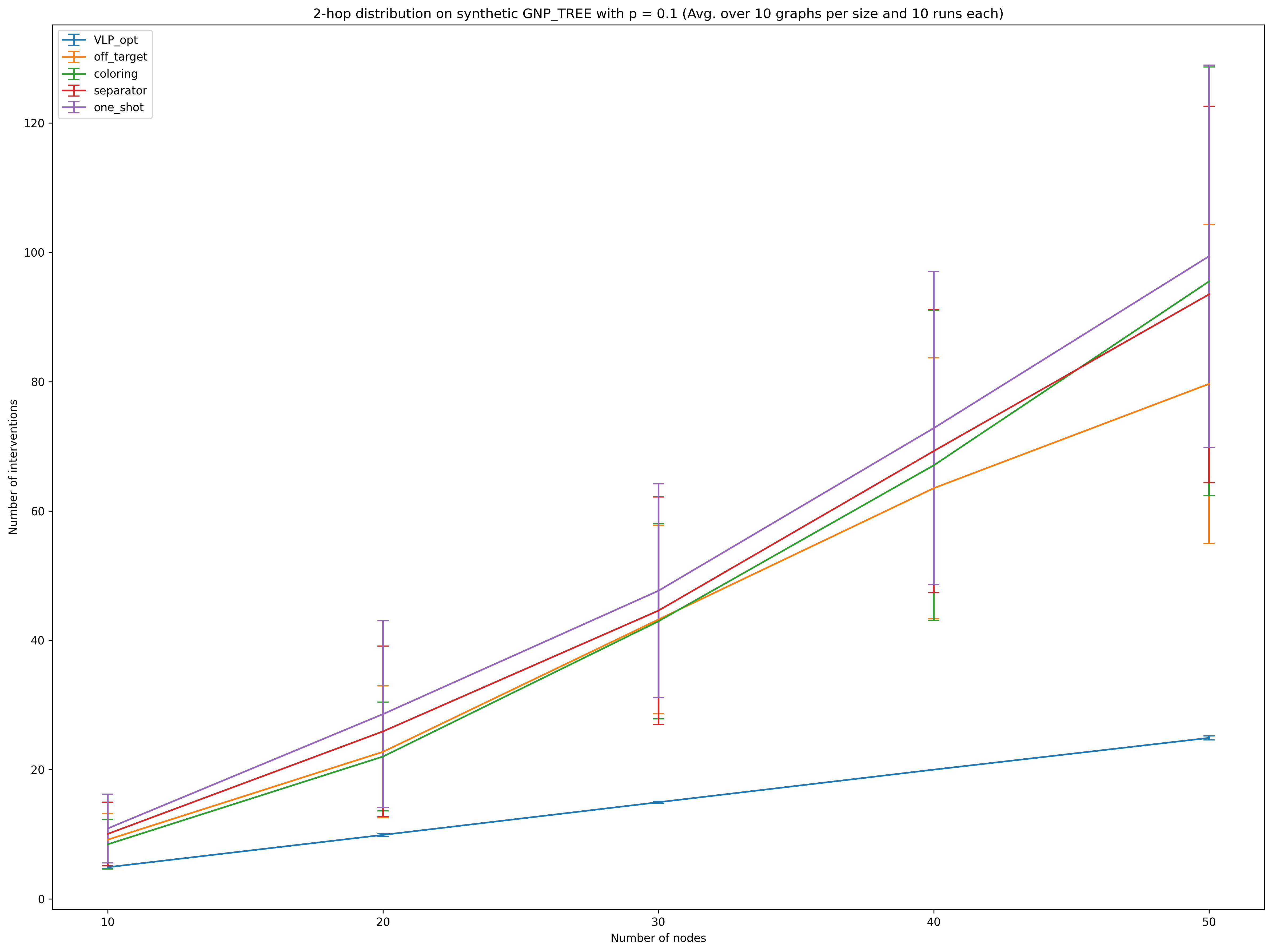}
\caption{2-hop}
\label{fig:gnp-tree-no-random-2-hop}
\end{subfigure}
\\
\begin{subfigure}[t]{0.45\linewidth}
\includegraphics[width=\linewidth]{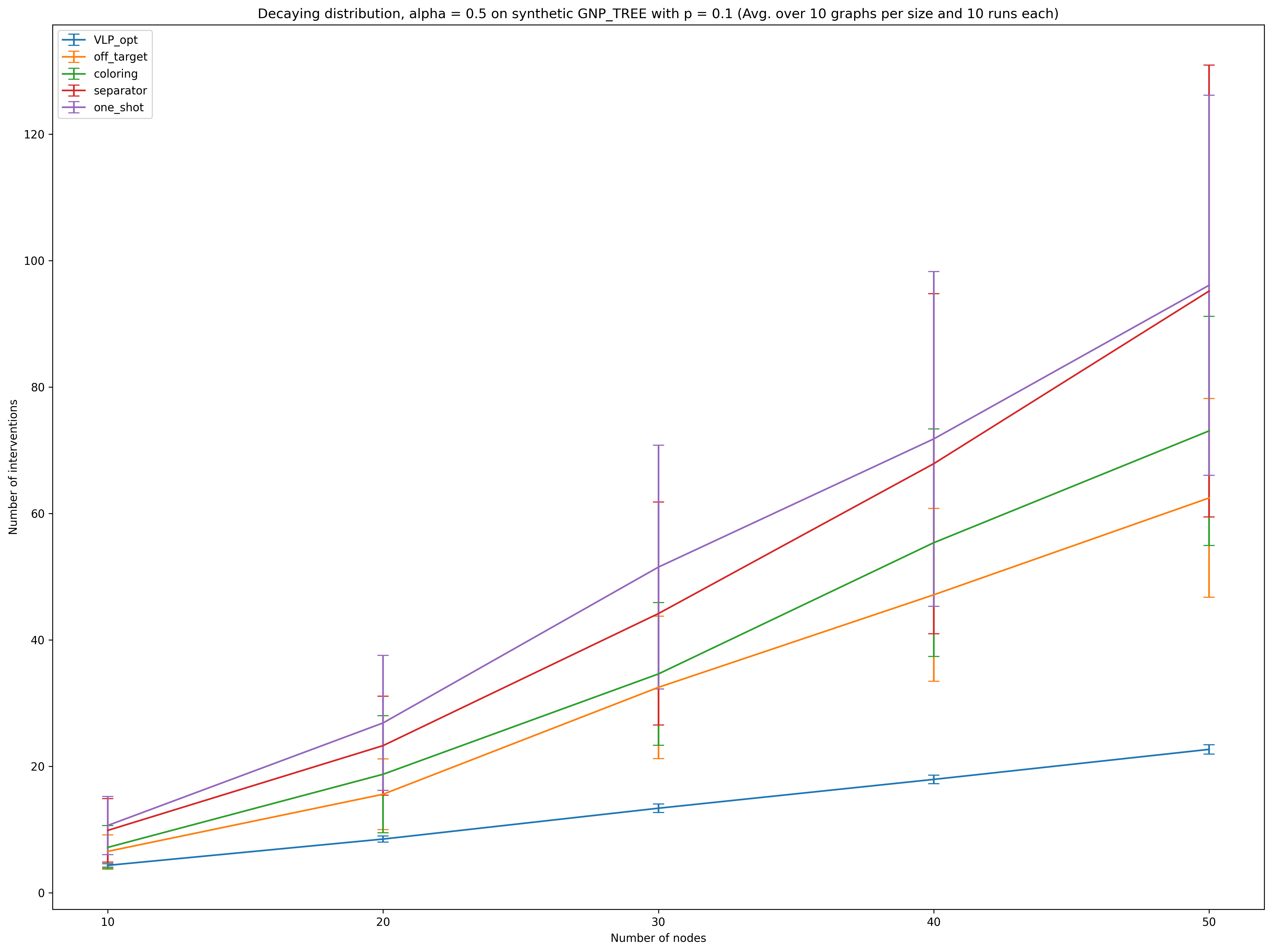}
\caption{Decaying, $\alpha = 0.5$}
\label{fig:gnp-tree-no-random-decaying-0.5}
\end{subfigure}
\quad
\begin{subfigure}[t]{0.45\linewidth}
\includegraphics[width=\linewidth]{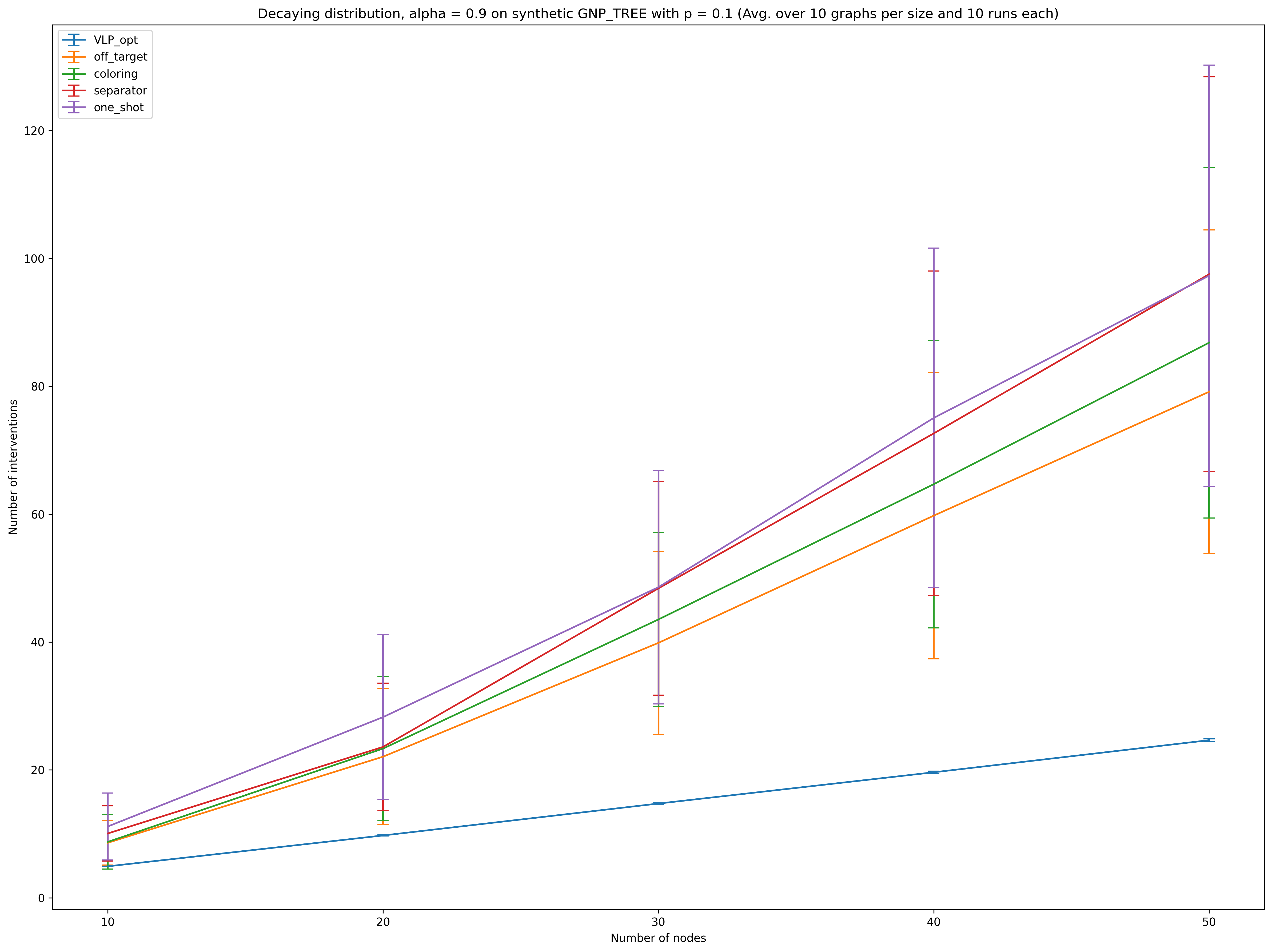}
\caption{Decaying, $\alpha = 0.9$}
\label{fig:gnp-tree-no-random-decaying-0.9}
\end{subfigure}
\\
\begin{subfigure}[t]{0.45\linewidth}
\includegraphics[width=\linewidth]{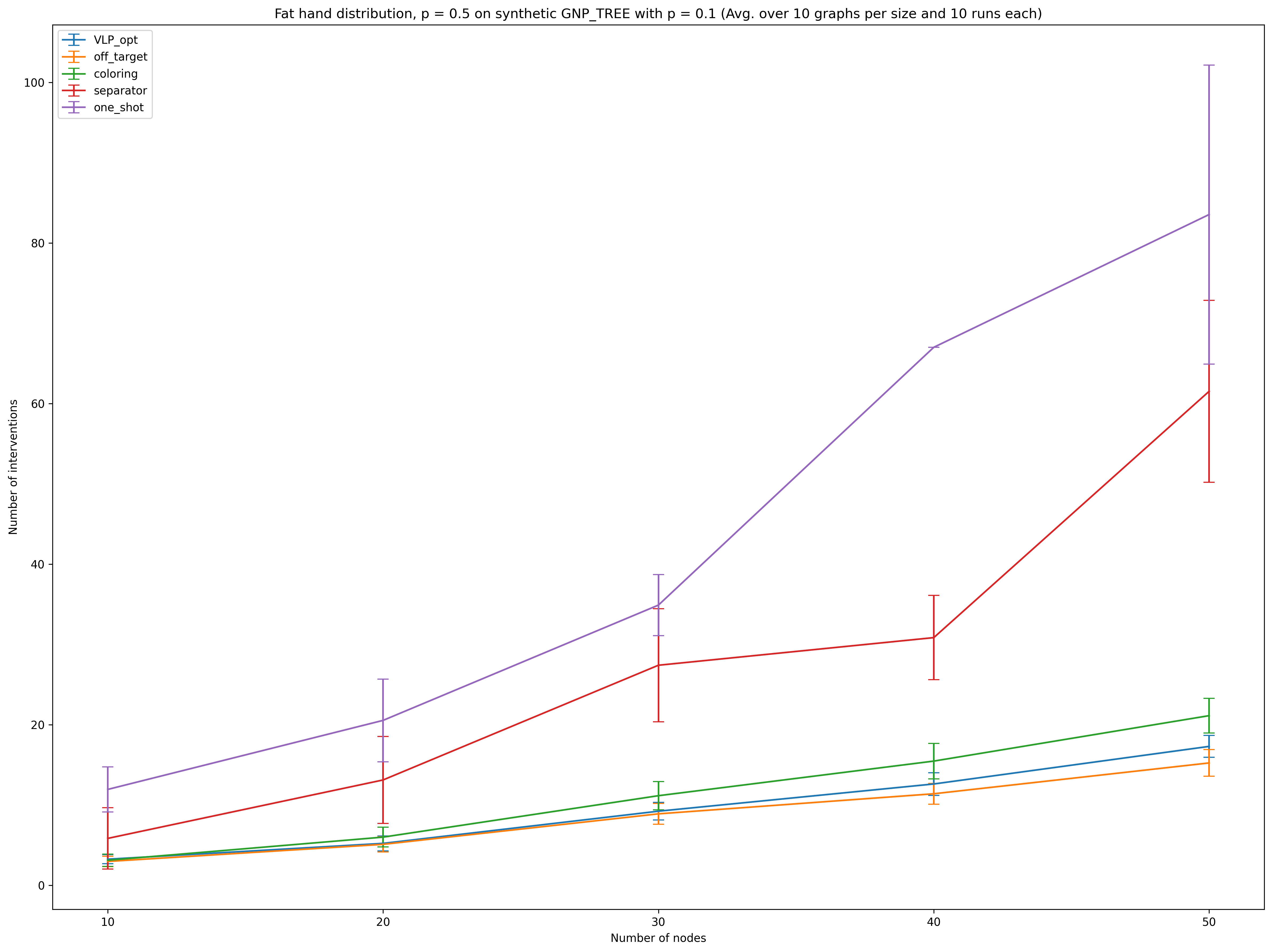}
\caption{Fat hand, $p = 0.5$}
\label{fig:gnp-tree-no-random-fat-hand-0.5}
\end{subfigure}
\quad
\begin{subfigure}[t]{0.45\linewidth}
\includegraphics[width=\linewidth]{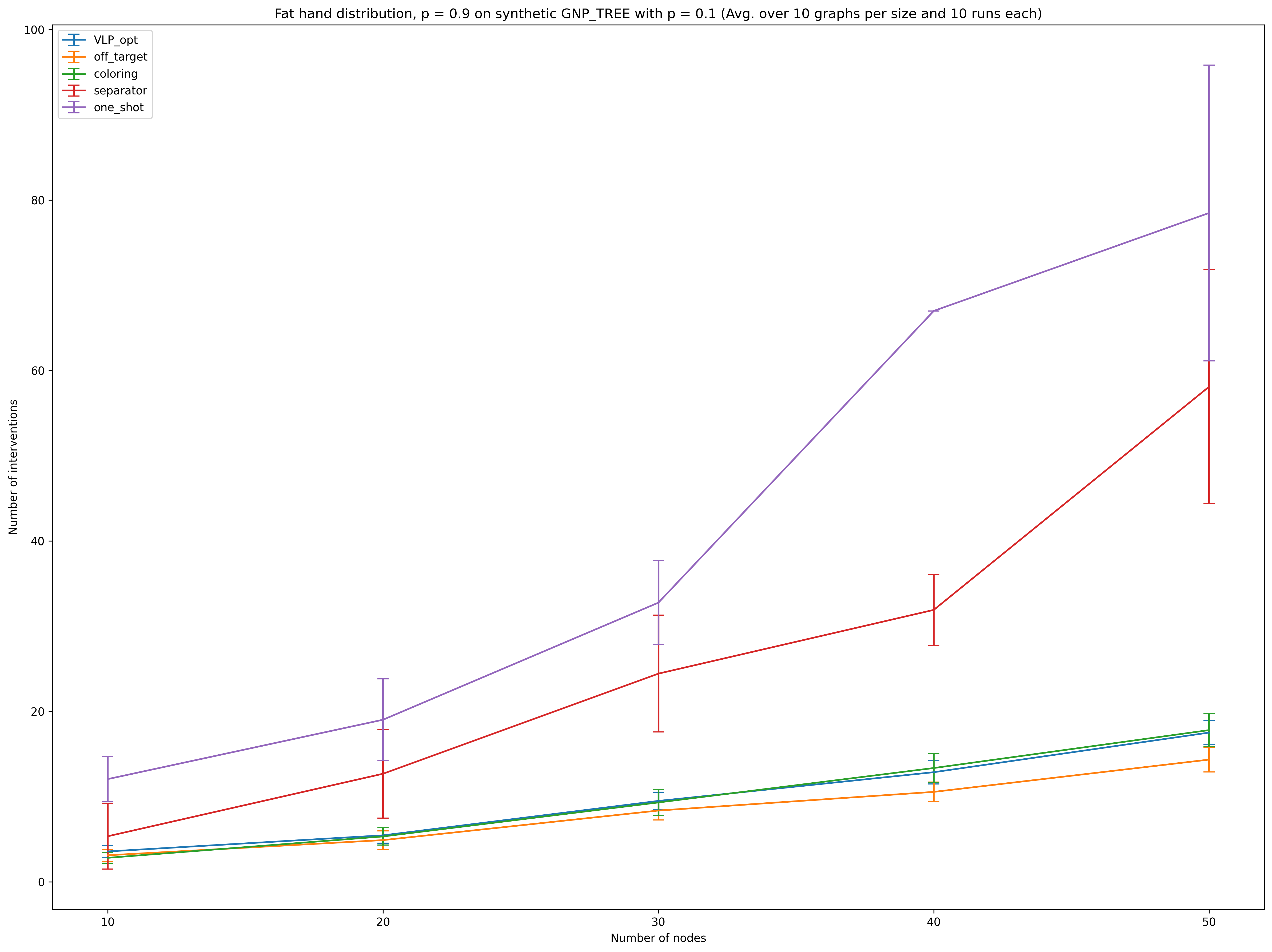}
\caption{Fat hand, $p = 0.9$}
\label{fig:gnp-tree-no-random-fat-hand-0.9}
\end{subfigure}
\caption{\texttt{GNP\_TREE} graphs without ``random''.
The optimal value of VLP in {\color{blue}blue} is an $\cO(\log n)$ approximation of $\nu(G^*)$.
Our off-target search \texttt{Off-Target} is in {\color{orange}orange}.
\texttt{Coloring} is in {\color{green!50!black}green}.
\texttt{Separator} is in {\color{red}red}.
\texttt{One-shot} is in {\color{purple!50!blue}purple}.
}
\label{fig:gnp-tree-no-random}
\end{figure}

\begin{figure}[H]
\centering
\begin{subfigure}[t]{0.45\linewidth}
\includegraphics[width=\linewidth]{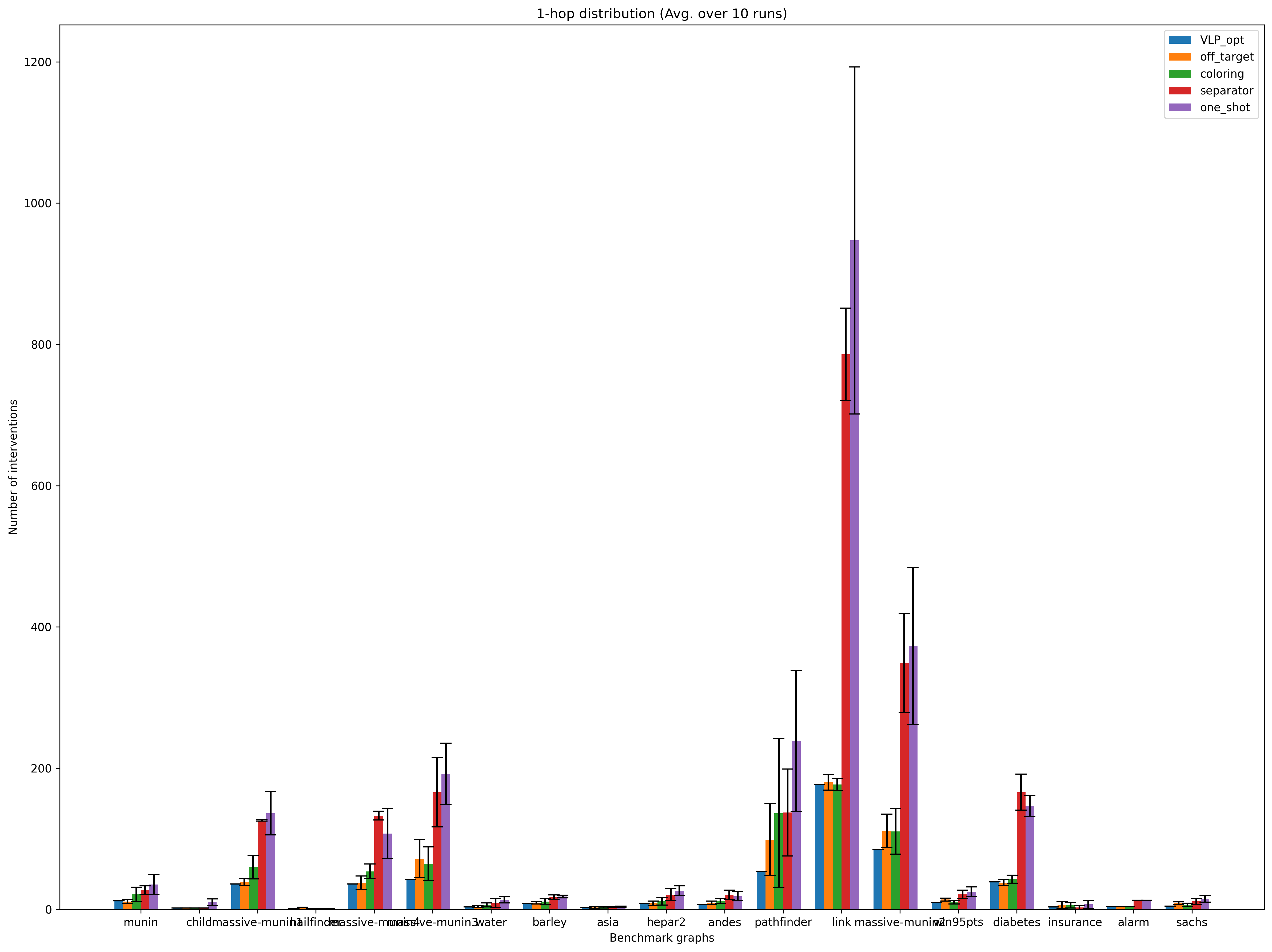}
\caption{1-hop}
\label{fig:bnlearn-no-random-1-hop}
\end{subfigure}
\quad
\begin{subfigure}[t]{0.45\linewidth}
\includegraphics[width=\linewidth]{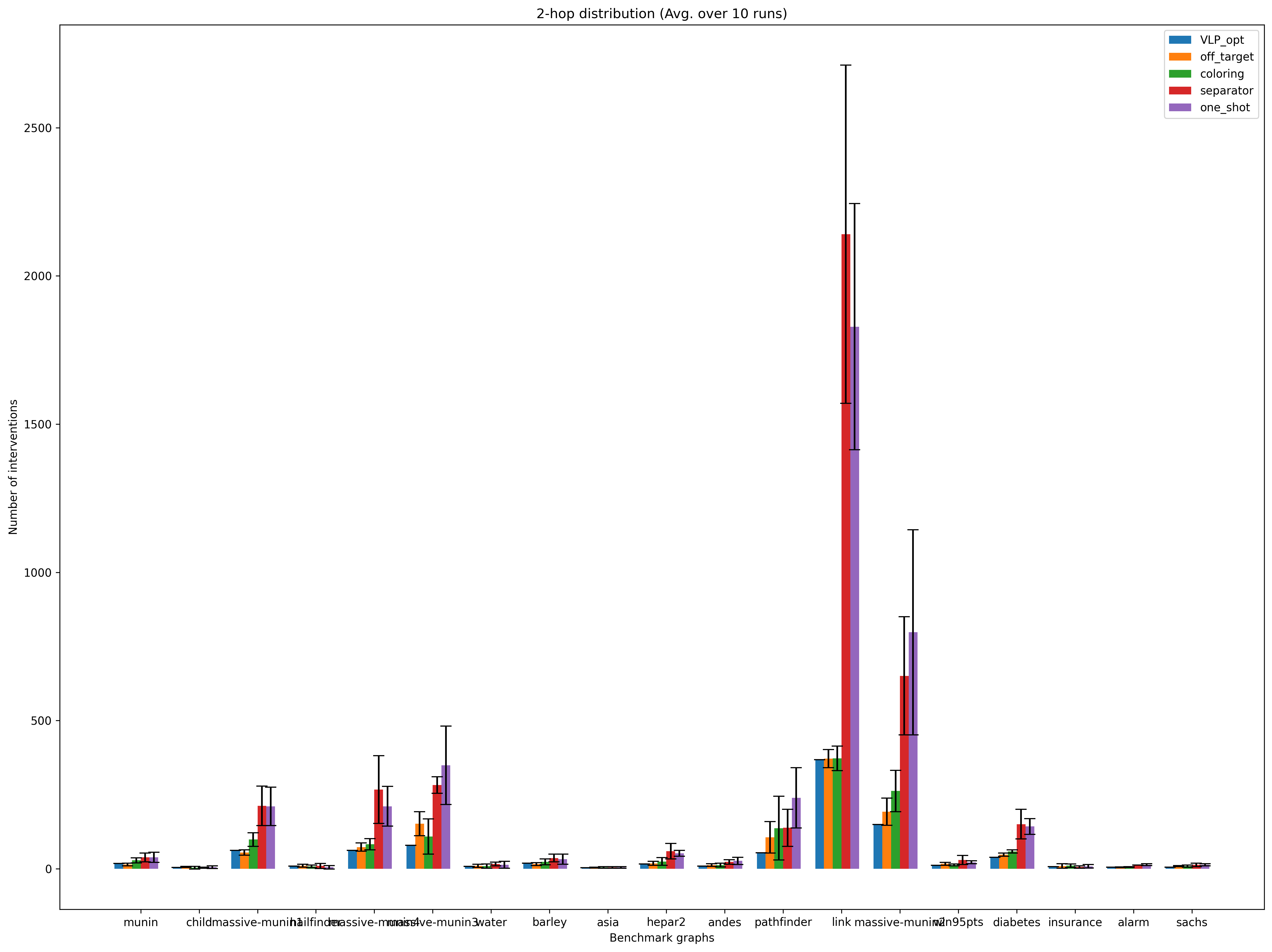}
\caption{2-hop}
\label{fig:bnlearn-no-random-2-hop}
\end{subfigure}
\\
\begin{subfigure}[t]{0.45\linewidth}
\includegraphics[width=\linewidth]{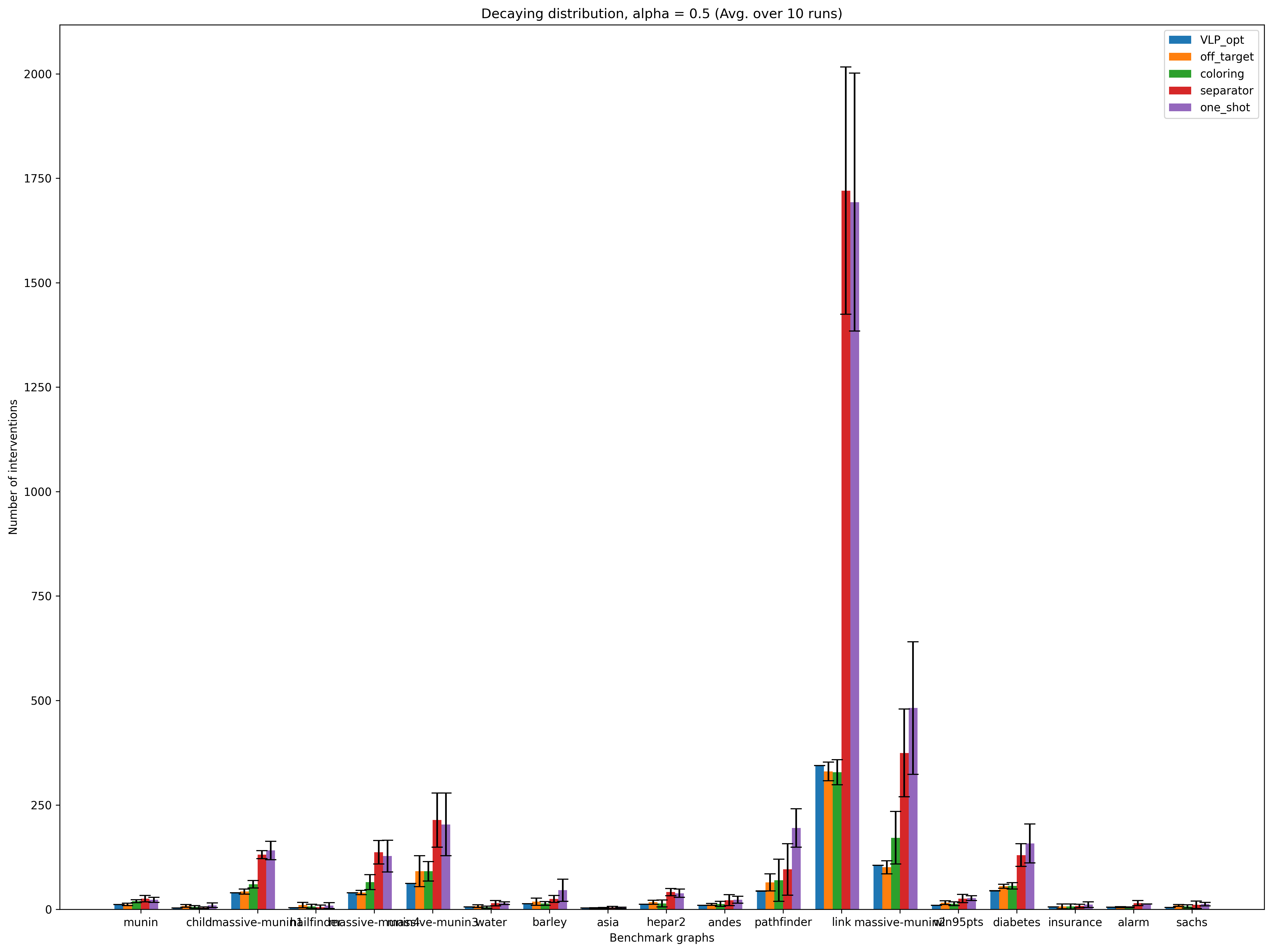}
\caption{Decaying, $\alpha = 0.5$}
\label{fig:bnlearn-no-random-decaying-0.5}
\end{subfigure}
\quad
\begin{subfigure}[t]{0.45\linewidth}
\includegraphics[width=\linewidth]{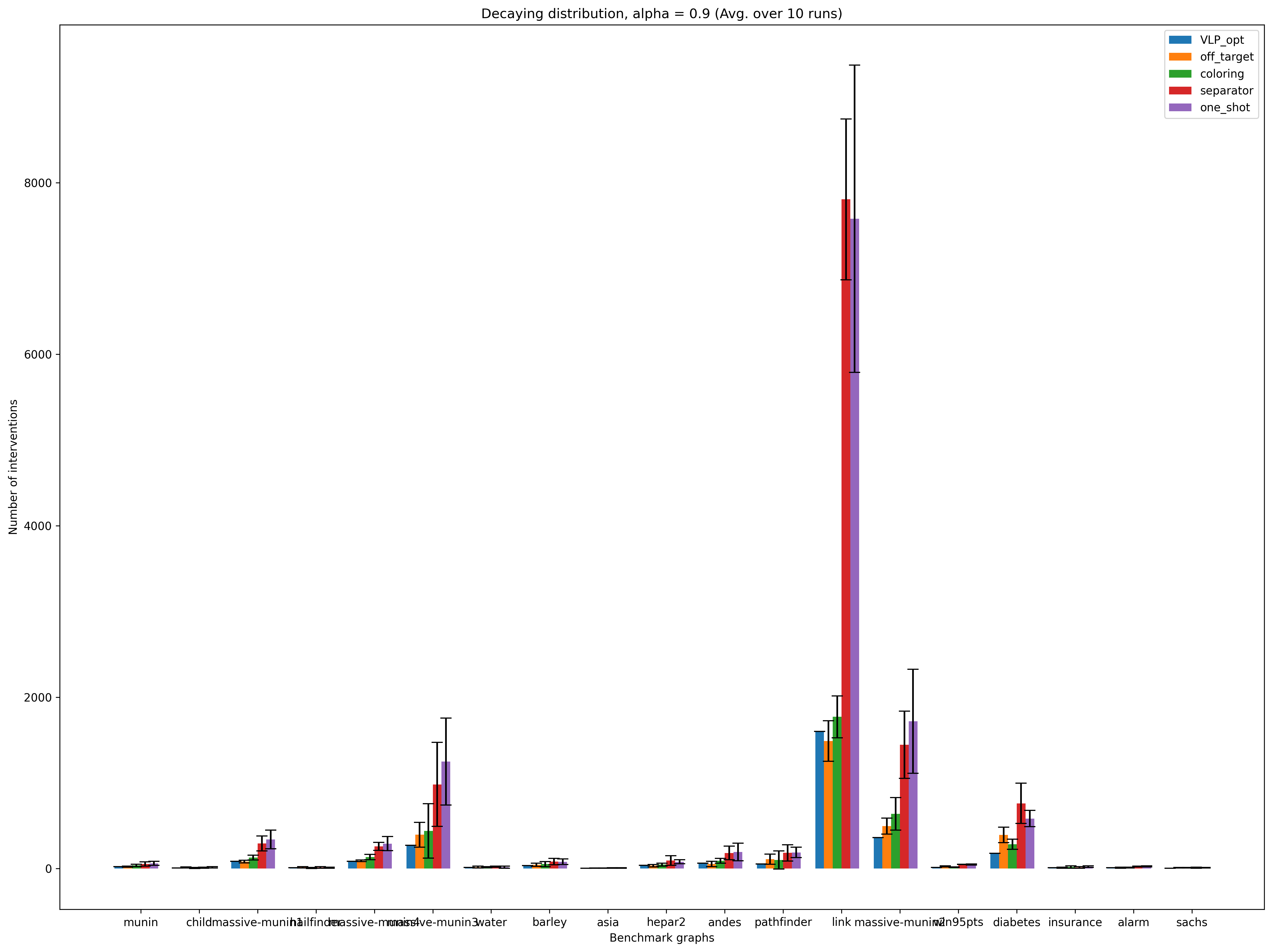}
\caption{Decaying, $\alpha = 0.9$}
\label{fig:bnlearn-no-random-decaying-0.9}
\end{subfigure}
\\
\begin{subfigure}[t]{0.45\linewidth}
\includegraphics[width=\linewidth]{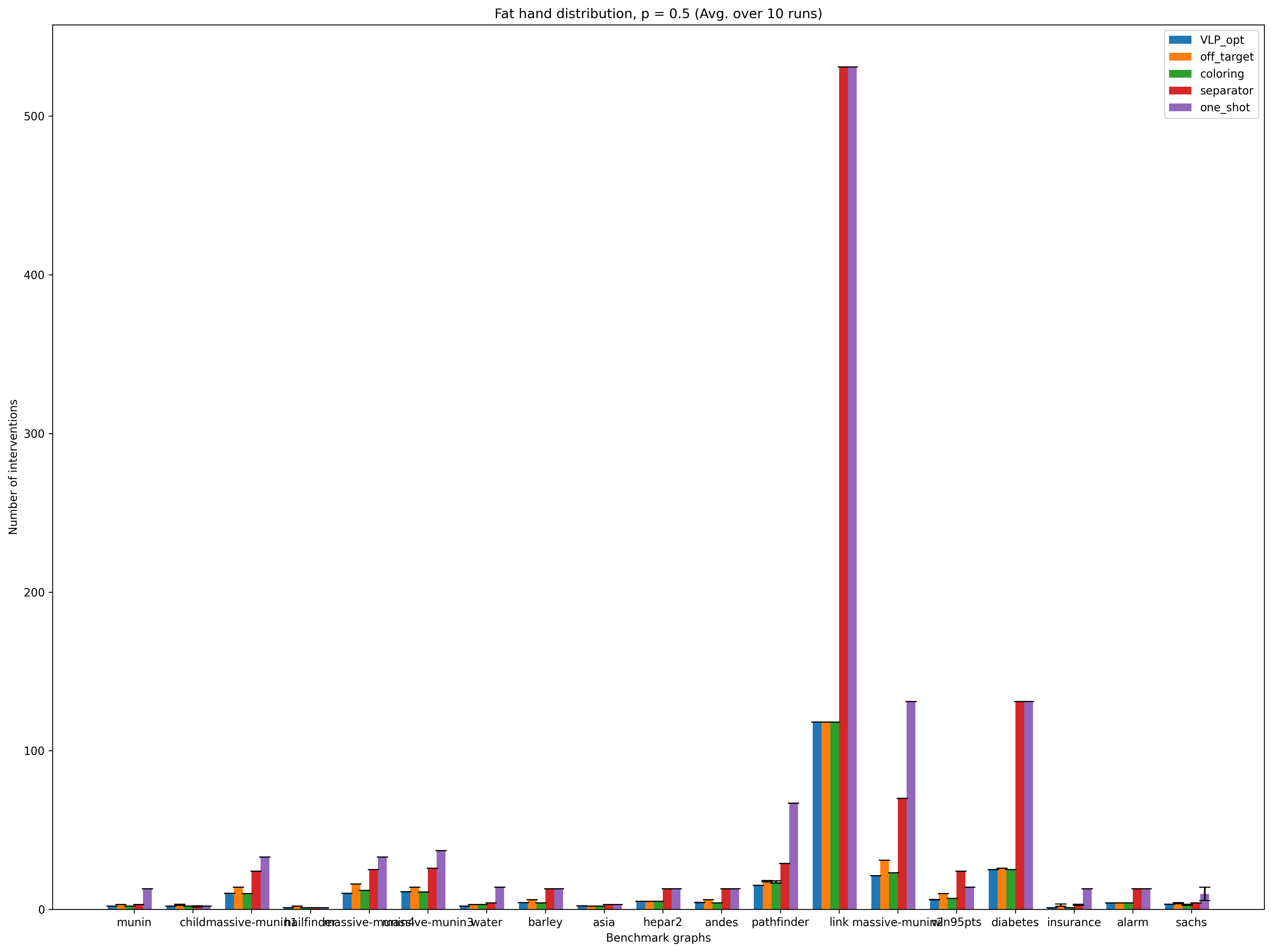}
\caption{Fat hand, $p = 0.5$}
\label{fig:bnlearn-no-random-fat-hand-0.5}
\end{subfigure}
\quad
\begin{subfigure}[t]{0.45\linewidth}
\includegraphics[width=\linewidth]{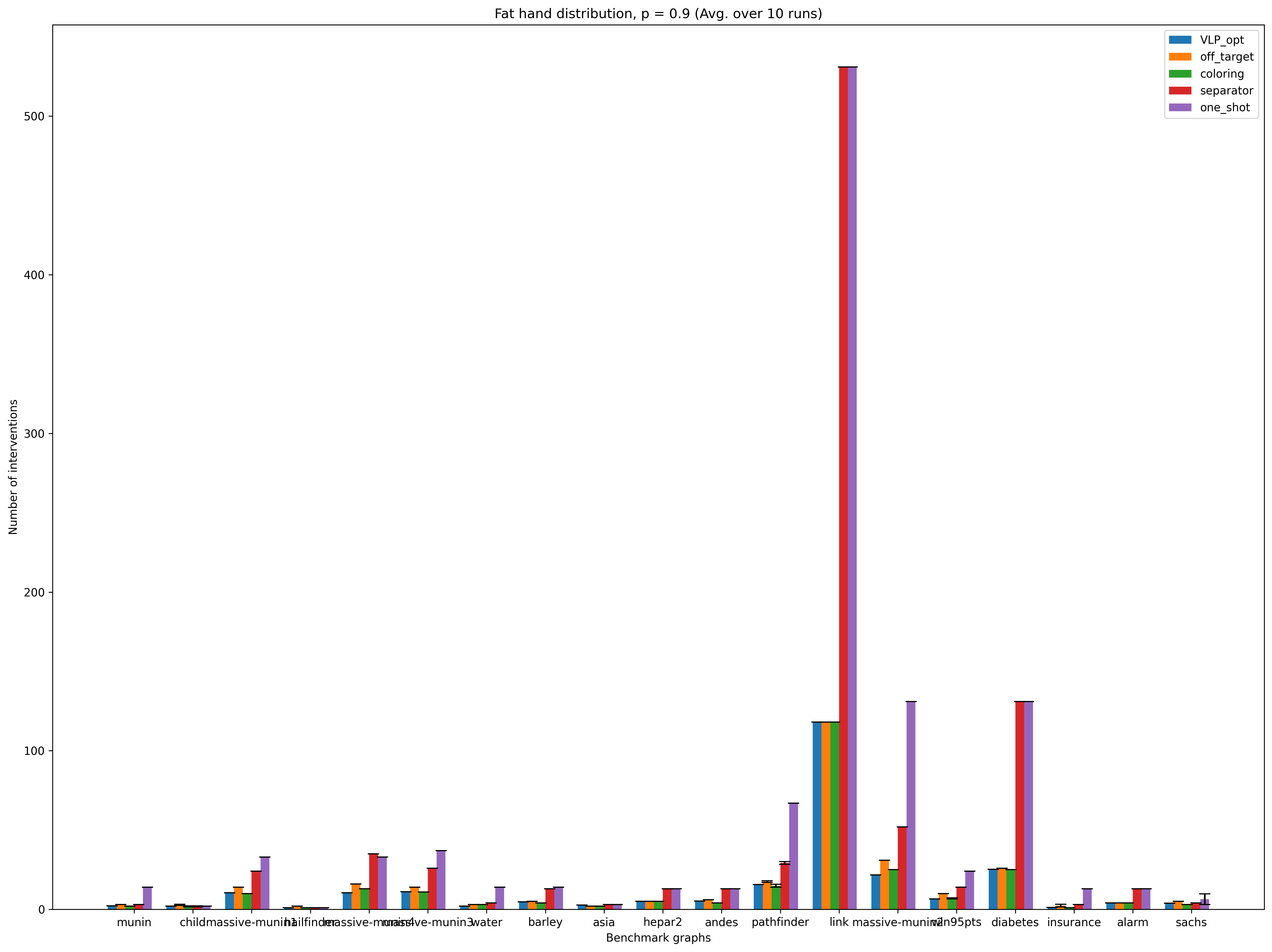}
\caption{Fat hand, $p = 0.9$}
\label{fig:bnlearn-no-random-fat-hand-0.9}
\end{subfigure}
\caption{\texttt{bnlearn} graphs without ``random''.
The optimal value of VLP in {\color{blue}blue} is an $\cO(\log n)$ approximation of $\nu(G^*)$.
Our off-target search \texttt{Off-Target} is in {\color{orange}orange}.
\texttt{Coloring} is in {\color{green!50!black}green}.
\texttt{Separator} is in {\color{red}red}.
\texttt{One-shot} is in {\color{purple!50!blue}purple}.
}
\label{fig:bnlearn-no-random}
\end{figure}

\newpage
\subsubsection{Plots (with ``random'')}

\begin{figure}[H]
\centering
\begin{subfigure}[t]{0.45\linewidth}
\includegraphics[width=\linewidth]{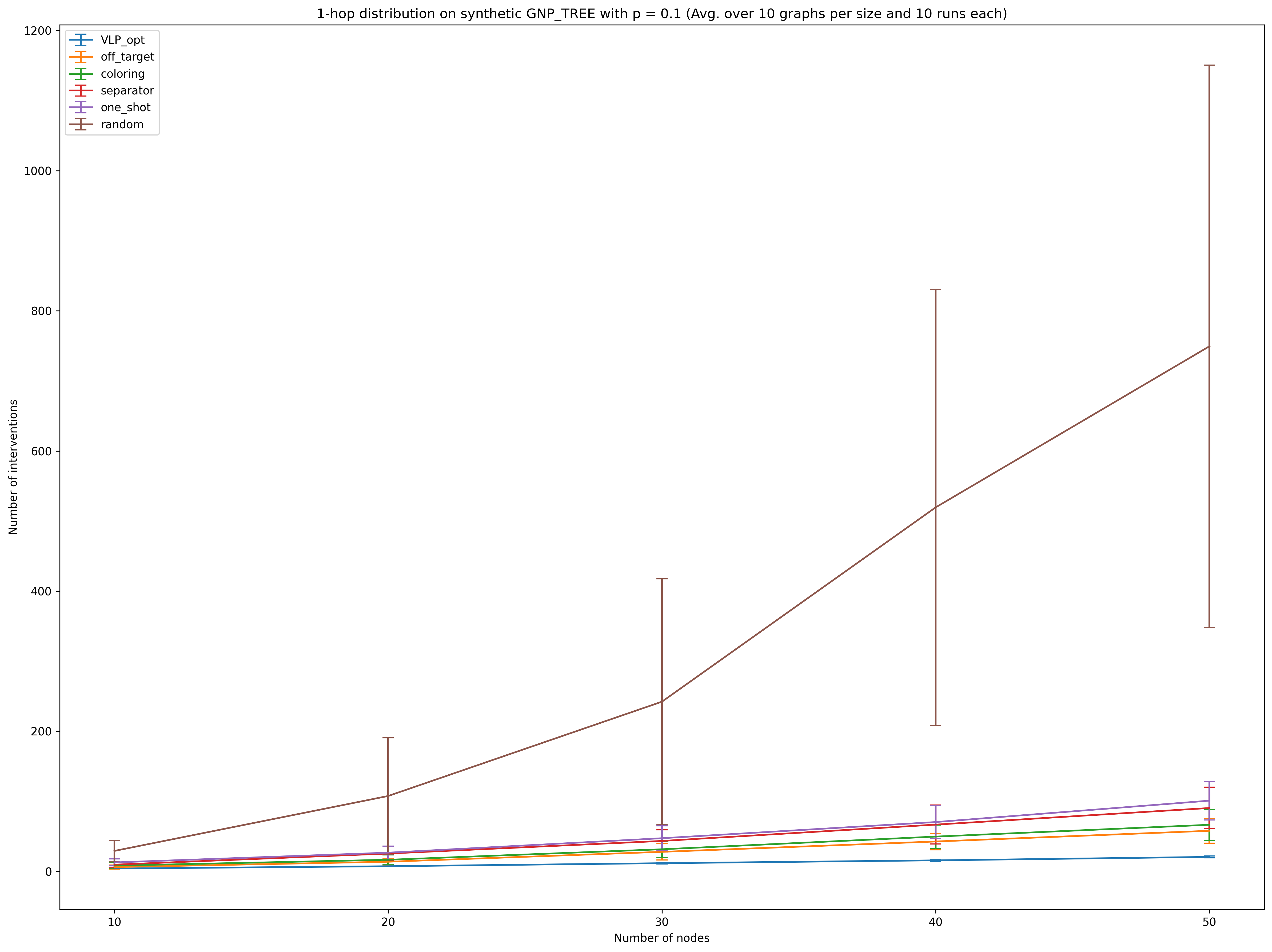}
\caption{1-hop}
\label{fig:gnp-tree-1-hop}
\end{subfigure}
\quad
\begin{subfigure}[t]{0.45\linewidth}
\includegraphics[width=\linewidth]{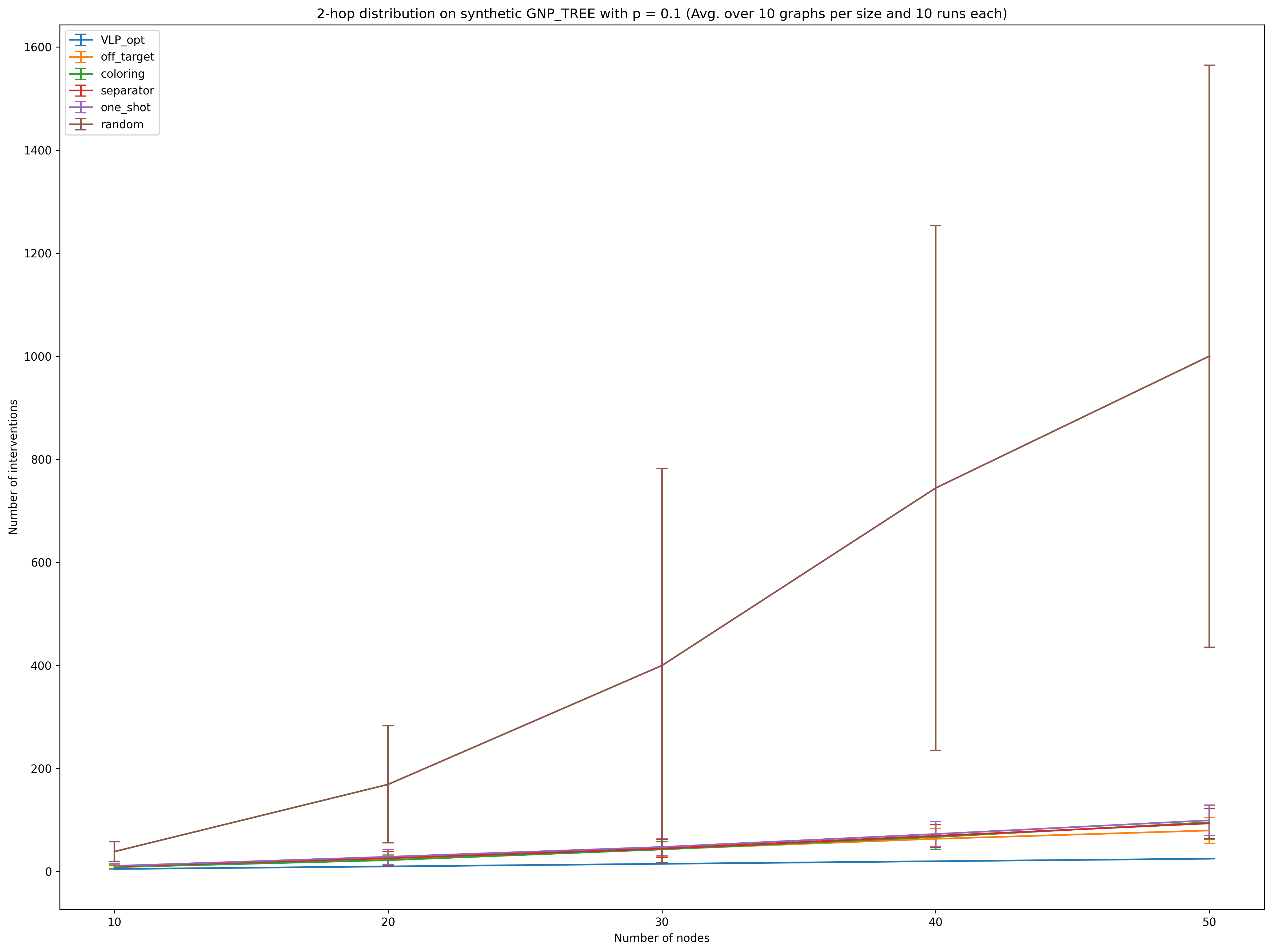}
\caption{2-hop}
\label{fig:gnp-tree-2-hop}
\end{subfigure}
\\
\begin{subfigure}[t]{0.45\linewidth}
\includegraphics[width=\linewidth]{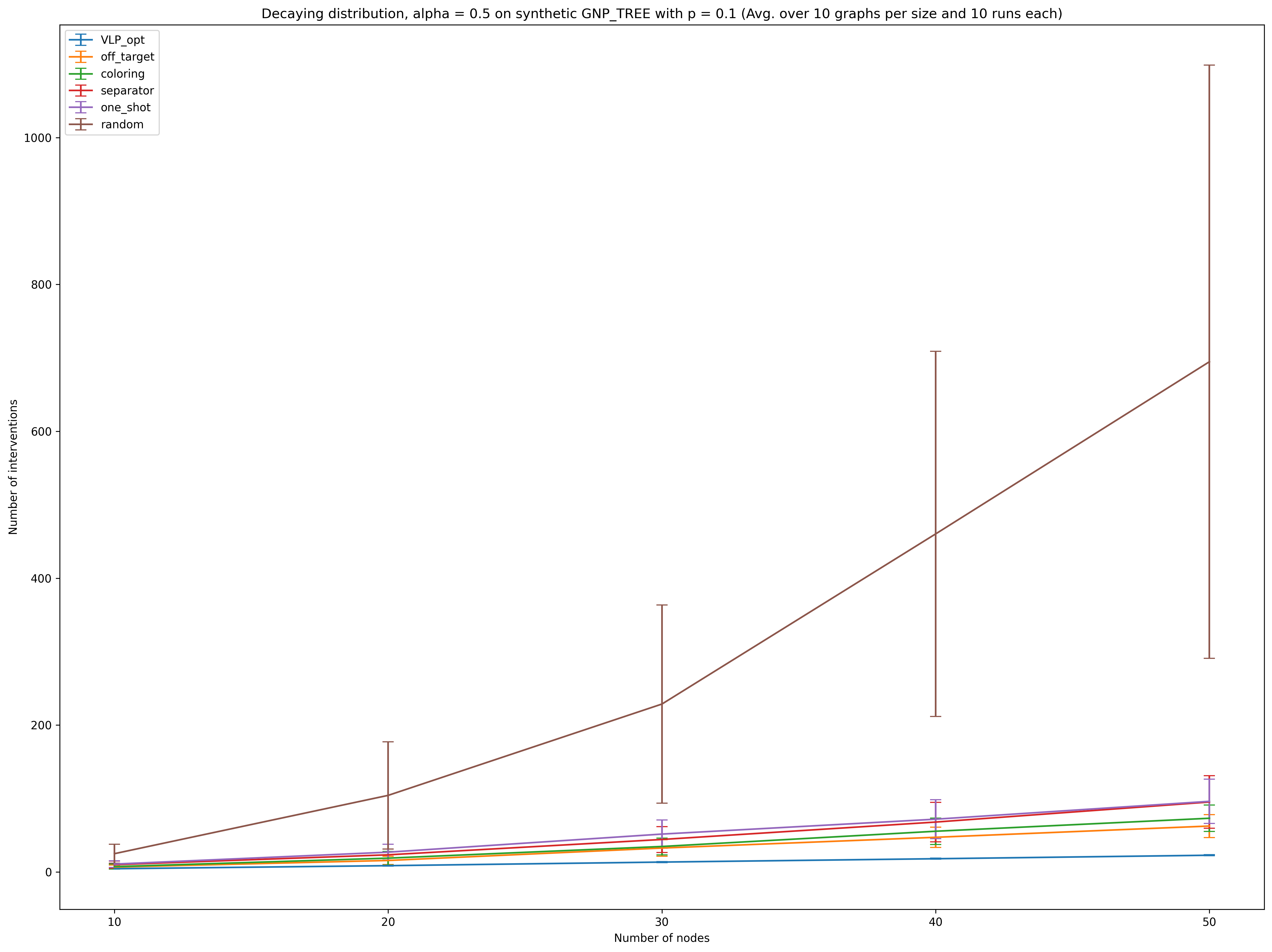}
\caption{Decaying, $\alpha = 0.5$}
\label{fig:gnp-tree-decaying-0.5}
\end{subfigure}
\quad
\begin{subfigure}[t]{0.45\linewidth}
\includegraphics[width=\linewidth]{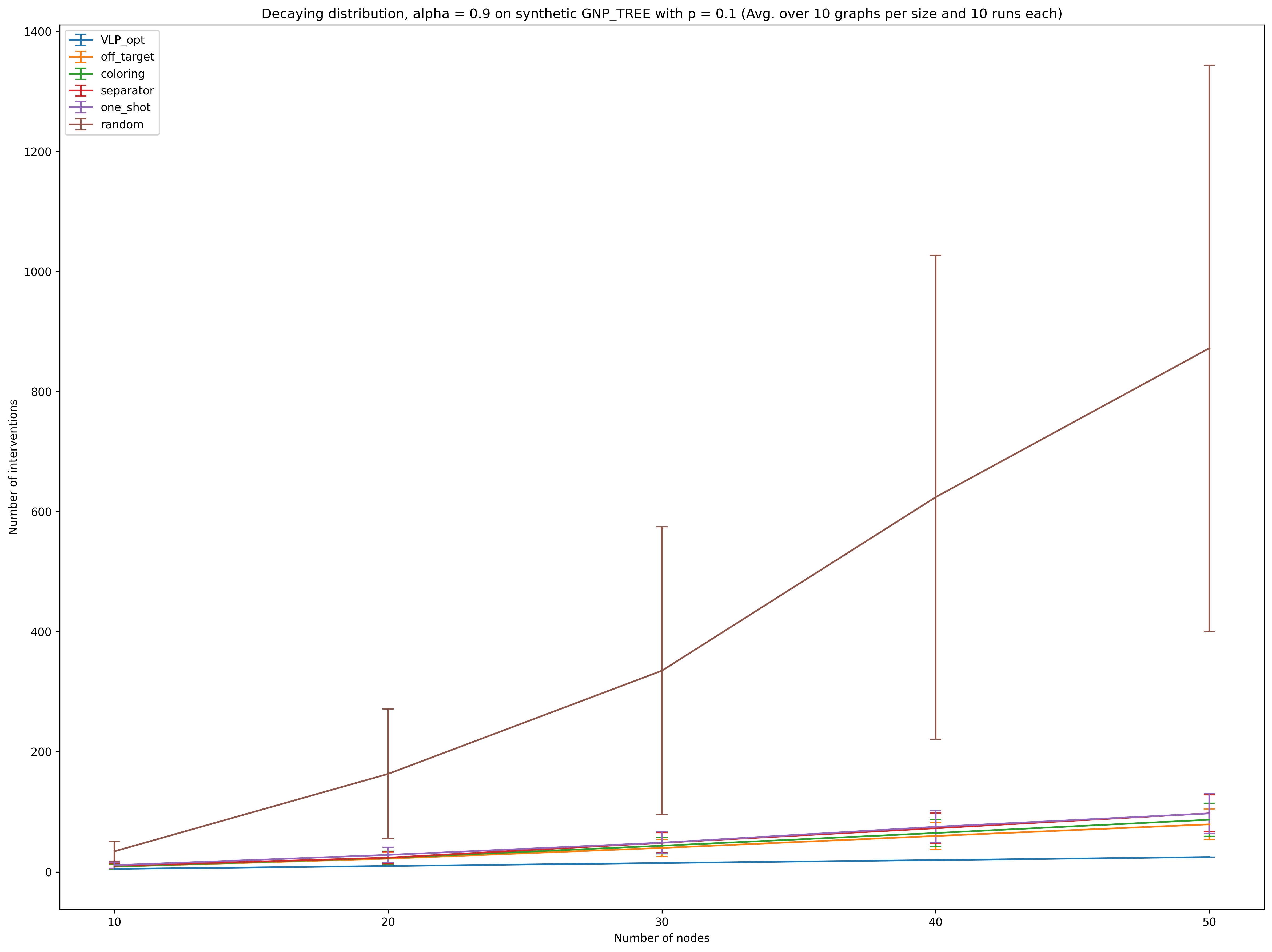}
\caption{Decaying, $\alpha = 0.9$}
\label{fig:gnp-tree-decaying-0.9}
\end{subfigure}
\\
\begin{subfigure}[t]{0.45\linewidth}
\includegraphics[width=\linewidth]{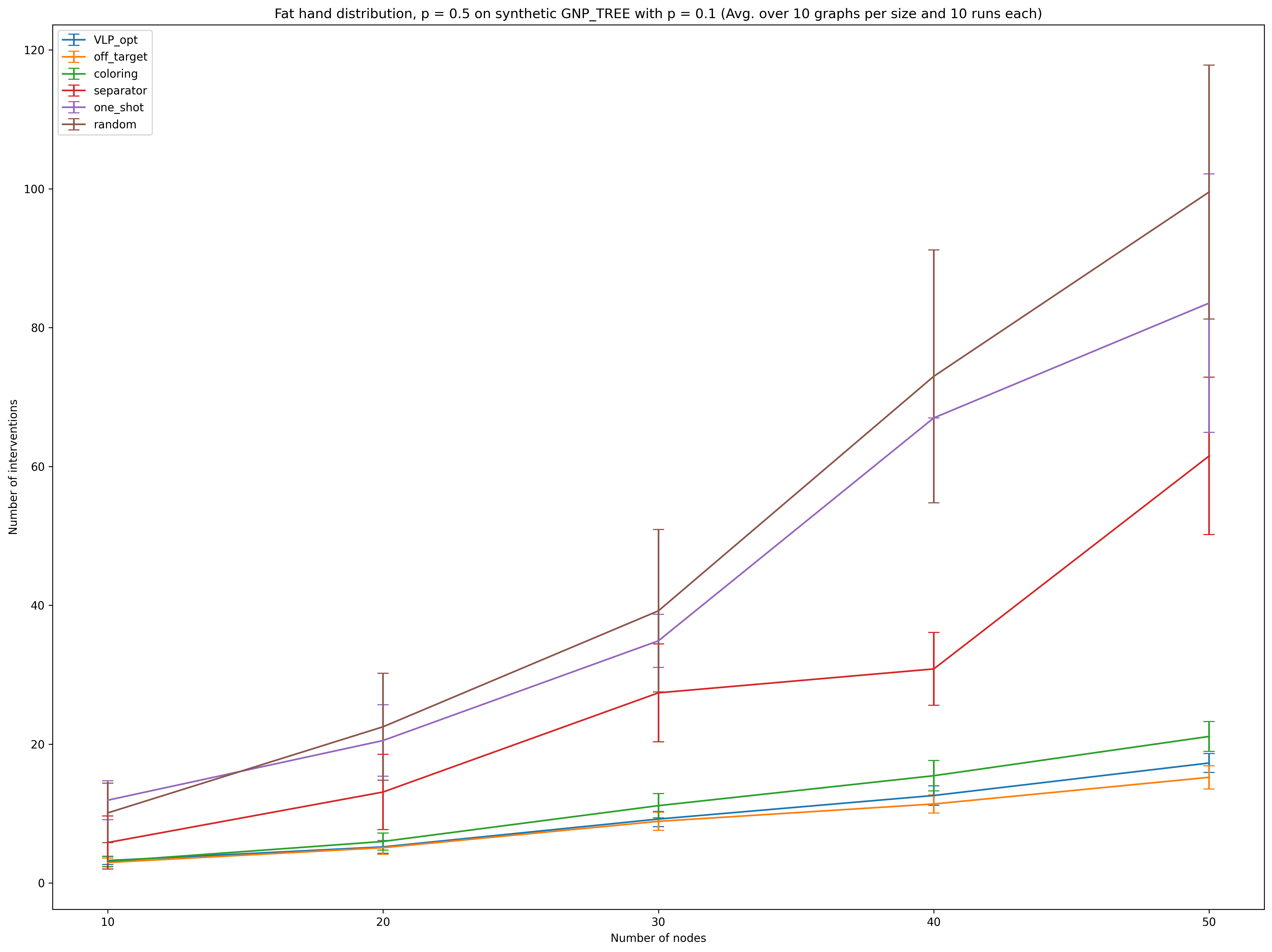}
\caption{Fat hand, $p = 0.5$}
\label{fig:gnp-tree-fat-hand-0.5}
\end{subfigure}
\quad
\begin{subfigure}[t]{0.45\linewidth}
\includegraphics[width=\linewidth]{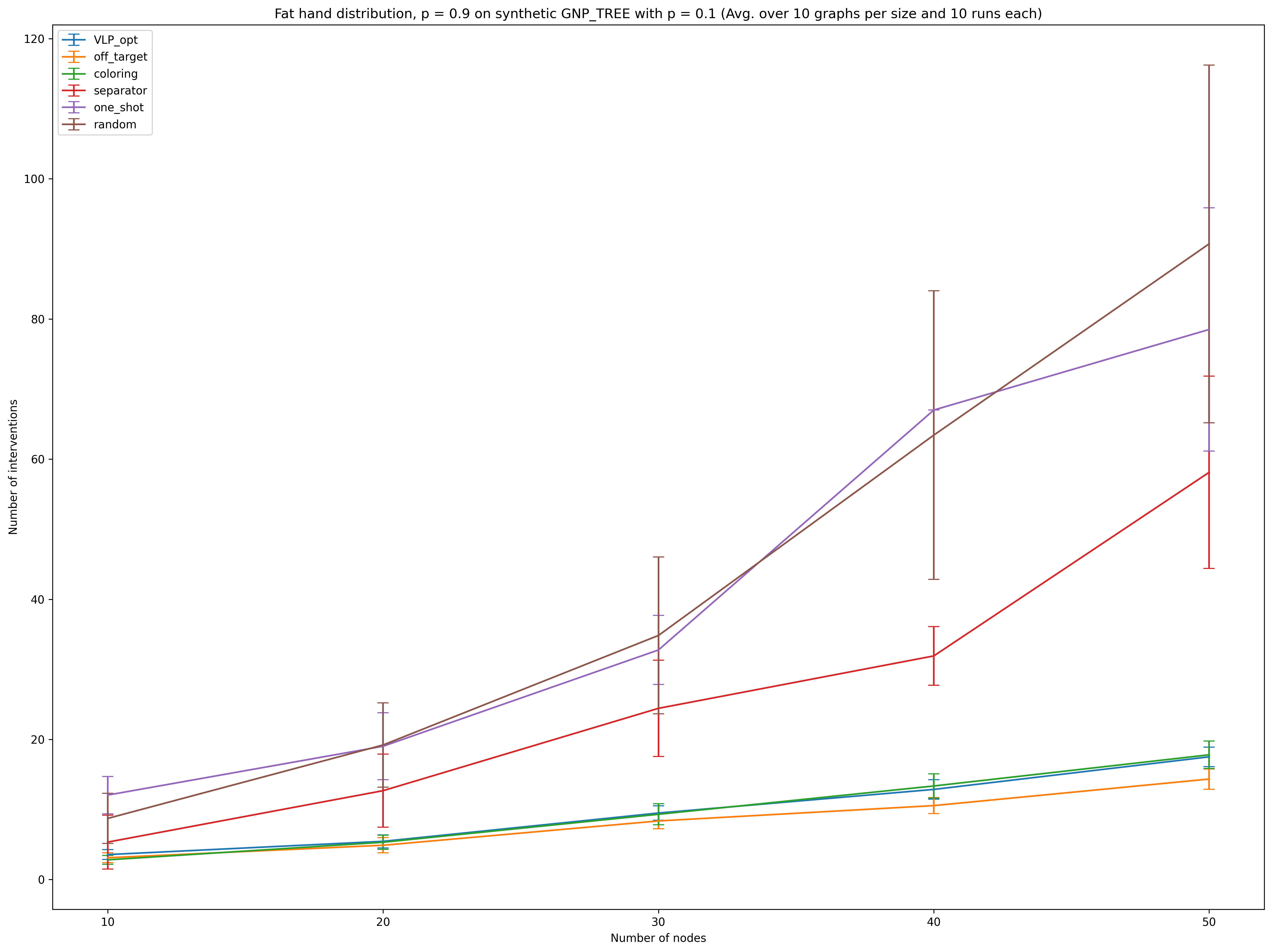}
\caption{Fat hand, $p = 0.9$}
\label{fig:gnp-tree-fat-hand-0.9}
\end{subfigure}
\caption{\texttt{GNP\_TREE} graphs.
The optimal value of VLP in {\color{blue}blue} is an $\cO(\log n)$ approximation of $\nu(G^*)$.
Our off-target search \texttt{Off-Target} is in {\color{orange}orange}.
\texttt{Coloring} is in {\color{green!50!black}green}.
\texttt{Separator} is in {\color{red}red}.
\texttt{One-shot} is in {\color{purple!50!blue}purple}.
\texttt{Random} is in {\color{brown!70!black}brown}.
}
\label{fig:gnp-tree}
\end{figure}

\begin{figure}[H]
\centering
\begin{subfigure}[t]{0.45\linewidth}
\includegraphics[width=\linewidth]{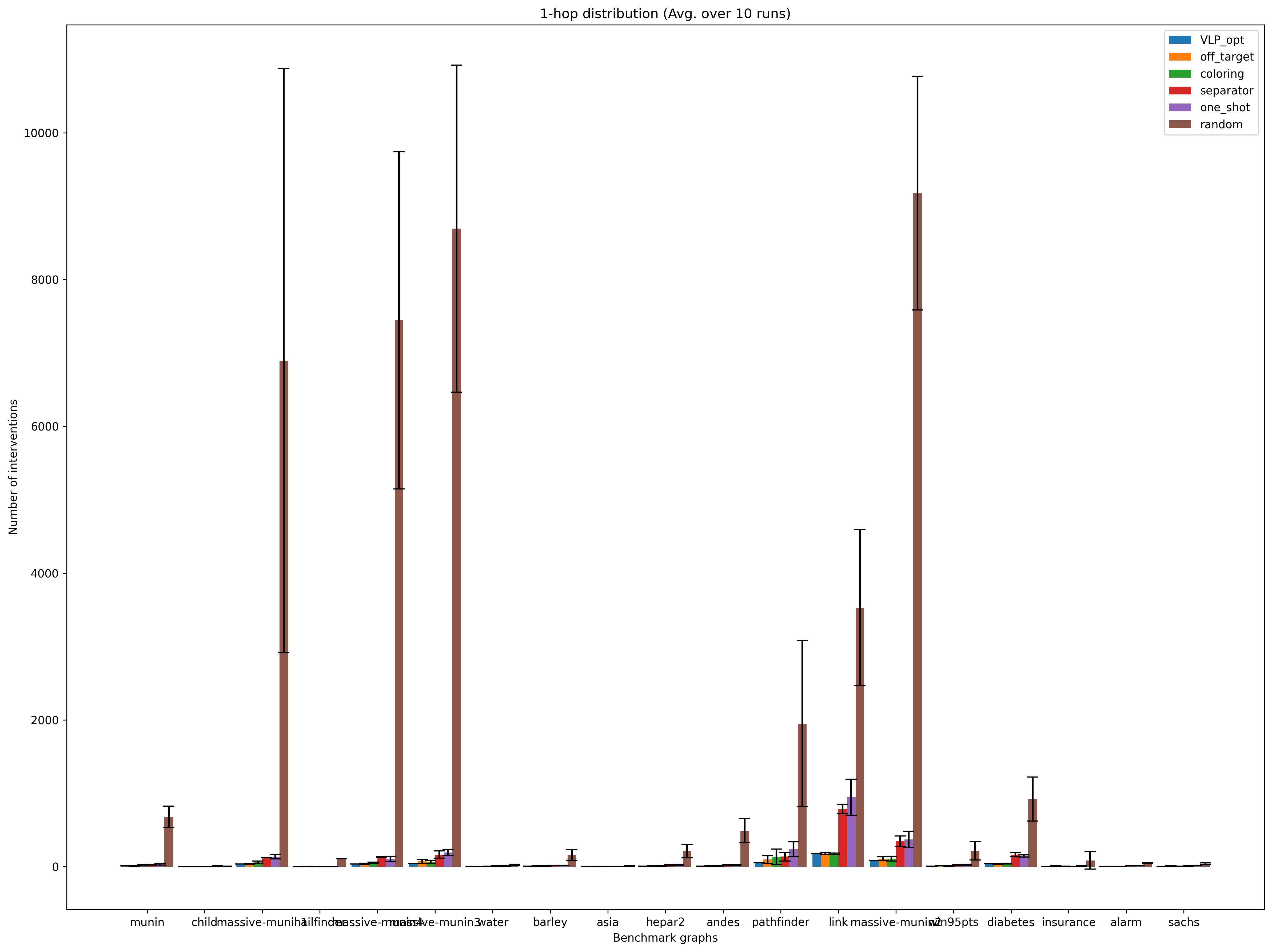}
\caption{1-hop}
\label{fig:bnlearn-1-hop}
\end{subfigure}
\quad
\begin{subfigure}[t]{0.45\linewidth}
\includegraphics[width=\linewidth]{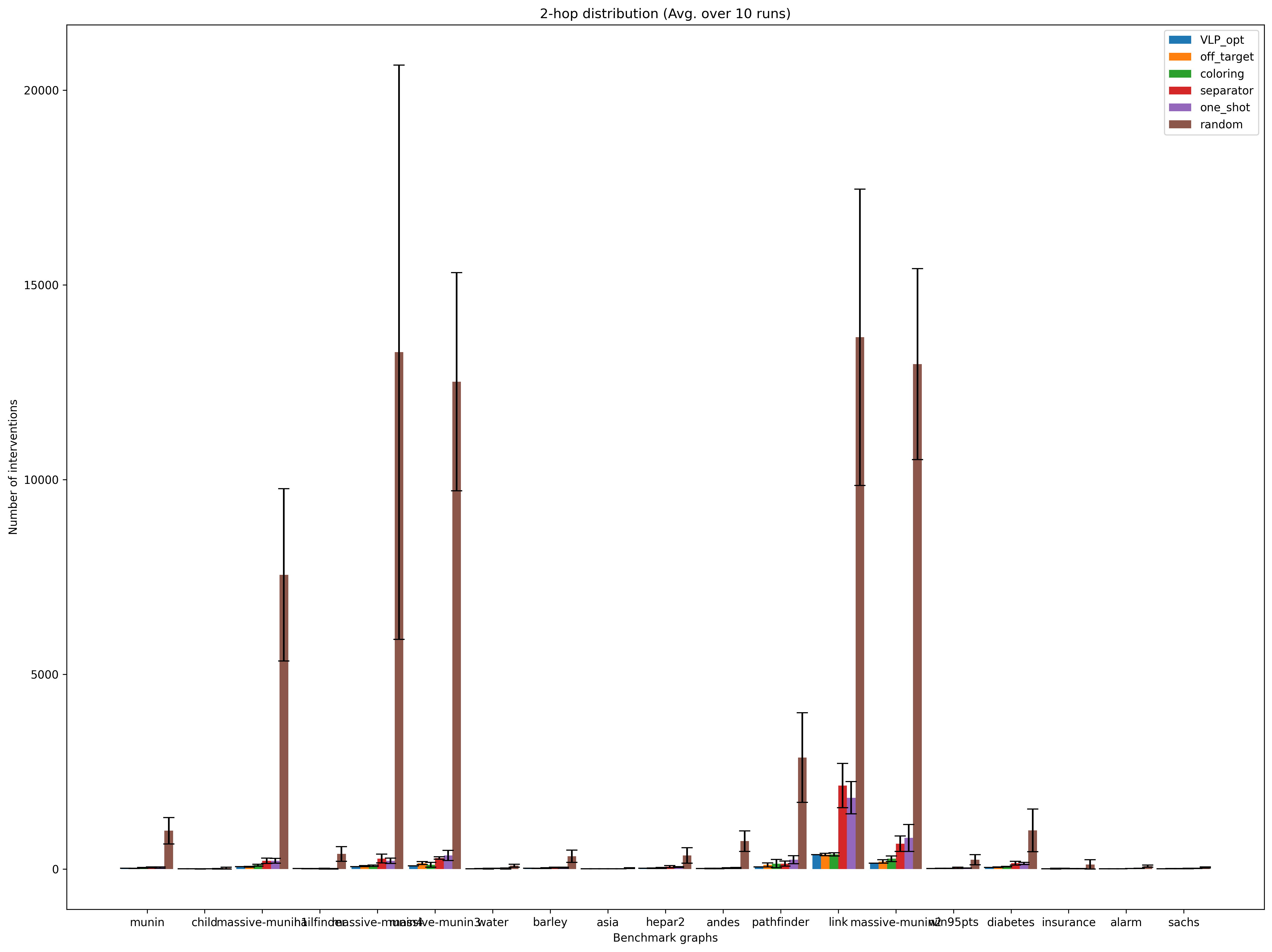}
\caption{2-hop}
\label{fig:bnlearn-2-hop}
\end{subfigure}
\\
\begin{subfigure}[t]{0.45\linewidth}
\includegraphics[width=\linewidth]{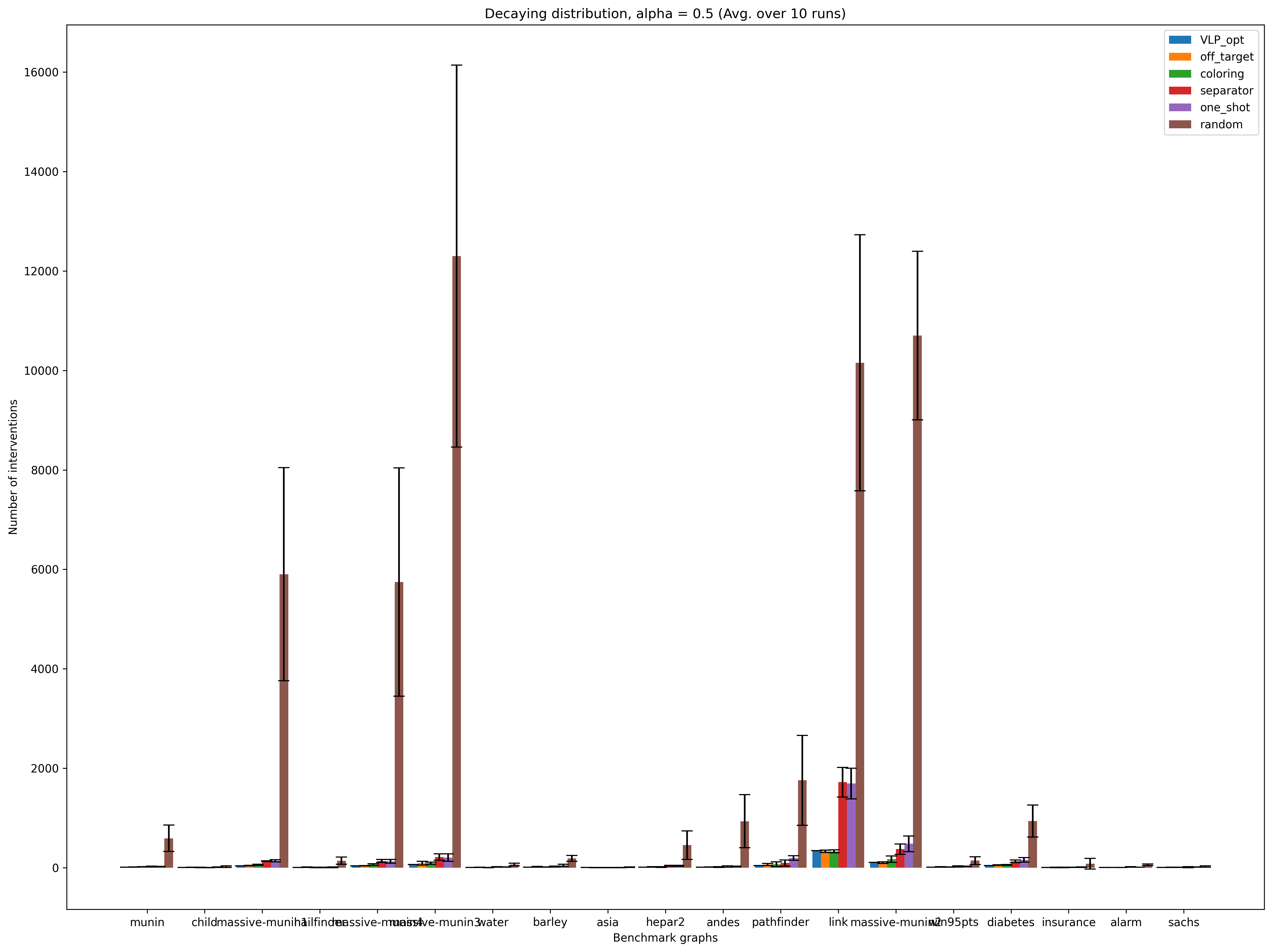}
\caption{Decaying, $\alpha = 0.5$}
\label{fig:bnlearn-decaying-0.5}
\end{subfigure}
\quad
\begin{subfigure}[t]{0.45\linewidth}
\includegraphics[width=\linewidth]{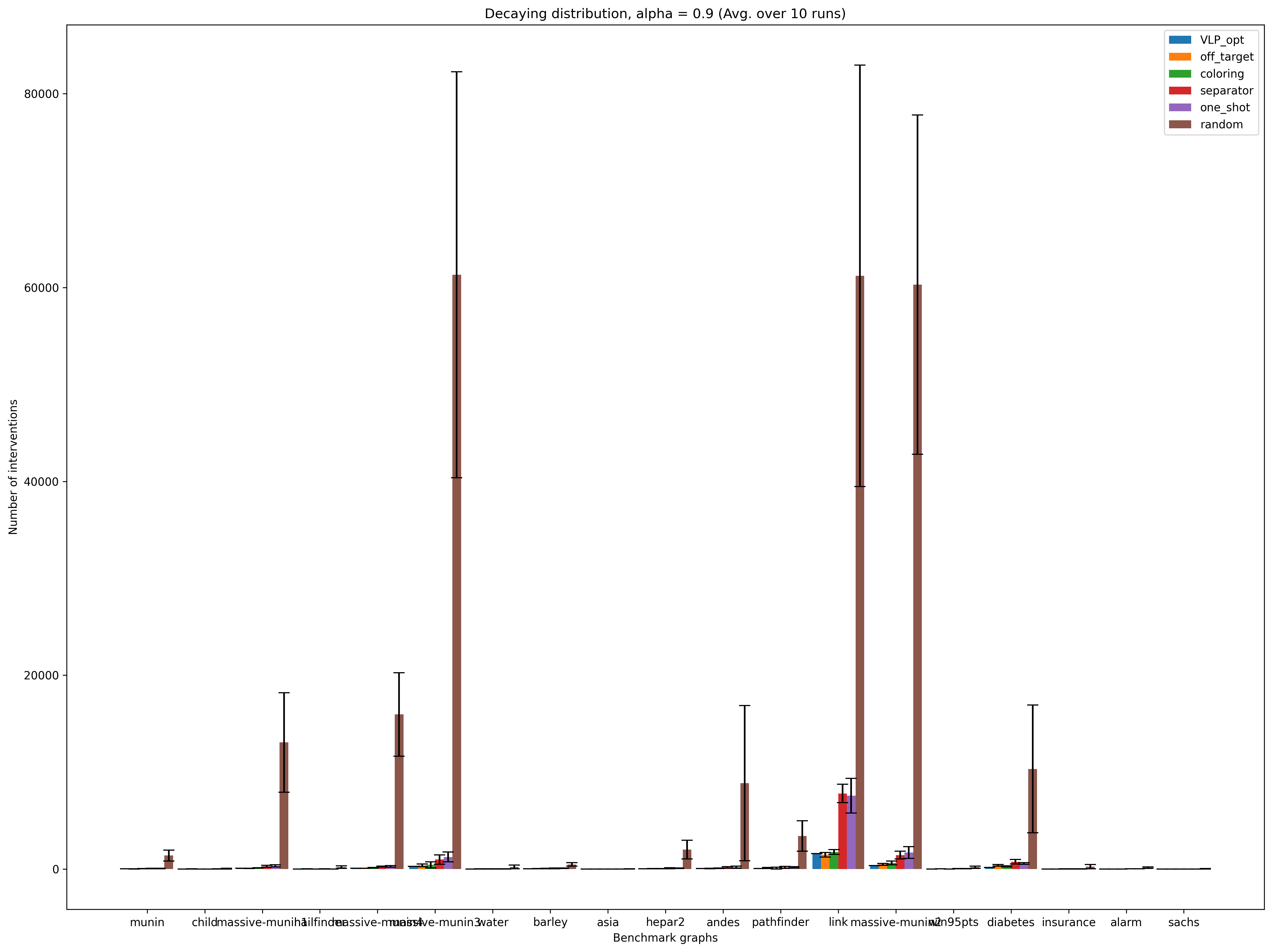}
\caption{Decaying, $\alpha = 0.9$}
\label{fig:bnlearn-decaying-0.9}
\end{subfigure}
\\
\begin{subfigure}[t]{0.45\linewidth}
\includegraphics[width=\linewidth]{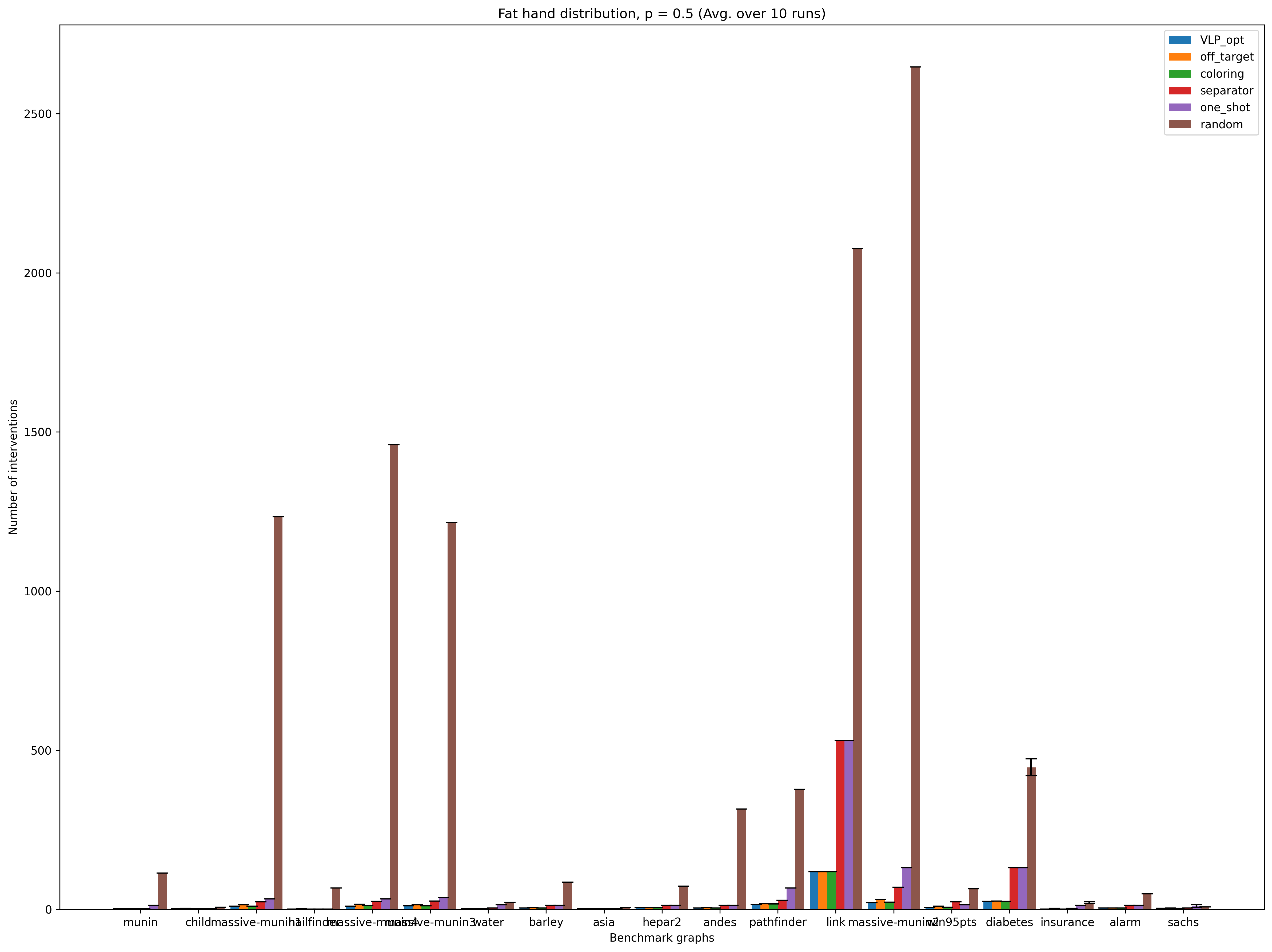}
\caption{Fat hand, $p = 0.5$}
\label{fig:bnlearn-fat-hand-0.5}
\end{subfigure}
\quad
\begin{subfigure}[t]{0.45\linewidth}
\includegraphics[width=\linewidth]{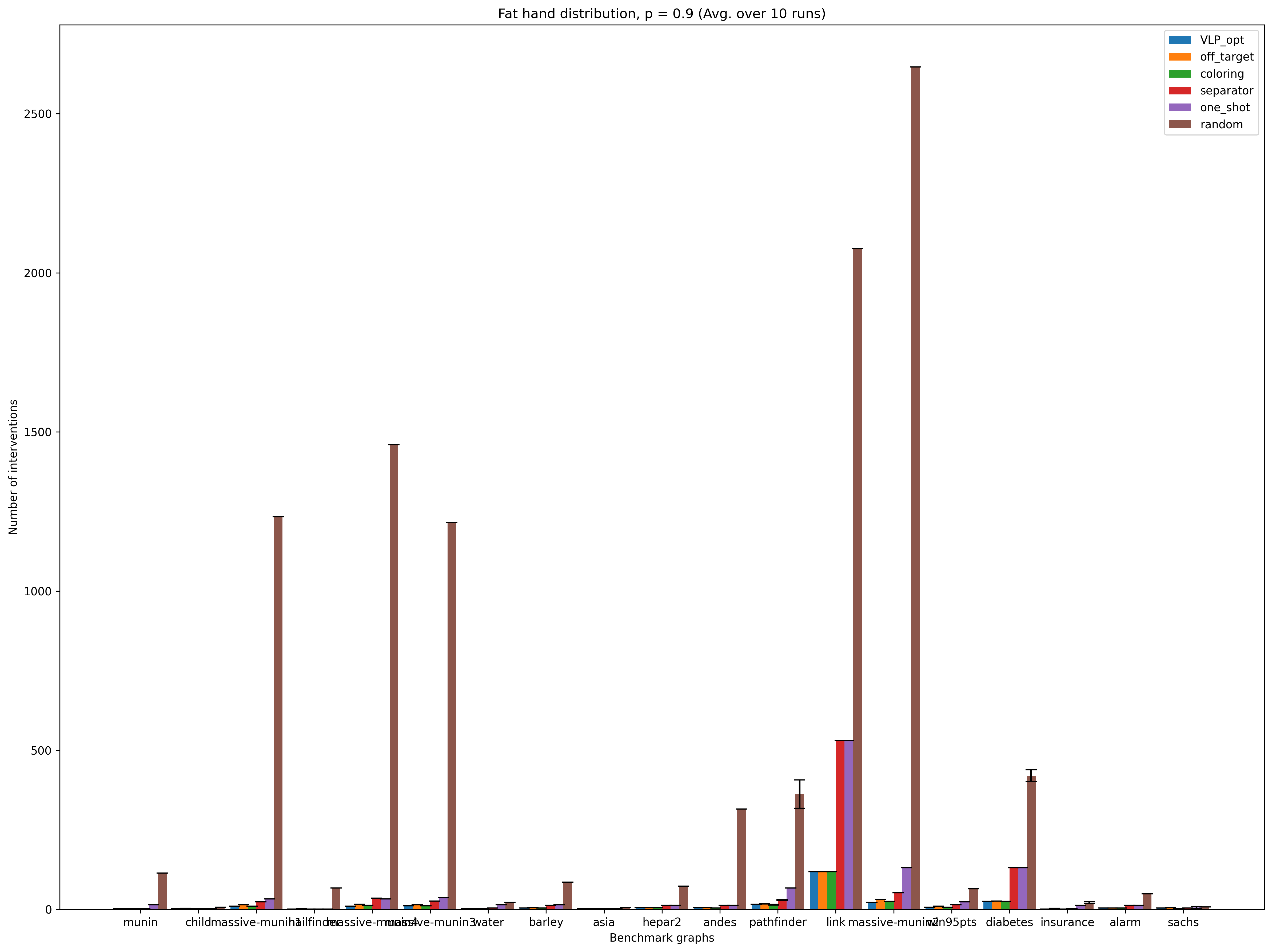}
\caption{Fat hand, $p = 0.9$}
\label{fig:bnlearn-fat-hand-0.9}
\end{subfigure}
\caption{\texttt{bnlearn} graphs.
The optimal value of VLP in {\color{blue}blue} is an $\cO(\log n)$ approximation of $\nu(G^*)$.
Our off-target search \texttt{Off-Target} is in {\color{orange}orange}.
\texttt{Coloring} is in {\color{green!50!black}green}.
\texttt{Separator} is in {\color{red}red}.
\texttt{One-shot} is in {\color{purple!50!blue}purple}.
\texttt{Random} is in {\color{brown!70!black}brown}.
}
\label{fig:bnlearn}
\end{figure}

\end{document}